\documentclass[5p,]{elsarticle}
\usepackage{bm}
\usepackage{amsmath,amsfonts}
\usepackage{amsthm}
\usepackage{tabularx}
\usepackage{array}
\usepackage{subfigure}
\usepackage{diagbox}
\usepackage{multirow}
\usepackage{makecell}
\usepackage{algorithm}
\usepackage{hyperref}
\usepackage{algpseudocode}
\newtheorem{defi}{Definition}
\newtheorem{proposition}{Proposition}
\newtheorem{thm}{Theorem}

\newdefinition{rmk}{Remark}
\newproof{pf}{Proof}

\biboptions{sort&compress}

\usepackage{placeins}  
\usepackage{color}
\usepackage{xcolor}
\usepackage[normalem]{ulem}

\begin{document}

\begin{frontmatter}

\title{Controlled oscillation modeling using port-Hamiltonian neural networks}

\author[1]{M. Linares\corref{cor1}}
\ead{maximino.linares@ircam.fr}
\author[1]{G. Doras}
\ead{guillaume.doras@ircam.fr}
\author[1]{T. H\'elie}
\ead{thomas.helie@ircam.fr}
\author[1]{A. Roebel}
\ead{axel.roebel@ircam.fr}
\cortext[cor1]{Corresponding author}

\affiliation[1]{organization={IRCAM},
                addressline={1, place Igor-Stravinsky},
                postcode={75004},
                city={Paris},
                country={France}}

\begin{abstract}
    Learning dynamical systems through purely data-driven methods is challenging as 
    they do not learn the underlying conservation laws that enable them to correctly 
    generalize. Existing port-Hamiltonian neural network methods have recently been 
    successfully applied for modeling mechanical systems. However, even though these 
    methods are designed on power-balance principles, they usually do not consider 
    power-preserving discretizations and often rely on Runge-Kutta numerical methods. 
    In this work, we propose to use a second-order discrete gradient method embedded 
    in the learning of dynamical systems with port-Hamiltonian neural networks. 
    Numerical results are provided for three systems deliberately selected to span 
    different ranges of dynamical behavior under control: a baseline harmonic oscillator 
    with quadratic energy storage; a Duffing oscillator, with a non-quadratic Hamiltonian 
    offering amplitude-dependent effects; and a self-sustained oscillator, which can 
    stabilize in a controlled limit cycle through the incorporation of a nonlinear 
    dissipation. We show how the use of this discrete gradient method outperforms the 
    performance of a Runge-Kutta method of the same order. Experiments are also carried 
    out to compare two theoretically equivalent port-Hamiltonian systems formulations 
    and to analyze the impact of regularizing the Jacobian of port-Hamiltonian neural 
    networks during training. 
\end{abstract}


\end{frontmatter}

\section{Introduction}
Purely data-driven methods for dynamical systems pose several challenges: 
the volume of useful data is generally limited, they produce accurate-but-wrong 
predictions, they are not capable of dealing with uncertainty and their predictions 
are not explainable nor interpretable~\cite{Cicirello_2024}. At the same time, it is 
natural to leverage the prior knowledge obtained through centuries of scientific 
study in the form of inductive bias~\cite{Baxter_2000} when designing predictive 
models~\cite{karniadakis2021}. In this direction, a successful data-driven physical 
model is one whose inductive bias better captures the true dynamics and is able to 
predict a correct outcome for data not observed during training. These inductive bias are 
incorporated through soft and hard constraints~\cite{eidnes2023}. Soft constraints add penalty terms to the training loss function, discouraging violations of physical laws. This approach is widely applicable, but the model must balance adherence to the constraints against fitting the observed data, without providing any formal guarantee. The seminal \textit{Physics-Informed Neural Networks} (PINNs)~\cite{raissi2019} 
is an example of soft-constraining the model. In contrast, hard constraints ensure 
strict compliance with specified physical laws by embedding them directly into the 
model’s structure, independently of the available data. In this sense, hard constraints 
can be used to incorporate energy conservation laws, symmetry, numerical methods for 
PDEs or Koopman theory~\cite{meng2025physics}. However, imposing hard constraints reduces
 the space of possible solutions and generally limits the model’s expressiveness. 
 As a result, hard constraints are difficult to apply in practice: incorrect assumptions
  about the physical system can lead to overly biased models with a poor generalization 
  performance. \\

In this paper, we consider dynamical systems whose state $\bm{x}$ is governed by the following ODE: 
\begin{equation}
    \label{eq: dynamical_system_introduction}
    \frac{d\bm{x}}{dt}:=\dot{\bm{x}}=\bm{f}(\bm{x})
\end{equation}
and analyze how hard constraints based on energy conservation and power balance principles are 
incorporated into neural networks based on \textit{physically consistent} port-Hamiltonian systems formulations. 
By \textit{physically consistent}, we refer, in this work, to the combination of power-balanced state-space 
models with discrete gradient numerical methods, which preserve the system's energy during discretization. 
The central hypothesis is that enforcing this physical structure as a hard constraint improves interpretability
 and generalization with respect to a vanilla NeuralODE. To substantiate this claim, we conduct a 
systematic study on three controlled oscillatory systems of increasing modeling complexity: a harmonic oscillator, 
a Duffing oscillator and a self-sustained oscillator. These systems are deliberately selected to include linear 
and nonlinear Hamiltonian dynamics as well as nonlinear dissipation mechanisms. The harmonic oscillator serves 
as the simplest baseline example with quadratic energy storage; the Duffing oscillator offers nonlinearities in 
the Hamiltonian, capturing amplitude-dependent effects; and the self-sustained oscillator incorporates a 
nonlinear dissipation which can stabilize the system in a controlled-limit cycle. \\

The main contributions of this paper are
\begin{itemize}
    \item a comparison of two theoretically equivalent port-Hamiltonian systems (PHS) formulations: the 
    semi-explicit PH-Differential-Algebraic-Equations (PH-DAE) and the input-state-output PHS with feedthrough; 
    when they are implemented as port-Hamiltonian neural networks (PHNNs).
    \item a performance comparison between the Gonzalez discrete gradient method, which is a second-order 
    energy-preserving numerical method, and a second-order explicit Runge-Kutta method when used to discretize the 
    PHNN model during learning.
    \item an empirical study of the impact of regularizing the Jacobian of PHNN through two methods already 
    applied to NeuralODEs and a new one tackling the stiffness of the learned ODE solutions.
\end{itemize}

The rest of this paper is organized as follows. Section \ref{sec:Preliminaries} introduces the necessary 
preliminaries on dynamical systems, port-Hamiltonian systems, numerical methods, neural ordinary differential 
equations, and port-Hamiltonian neural networks. Section \ref{sec:Port-Hamiltonian-Model} presents the port-Hamiltonian 
formulations considered in this work, the enforced physical constraints, and the oscillatory examples used 
throughout the paper. Section \ref{sec:phnn-models} focuses on port-Hamitonian neural networks, detailing how 
physical constraints are incorporated into the learning process, the comparison between continuous- and 
discrete-time models, and the Jacobian regularization in the port-Hamiltonian neural networks. Section \ref{sec:experiments} 
formulates the key research questions addressed by the experimental study. Section \ref{sec:results} reports 
and discusses the results of the experiments. Finally, Section \ref{sec:conclusion} concludes the paper and outlines 
directions for future work. The code is publicly available:  \url{https://github.com/mlinaresv/ControlledOscillationPHNNs} [will be released after paper acceptance].
\section{Preliminaries}
\label{sec:Preliminaries}

\subsection{Dynamical systems}
\label{sec:dynamical_systems}
\subsubsection{Dynamical system ODE}

Consider a dynamical system governed by the following system of equations:
\begin{equation}
\label{eq: dynamical_system}
\begin{cases}
    \dot{\bm x}(t)=\bm f(\bm x(t),\bm u(t))  \\
     \bm y(t)=\bm h(\bm x(t),\bm u(t)) 
\end{cases}
\end{equation}
where $\bm{x}(t)\in\mathbb{R}^{n_x},\bm{u}(t)\in\mathbb{R}^{n_u}$ and $\bm{y}(t)\in\mathbb{R}^{n_y}$ 
are referred to as the \textit{state}, \textit{input} and \textit{output} of the system, respectively; 
and where $\bm{f}:\mathcal{D}\rightarrow\mathbb{R}^{n_x}$, $\bm{h}:\mathcal{D}\rightarrow\mathbb{R}^{n_y}$ 
are $\mathcal{C}^1$ functions with domain $\mathcal{D}\subseteq\mathbb{R}^{n_x}\times\mathbb{R}^{n_u}$. The system of equations \eqref{eq: dynamical_system} as a whole is referred to as \textit{state-space model}~\cite{khalil_nonlinear_2002}.\\

In this work, we consider autonomous systems (i.e. where $\bm u(t)$ is constant for all $t$).
Obtaining a solution for a given initial condition is often referred to as solving the \textit{initial value problem} (IVP)
\begin{equation}
\label{eq: ivp}
\dot{\bm x}(t)=\bm f(\bm x(t))~~~~~~~~\bm x(0)=\bm x_0.
\end{equation}
In the following, we omit the explicit time dependence of $\bm x(t)$, $\bm y(t)$ and $\bm u (t)$ to simplify the notation, when there is no ambiguity, and we let $\bm J_{\bm f}(\bm x)$ denote the \textit{Jacobian matrix} of the function $\bm f$ of \eqref{eq: ivp}.

\subsubsection{Well-posed problems}
The term \textit{well-posed} was introduced 
to refer to problems where a) the solution exists, b) is unique, and c) depends continuously on the initial conditions and parameters~\cite{hadamard1902problemes}.
A sufficient condition for an IVP to be well-posed is that $\bm f$ is $K$-Lipschitz~\cite{hirsch1974differential}. \\

The \textit{spectral norm} of a matrix $\bm A$ is defined as: 
\begin{equation}
    \label{eq: spectral-norm-jacobian}
    \| \bm A \|_2=\max_{\|\bm x\|_2=1}\|\bm A\bm x\|_2=\sigma_{max}(\bm A),
\end{equation}
where $\sigma_{max}(\bm A)$ is the largest \textit{singular value} of $\bm A$~\cite{golub2013matrix}. It is a standard result that if the following condition holds:
\[ \|\bm J_{\bm f}(\bm x)\|_2\leq K<\infty~~~~~\forall\bm x\in\mathcal{D} \]
then $\bm f$ is K-Lipschitz \cite{hirsch1974differential} (see \ref{ap: preliminaries}). \\

In practice, controlling the well-posedness of the IVP reduces to enforcing an upper bound on the spectral norm of the Jacobian of $\bm f$.

\subsubsection{Well-conditioned problems}
The term \textit{well-conditioned} refers to problems seen as a function $\bm f$ where a small perturbation of $\bm x$ yields only small changes in $\bm f (\bm x)$ -- the meaning of \textit{small} depending on the context~\cite{trefethrn1997}. This relationship can be characterized by the \textit{relative condition number} of $\bm f$ defined as:
\begin{equation}
    \label{eq: condition-number-trefethen}
    \kappa_{\bm f}(\bm x)=\lim_{\delta\rightarrow0}\sup_{\|\delta \bm x\|\leq\delta}\left(\frac{\|f(\bm x+\delta\bm x)-f(\bm x)\|_2}{\|f(\bm x)\|_2}\left / \frac{\|\delta\bm x\|_2}{\| \bm x\|_2}\right. \right),
\end{equation}
where $\delta \bm x$ and $\delta\bm f$ are infinitesimal. If $\bm f\in\mathcal{C}^1$, this rewrites as:
\begin{equation}
    \label{eq: condition-number-trefethen-def}
    \kappa_{\bm f}(\bm x)=\frac{\|\bm J_{\bm f}(\bm x)\|_2}{\|\bm f(\bm x)\|_2/\|\bm x \|_2}.
\end{equation}
for $\|\bm x \|_2 \neq 0$ and its limit when $\|\bm x \|_2 \rightarrow 0$. If $\kappa_{\bm f}(\bm x)$ is small (resp. large), the problem is said to be \textit{well-conditioned} (resp. \textit{ill-conditioned}). \\

For a linear system $\bm f(\bm x)=\bm A\bm x$, the Jacobian of $\bm f$ becomes $\bm J_{\bm f}(\bm x) = \bm A, \quad \forall \bm x$. It is a standard result that the condition number can then be bounded s.t.:
\begin{equation}
    \label{eq: condition-number-linear-systems}
    \kappa_{\bm f}(\bm x)\leq\|\bm J_{\bm f}(\bm x)\|_2\|\bm J^{-1}_{\bm f}(\bm x)\|_2=\frac{\sigma_{max}(\bm J_{\bm f}(\bm x))}{\sigma_{min}(\bm J_{\bm f}(\bm x))}
\end{equation}
where $\sigma_{max}(\bm J_{\bm f}(\bm x)),\sigma_{min}(\bm J_{\bm f}(\bm x))$ are the maximum and minimum singular values of $\bm J_{\bm f}(\bm x)$~\cite{trefethrn1997}. In the following, the upper bound in \eqref{eq: condition-number-linear-systems} is denoted as $\kappa(\bm J_{\bm f}(\bm x))$.\\

In practice, controlling the well-conditioning of a linear problem $\bm f$ reduces to enforcing an upper bound on its condition number.


\subsubsection{Stiff ODE systems}
The term \textit{stiff} refers to an initial value problem for which certain numerical methods require prohibitively small step sizes to maintain stability. This behavior can be characterized by the \textit{stiffness ratio} defined as:
\begin{equation}
\label{eq: stiffness-ratio}
    \rho(\bm J_{\bm f})=\frac{\max|\mathcal{R}(\lambda(\bm J_{\bm f}))|}{\min|\mathcal{R}(\lambda(\bm J_{\bm f}))|},
\end{equation}
where $\lambda(\bm J_{\bm f})$ are the \textit{eigenvalues} of the Jacobian matrix and $\mathcal{R}(\cdot):\mathbb{C}\rightarrow\mathbb{R}$ is the real part operator~\cite{iserles2009first}. Although it is rigorously only true for linear equations, a system with a large stiffness ratio is generally stiff~\cite{iserles2009first}. \\

In practice, controlling the stiffness of an ODE system reduces to enforcing an upper bound on its stiffness ratio.

\subsection{Port-Hamiltonian systems}

In the Hamiltonian formalism \cite{taylor2005}, the mechanical state of $\mathcal{S}$ is represented by a vector $\bm{x} = [\bm{q}, \bm{p}]^\top \in \mathbb{R}^{n_x}$, 
where $\bm{q}, \bm{p} \in \mathbb{R}^{n_x/2}$ denote the generalized coordinates and conjugate momenta. The dynamics is governed through a scalar-valued energy function 
$H(\bm{x})$ known as \textit{Hamiltonian}, and the time evolution of the system follows Hamilton’s equations:
\begin{equation}
\label{eq: hamilton-equation}
\dot{\bm{x}} = \bm f(\bm x)=\bm{S} \nabla H(\bm{x}) = \begin{bmatrix}
0 & \mathbb{I}_n \\
- \mathbb{I}_n & 0
\end{bmatrix} \nabla H(\bm{x}),
\end{equation}
where $\bm{S} \in \mathbb{R}^{n_x \times n_x}$ is the canonical symplectic matrix which imposes the energy conservation principle. This formalism is restricted to 
conservative closed systems and does not readily describes dissipation or external control, which are common to many real-world systems. \\ 

Port-Hamiltonian formalism~\cite{maschke1992,duindam2009modeling,van2014port} generalizes Hamiltonian mechanics to multi-physics open systems by explicitly modeling energy exchange with the 
environment through ports (inputs/outputs) and dissipation. In this work, the considered class of open systems is represented in this formalism as a network of:
\begin{itemize}
    \item \textit{energy storing components} 
    with state $\bm{x}$ and energy $E := H(\bm{x})$, with $H$ positive definite and $\mathcal{C}^1$-regular, so the stored power is $P_{\text{stored}} := \dot{E} = \nabla H(\bm{x})^\top \dot{\bm{x}}$,
    \item \textit{dissipative components}, described by an effort law $\bm{z}(\bm{w})$ for a flow variable $\bm{w}$, where $P_{\text{diss}}:=\bm{z}(\bm{w})^\top \bm{w}$; 
    by convention, the dissipated power is counted positively, i.e. $P_{\text{diss}}\geq 0$, with $P_{\text{diss}} = 0$ for conservative components,
    \item \textit{external components} are represented through ports by system inputs $\bm{u}$ and outputs $\bm{y}$, with the convention that $P_{\text{ext}} := \bm{u}^\top \bm{y}$ 
    is positive when received by these external components.
\end{itemize}
The coupling of internal flows
$\mathcal{F}$ and efforts $\mathcal{E}$ governs 
the time evolution of the system, expressed as:

\begin{equation} \label{eq:phs_S_intro}
    \underbrace{
    \begin{bmatrix}
    \dot{\bm{x}} \\
    \bm{w} \\
    \bm{y}
    \end{bmatrix}}_{:=\mathcal{F}}
    =
    \bm{S}
    \underbrace{
    \begin{bmatrix}
    \nabla H(\bm{x}) \\
    \bm{z}(\bm{w}) \\
    \bm{u}
    \end{bmatrix}
    }_{:=\mathcal{E}},
    \end{equation}
    where $\bm S=-\bm S^T$ is skew-symmetric so that
    \begin{equation}
        0 = \underbrace{P_{\text{stored}}}_{\nabla H(\bm x)^T\dot{\bm x}} + \underbrace{P_{\text{diss}}}_{\bm z(\bm w)^T\bm w}+ \underbrace{P_{\text{ext}}}_{\bm u^T\bm y} = \langle {\mathcal{E}}|{\mathcal{F}}\rangle 
    \end{equation}
    with $\langle\mathcal{E}|\mathcal{F}\rangle=\mathcal{E}^T\mathcal{F}$. Indeed,
    \begin{equation}
        \mathcal{E}^T\mathcal{F}\overset{(14)}{=}\mathcal{E}^TS\mathcal{E}\overset{(S=-S^T)}{=}0
    \end{equation}
The passivity of the system stems from the fact that $P_{\text{diss}} \ge 0$, which imposes that $P_{\text{stored}} = -P_{\text{diss}} - P_{\text{ext}}\leq-P_{\text{ext}}$. As a consequence, $P_{\text{stored}} \leq 0$ (non-increasing internal energy) when external sources are off. This intrinsically physics-consistent formulation is highly general and 
can be applied to a wide range of physical domains, including acoustics, fluid mechanics, quantum physics and others \cite{chaigne2016acoustics, aoues2017modeling,helie2022,cardoso2024port,roze2024time} 
(see~\cite{van2014port} for formulations more general than \eqref{eq:phs_S_intro}).

\subsection{Numerical methods}
For a given IVP \eqref{eq: ivp}, a \textit{numerical method} approximates the solution without the need to analytically solve the ODE. A \textit{discrete trajectory} $\mathcal{T}$ is 
defined as the set $\{\bm x_n\}_{n\in \mathbb{N}}$ of consecutive states where $\bm x_n:=\bm x(nh)$, $h=1/sr$ is the time step and $sr$ the sampling rate. 
Given the \textit{initial state} $\bm x_0=\bm x(0)$, the applied control $\bm u$ and the ODE governing $\bm x$,
a trajectory $\mathcal{T}$ is obtained by such numerical method. If the governing equation is given by \eqref{eq:phs_S_intro}, the resulting $\mathcal{T}$ 
is said to be a \textit{discretization of a port-Hamiltonian system}. In this work, we use numerical methods that belong to two different categories:
\subsubsection{Runge-Kutta (RK) methods} 
 A well-known family of numerical methods is the \textit{Runge-Kutta methods}~\cite{Hairer2006}. Let $b_i,a_{i,j}~(i,j=1,...,s)$ be real numbers and let $c_i=\sum_{j=1}^sa_{ij}$. 
 An \textit{s-stage} Runge-Kutta method is given by 
\begin{equation}
    \begin{cases}
        \bm k_i=\bm f\left(t_0+hc_i,\bm x_n+h\sum_{j=1}^sa_{i,j}\bm k_j \right),~~~i=1,...,s \\
        \bm x_{n+1}=\bm x_n+h\sum_{i=1}^sb_i\bm k_i
    \end{cases}
\end{equation}
where the weights $b_i,a_{i,j}$ are chosen to reach an accuracy order. In particular, the \textit{explicit midpoint method} (RK2)
\begin{equation}
    \label{eq: rk2-method}
    \bm x_{n+1}=\bm x_n+\bm f\left(t_n+\frac{h}{2},\bm x_n+\frac{h}{2}\bm f(t_n,\bm x_n) \right)
\end{equation}
is a second-order explicit Runge-Kutta method.

\subsubsection{Discrete gradient (DG) methods}
Physical properties like energy-preservation are not in general respected when discretizing a port-Hamiltonian system using RK methods 
\cite{celledoni2017energy}. However, in Hamiltonian mechanics there is a rich theory of structure preserving integrators \cite{Hairer2006}. In particular, discrete gradient methods 
are a family of geometrical integrators that preserve the exact energy by construction \cite{quispel1996discrete}. A \textit{discrete gradient} $\overline{\nabla}H:\mathbb{R}^{n_x}\times\mathbb{R}^{n_x}\rightarrow\mathbb{R}^{n_x}$ 
is an approximation of the gradient of a function $H:\mathbb{R}^{n_x}\rightarrow\mathbb{R}$, satisfying the following two properties\footnote{In the literature, these two properties are usually written as 
\begin{enumerate}
    \item $\overline{\nabla}H(\bm x,\bm x')^T(\bm x'-\bm x)=H(\bm x')-H(\bm x)$
    \item $\overline{\nabla}H(\bm x,\bm x)=\nabla H(\bm x)$.
\end{enumerate}
}:
\begin{enumerate}
    \item $\overline{\nabla}H(\bm x,\delta\bm x)^T\delta\bm x=H(\bm x+\delta\bm x)-H(\bm x)$
    \item $\overline{\nabla}H(\bm x,\bm 0)=\nabla H(\bm x)$
\end{enumerate}
where $\delta\bm x=\bm x'-\bm x$. In this work, we consider the \textit{Gonzalez discrete gradient} method \cite{gonzalez1996time} (DG)
\begin{equation}
\label{eq: gonzalez-discrete-gradient}
\begin{split}
    \overline{\nabla}_{mid}H(\bm x,\delta\bm x)&=\nabla H\left(\bm x+\frac{1}{2}\delta\bm x\right)\\ &+\frac{H(\bm x+\delta\bm x)-H(\bm x)-\nabla H(\bm x+\frac{1}{2}\delta\bm x)^T\delta\bm x}{\delta\bm x^T\delta\bm x}\delta\bm x.
\end{split}
\end{equation} 
Note that this method is, in general, a second-order inverse-explicit integrator \cite{CELLEDONI2025134471} that becomes linearly implicit when 
the Hamiltonian satisfies $H(\bm x)=\frac{1}{2}\bm x^T\bm Q\bm x$.

\subsection{Neural ODEs and Jacobian regularization}
\textit{Neural Differential Ordinary Equations} (NODEs) were introduced by Chen et al.~\cite{chen2018} to define the evolution of a system's state using an ODE whose dynamics are modeled by a neural network:
\begin{equation}
\label{eq: neural-ode}
    \frac{d\bm x(t)}{dt}=\bm f_{\theta}(t,\bm x(t))
\end{equation}
where $\theta$ represents the parameters of the neural network. The forward pass of a NODE is defined as
\begin{equation}
    \bm x(t_{n+1})=\bm x(t_n)+\int_{t_n}^{t_{n+1}}\bm f_{\theta}(t,\bm x(t))dt
\end{equation}
where in practice the integration is approximated by any numerical method. \\

A well-documented challenge in the training of NODEs is the fact that the Jacobian of the learned dynamics $\bm f_{\theta}$ becomes poorly conditioned as the training progresses. Several studies have underlined this phenomenon. 
For instance, Dupont et al.~\cite{dupont2019augmentedneuralodes} reported that when NODEs are overfitted on the MNIST dataset~\cite{deng2012mnist}, the resulting flow may become so ill-conditioned that the numerical ODE solver 
is forced to take time steps smaller than machine precision, leading to numerical underflow and an increase in the number of function evaluations (NFE). Such pathological dynamics can also destabilize training and cause the loss 
to diverge. Motivated by these observations, Finlay et al.~\cite{finlay2020trainneuralodeworld} emphasized the importance of explicitly constraining the learned vector field. They noted that, in the absence of such constraints, 
the learned dynamics may exhibit poor conditioning, which in turn degrades numerical integration and degrades training performance. To address this issue, they proposed regularizing the Jacobian of $\bm f_{\theta}$ using its 
Frobenius norm. In this direction, Josias et al.~\cite{josias2022} proposed instead to regularize the Jacobian condition number, arguing that it reduced NFE without a significant loss in accuracy and controlling at the same time 
the Jacobian norm. Focusing on Jacobian regularization enables a connection to sensitivity analysis found in neural network literature, where also the spectral norm regularization has been employed~\cite{yoshida2017spectralnormregularizationimproving}
\cite{miyato2018spectralnormalizationgenerativeadversarial}.

\subsection{Port-Hamiltonian neural networks}
In the context of learning Hamiltonian ODEs, Greydanus et al. \cite{greydanus2019} introduced \textit{Hamiltonian Neural Network} (HNN), modifying the usual NODEs framework by parameterizing the Hamiltonian function of a given conservative physical system. 
Once the Hamiltonian is parameterized, \eqref{eq: hamilton-equation} is leveraged in the loss function. This work led to many others trying to generalize the framework to learn more general physical systems. For example, Sosanya et al.~\cite{greydanus2022} 
generalized the approach of HNN to non-conservative systems using Helmholtz's decomposition theorem to parameterize the dissipation potential. Several works integrating port-Hamiltonian systems theory into neural networks are found in literature. 
Most of them implement the input-state-output representation\footnote{In the literature, this equation adopts the convention that $\overline{\bm u}^T\overline{\bm y}=\overline{P}_{ext}$ is positive when given to the system by the external components, 
so that $\overline{P}_{ext}=-P_{ext}$ and that $(\overline{\bm u},\overline{\bm y})=(\bm u,-\bm y)$ or $(-\bm u,\bm y)$, contrary to conventions used in \eqref{eq:phs_S_intro} and all this paper}~\cite{van2014port}
\begin{equation}
    \label{eq:symplectic-like formulation phs}
    \begin{cases}
        \dot{\bm x}=(\bm J(\bm x)-\bm R(\bm x))\nabla_{\bm x}H(\bm x)+\bm G(\bm x)\overline{\bm u} \\
        \overline{\bm y}=\bm G^T(\bm x)\nabla H(\bm x)
    \end{cases}
\end{equation}
for designing the computation graph of the dynamics. In this context, Desai et al. \cite{desai2021} introduce the \textit{Port-Hamiltonian Neural Network} (PHNN) to learn damped and controlled systems. Zhong et al. \cite{zhong2020} use this formulation to model 
conservative systems with external control introducing the \textit{Symplectic ODE-Net} (SymODEN). In another work, the same authors generalize the framework of SymODEN to learn also the dissipation~\cite{zhong2020_workshop}. Cherifi et al.~\cite{cherifi2025nonlinearporthamiltonianidentificationinputstateoutput} 
propose a framework based on input-state-output data to learn port-Hamiltonian systems described by \eqref{eq:symplectic-like formulation phs}. Roth et al. \cite{roth2025stableporthamiltonianneuralnetworks} introduce \textit{Stable Port-Hamiltonian Neural Networks} 
(sPHNN) to learn dynamical systems with a single equilibrium under stability guarantees. A similar work to the one presented in this article, but based on pseudo-Hamiltonian formalism, is found in \cite{eidnes2023}. These approaches learn the dynamics through 
\eqref{eq:symplectic-like formulation phs} and impose \textit{a priori} knowledge, either on the expression of the Hamiltonian \cite{zhong2020, zhong2020_workshop, cherifi2025nonlinearporthamiltonianidentificationinputstateoutput, roth2025stableporthamiltonianneuralnetworks,chen2019} 
or on the way dissipation affects the system \cite{eidnes2023, desai2021, cherifi2025nonlinearporthamiltonianidentificationinputstateoutput, roth2025stableporthamiltonianneuralnetworks}. As in any other NODE framework, \eqref{eq:symplectic-like formulation phs} is discretized, 
during or after the training, using a numerical method so that there is no preservation of the power balance and no guarantee of stability. Generally, this has been done either by high-order RK methods~\cite{greydanus2019,greydanus2022,desai2021,cherifi2025nonlinearporthamiltonianidentificationinputstateoutput,roth2025stableporthamiltonianneuralnetworks}, 
but also via symplectic integrators~\cite{chen2019, zhu2020deephamiltoniannetworksbased, dipietro2020sparsesymplecticallyintegratedneural, xiong2022nonseparablesymplecticneuralnetworks, choudhary2025learninggeneralizedhamiltoniansusing}, although the impact on the 
accuracy of one or the other approach has not been thoroughly studied. To the best of our knowledge, the use of discrete gradients methods has not yet been described in the PHNN literature.
\section{Port Hamiltonian Models}\label{sec:Port-Hamiltonian-Model}

\subsection{Port-Hamiltonian formulations}\label{sec:PHS-formulations}

Let $\bm x\in\mathcal{X}\subset\mathbb{R}^{n_x}$ be the state of $\mathcal{S}$ with an 
associated Hamiltonian $H:\mathcal{X}\rightarrow\mathbb{R}$ and $\bm u,\bm y\in\mathbb{R}^{n_u}$ 
be the input and output of the system, respectively. We consider three formulations in 
this work: 

\begin{itemize}

\item[(i)] Semi-explicit PH-Differential-Algebraic-Equations (PH-DAE):
\begin{equation}
\label{eq:phs_S}
    \underbrace{
    \begin{bmatrix} \dot{\bm x} \\ \bm w \\ \bm y \end{bmatrix}
    }_{\text{flows~}\mathcal{F}_i}
= \bm S
\underbrace{
\begin{bmatrix} \nabla H(\bm x) \\ \bm z(\bm w) \\ \bm u \end{bmatrix}
}_{\text{efforts~}\mathcal{E}_i},
\end{equation}
where $\bm S=-\bm S^{T}\in\mathbb{R}^{(n_x+n_w+n_u)\times(n_x+n_w+n_u)}$ and 
$\bm z(\bm w):\mathbb{R}^{n_w}\rightarrow\mathbb{R}^{n_w}$ is the resistive structure of $\mathcal{S}$. 
The skew-symmetry of $\bm S$ guarantees energy conservation and the resistive property of $\bm z(\bm w)$ 
is given by $\bm z(\bm w)^T\bm w \geq 0$, which guarantees passivity.

\item [(ii)] Input-state-output PHS with feedthrough~\cite{van2000l2}:
\begin{equation}
    \label{eq:phs_JR}
    \underbrace{
    \begin{bmatrix} \dot{\bm x} \\ \bm y \end{bmatrix}
    }_{\text{flows~}\mathcal{F}_{ii}}
    = 
\big(\bm J- \bm R(\bm x, \bm u)\big) 
\underbrace{
\begin{bmatrix} \nabla H(\bm x) \\ \bm u \end{bmatrix}
}_{\text{efforts~}\mathcal{E}_{ii}},
\end{equation}
where $\bm J$ and $\bm R(\bm x,\bm u)\in\mathbb{R}^{(n_x+n_u)\times(n_x+n_u)}$ 
satisfy $\bm J=-\bm J^{T}$ and $\bm R=\bm R^{T}\succeq0$. The skew-symmetry of 
$\bm J$ accounts for conservative connections and the positive semi-definiteness 
of $\bm R(\bm x,\bm u)$ guarantees passivity. This formulation can be retrieved 
from (i) in particular when $\bm S_{ww}=0$ and $\bm S_{wx},\bm S_{wu}$ do not depend on $\bm w$. 
It is also an extension of \eqref{eq:symplectic-like formulation phs} for systems 
with direct feed-through.

\item[(iii)] Skew-symmetric gradient PH-DAE: If the function $\bm z(\bm w)=\partial_{\bm w} Z(\bm w)$ is 
derived from a potential $Z$ (often referred to as the Rayleigh potential),
the system (\ref{eq:phs_S}) rewrites as the skew-gradient system~\cite{muller2018power}
\begin{equation}
    \label{eq:skew-gradient-system}
    \underbrace{
    \begin{bmatrix} \dot{\bm x} \\ \bm w \\ \bm y \end{bmatrix}
    }_{\text{flows~}\mathcal{F}_{iii}}
    = \bm S \, \underbrace{\nabla F\!\!\left(\! \begin{bmatrix}
        \bm x \\ \bm w \\ \bm u
    \end{bmatrix}
    \!\right)}_{\text{efforts~}\mathcal{E}_{iii}},
\end{equation}
where $\bm S=-\bm S^{T}\in\mathbb{R}^{(n_x+n_w+n_u)\times(n_x+n_w+n_u)}$ and $F= H(\bm x)+Z(\bm w) + \frac{\bm u^T \bm u}{2}$.
\end{itemize}
Only the formulations (i) and (ii) are considered in the experiments. The skew-symmetric gradient PH-DAE 
is presented as a generalization of the formulation used by Sosanya et al.~\cite{greydanus2022}. 
The PHS formulations (i-iii) are passive (see \ref{appendix-a-passive-power-balance}).

\subsection{Physical constraints enforcement}\label{sec:physical_constraints} 
The general class of physical systems considered in this work is such that the Hamiltonian $H$ is $\mathcal{C}^1$-regular positive definite. 
In addition, we consider the subclass of Hamiltonians of the form
\begin{equation} 
\label{eq:quadratic_H}
H(\bm x)=\frac{1}{2}\bm x^T\bm Q(\bm x)\bm x, 
\end{equation}
where $\bm Q(\bm x)$ is a $\mathcal{C}^1$-regular symmetric positive definite matrix function. As already mentioned, The PHS formulations (i)-(ii) 
satisfy two types of constraints by construction:
\begin{align}
\label{eq:physical_constraints}
\begin{cases} 
    \bm S \text{ (or } \bm J) \text{ skew-symmetric} &\text{(energy conservation)}   \\
    \bm z(\bm w)^T\bm w \geq0, \forall \bm w \text{ (or } \bm R \succeq 0) &\text{(passive laws)}
\end{cases}
\end{align}
In the PH-DAE formulation, we consider the subclass of dissipation function of the form:
\begin{equation} 
\label{eq:z_form}
\bm z(\bm w)=\bm \Gamma(\bm w)\bm w 
\end{equation}
where $\bm \Gamma(\bm w)\in\mathbb{R}^{n_w\times n_w}$ is $\mathcal{C}^0$-regular positive semidefinite matrix function that admits a symmetric/skew-symmetric 
decomposition  $\bm\Gamma(\bm w)=\bm \Gamma^{skew}(\bm w)+\bm\Gamma^{sym}(\bm w)$ with $\bm \Gamma^{skew}=\frac{1}{2}(\bm \Gamma-\bm\Gamma^T)$ and $\bm\Gamma^{sym}=\frac{1}{2}(\bm\Gamma+\bm\Gamma^T)\succeq0$. The resistive property is satisfied as
\begin{equation}
\label{eq: proof-resistive-property}
\begin{split}
    \forall \bm w,\bm z(\bm w)^T\bm w &=\bm w^T\Gamma(\bm w)^T\bm w\\&=\bm w^T\Gamma^{sym}(\bm w)^T\bm w\overset{\Gamma^{sym}\succeq0}{\geq}0
\end{split}
\end{equation}
Note that the classes of systems described by equations \eqref{eq:quadratic_H} to \eqref{eq: proof-resistive-property} cover a large spectrum of physical systems and remain fairly general.

\subsection{Reparameterization of constraints}

We now reparameterize the matrix constraints using the following properties:
\begin{itemize}
    \item any symmetric positive semidefinite (resp. definite) matrix $\bm M^{sym}$ can be written in the form $\bm M^{sym} = \bm L^T \bm L$ where $\bm L$ is a 
    lower triangular matrix with positive (resp. strictly  positive) diagonal coefficients (Cholesky factorization~\cite{press2007numerical}),
    \item any skew-symmetric matrix $\bm M^{skew}$ can be written in the form $\bm M^{skew} = \bm K - \bm K^T$ where $\bm K$ is a strictly 
    lower triangular matrix. 
\end{itemize} 

Then, the physical constraints \eqref{eq:physical_constraints} are naturally satisfied considering the following reparametrizations:
\begin{align}
&\bm Q(\bm x) = \bm L_Q^T(\bm x) \bm L_Q(\bm x)  \label{eq:physical_constraints_H}\\ 
&\bm \Gamma(\bm w) = \bm L_\Gamma^T(\bm w)\bm L_\Gamma(\bm w)+\bm K_\Gamma(\bm w) - \bm K_\Gamma^T(\bm w) \label{eq:physical_constraints_Gamma} \text{ or }\\
&\bm R(\bm x, \bm u) = \bm L_R^T(\bm x, \bm u) \bm L_R(\bm x, \bm u) \label{eq:physical_constraints_R} 
\end{align}
where $\bm L_{Q, R, \Gamma}$, resp. $\bm K_\Gamma$, are lower, resp. strictly lower, triangular matrices.

\begin{table*}[h!]
    \centering
    
    \begin{tabular}{|l|c|c|c|}
    \hline
    System
    & Harmonic oscillator & Duffing oscillator  & Self-sustained oscillator
    \\
     & (linear) & (nonlinear)    & (nonlinear) 
    \\ \hline
    $
    \begin{array}{l}
    \bm x   = \begin{bmatrix}
      q \\ p  
    \end{bmatrix}
    \\
    \bm u
    \end{array}
    $
    & \multicolumn{3}{c|}{
    $\begin{array}{rll}
          q : & \text{centered position (elongation)} & \text{[m]}
          \\
          p : & \text{momentum} & \text{[Kg.m.s$^{-1}$]}
          \\
          u=f : & \text{force (applied or exterior)} & \text{[N]}
    \end{array}$
    }
    \\
    \hline
    \rule{0pt}{0.5cm} $H(\bm x)$     
    & $\displaystyle \frac{p^2}{2m}+\frac{k q^2}{2}$ 
    & $\displaystyle \frac{p^2}{2m}+\frac{k_1 q^2}{2}+\frac{k_3 q^4}{4}$
    & $\displaystyle \frac{p^2}{2m}+\frac{k q^2}{2}$ 
    \\
    
    \hline
    \rule{0pt}{0.8cm}
    $\bm J$
    & \multicolumn{2}{c|}{
    $\bm J = 
    \left[ \begin{array}{cc|c} 0 & 1 & 0 \\ -1 & 0 & -1 \\ \hline 0 & 1 & 0 \end{array}\right]$} & $\bm J = 
    \left[ \begin{array}{cc|c} 0 & -1 & 0 \\ 1 & 0 & 0\\ \hline 0 & 0 & 0 \end{array}\right]$ \\
    
    \hline
    \rule{0pt}{1cm}
    $\bm R$ & \multicolumn{2}{c|}{\centering $\bm R=\alpha\, \left[ \begin{array}{cc|c} 0 & 0 & 0 \\ 0 & 1 & 0 \\ \hline 0 & 0 & 0 \end{array}\right]$~with 
    $\alpha >0$~[N$/ms^{-1}$]
    } & {\centering $\bm R=\Gamma(w)\left[ \begin{array}{cc|c} 1 & 0 & 1 \\ 0 & 0 & 0 \\ \hline 1 & 0 & 1 \end{array}\right]$~\text{with}~$\Gamma(w)>0$~[$m/s$N]
    } \\
    
    \hline
    \rule{0pt}{1cm}
    $\bm S$ & \multicolumn{2}{c|}{\centering 
    $\bm S = 
    \left[ \begin{array}{cc|c|c} 0 & 1 & 0 & 0\\ -1 & 0 & -1 & -1\\ \hline 0 & 1 & 0 & 0 \\ 0 & 1 & 0 & 0\end{array}\right]$
    } & {\centering 
    $\bm S = 
    \left[ \begin{array}{cc|c|c} 0 & -1 & 1 & 0\\ 1 & 0 & 0 & 0\\ \hline -1 & 0 & 0 & -1 \\ 0 & 0 & 1 & 0\end{array}\right]$
    } 
    \\
    \hline
    $\Gamma(w)$ & \multicolumn{2}{c|}{$\alpha$} & $aw^2+bw+c$ \\ \hline 
    $z(w)$ & 
    \multicolumn{2}{c|}{$z(w)=\alpha \, w$}
    & $z(w)=aw^3+bw^2+cw$ 
    \\
    \hline
    $Z(w)$ & 
    \multicolumn{2}{c|}{$Z(w) = \frac{\alpha w^2}{2}$}
    & $Z(w)=\frac{aw^4}{4}+\frac{bw^3}{3}+\frac{cw^2}{2}$
    \\
    \hline
    \rule{0pt}{0.5cm}
    $\bm J_{\bm f}(\bm x)$ & $\begin{bmatrix}
        0 & \frac{1}{m} \\ -k & -\frac{\alpha}{m}
    \end{bmatrix}$ & $\begin{bmatrix}
        0 & \frac{1}{m} \\ -k_1-3k_3q^2 & -\frac{\alpha}{m}
    \end{bmatrix}$ & $\begin{bmatrix}
        -\Gamma'(w)kq-\Gamma(w)k & -\frac{1}{m} \\
        k & 0
    \end{bmatrix}$ \\[3mm]
    \hline 
    \rule{0pt}{0.5cm}
    $\|\bm J_f(\bm x) \|_2$ & \multicolumn{3}{c|}{${\sqrt{\lambda_{max}((\bm J_{\bm f}(\bm x))^T\bm J_{\bm f}(\bm x))}}$~(see the exact details on the \ref{ap: jacobian-quantities})} \\[1.5mm]
    \hline
    \rule{0pt}{0.5cm}
    $\kappa(\bm J_{\bm f}(\bm x))$ & \multicolumn{3}{c|}{${\frac{\sigma_{max}(\bm J^{\bm f}(\bm x))}{\sigma_{min}(\bm J^{\bm f}(\bm x))}=\sqrt{\frac{\lambda_{max}((\bm J_{\bm f}(\bm x))^T\bm J_{\bm f}(\bm x))}{\lambda_{min}((\bm J_{\bm f}(\bm x))^T\bm J_{\bm f}(\bm x))}}}$~(see the exact details on the \ref{ap: jacobian-quantities})} \\
    \hline 
    $\rho(\bm J_{\bm f}(\bm x))$ & \multicolumn{3}{c|}{$1,~\text{if the system is underdamped}$}  \\
    \hline 
    \end{tabular}
    \caption{Different port-Hamiltonian structural elements, Jacobian matrices and related quantities for the considered oscillators.}
    \label{tab:PHS-examples-and-formulations}
\end{table*}

\subsection{Oscillatory physical examples}
Three oscillating systems that satisfy \eqref{eq:quadratic_H}-(\ref{eq:z_form}) are considered: 
\begin{itemize}
    \item the harmonic oscillator, which is 
quadratic in energy, with a linear dissipation; 
\item the Duffing oscillator, which is non-quadratic in energy, with a linear dissipation; 
\item a self-sustained oscillator, 
which is quadratic in energy, with a nonlinear dissipation. 
\end{itemize}
This last physical system is of particular dynamical interest as it is designed to self-oscillate when combined with an adapted constant input, leading to a stable limit cycle~\cite{helie2025}. Table \ref{tab:PHS-examples-and-formulations} shows the different port-Hamiltonian structural elements for each of these systems.

\section{Port-Hamiltonian neural networks}
\label{sec:phnn-models}

\subsection{Port-Hamiltonian neural network models}
\label{subsec:phnn-models-neural-networks}

A port-Hamiltonian neural network (PHNN) parameterizes the right-hand side of \eqref{eq:phs_S}-(\ref{eq:phs_JR}), 
modeling the Hamiltonian and the dissipative terms with two distinct neural networks. In this work, we assume 
that the interconnection matrices $\bm{J}$ and $\bm{S}$, are given a priori. The Hamiltonian function is 
parameterized identically for each formulation by a neural network $H_{\theta_H}(\bm x)$, and its gradient is 
derived by auto-differentiation as proposed in the seminal HNN \cite{greydanus2019}. The implementation of the 
dissipative term depends on the formulation considered (i or ii), which yields two different PHNN architectures:

\begin{itemize}
    \item[(i)] PHNN-S models \eqref{eq:phs_S}, where the dissipation function $\bm z(\bm w):\mathbb{R}^{n_w}\rightarrow\mathbb{R}^{n_w}$ 
    is parameterized by a neural network $\bm z_{\theta_z}(\bm w)$, implementing the dynamic function:
    
        \begin{equation}
        \label{eq:f_theta-PHS-DAF}
            \bm f_{\theta}(\bm x,\bm u)=\bm f^{S}_{\theta}(\bm x,\bm u)=\bm S\begin{bmatrix}
                \nabla H_{\theta_H}(\bm x) \\ \bm z_{\theta_z}(\bm w) \\ \bm u
            \end{bmatrix}
        \end{equation}
        
    \item[(ii)] PHNN-JR models \eqref{eq:phs_JR}, where the coefficients of the dissipation matrix $\bm R(\bm x,\bm u)$ 
    are parameterized by the outputs a neural network $\bm R_{\theta_R}(\bm x,\bm u)$, implementing the dynamic function:
    
        \begin{equation}
        \label{eq:f_theta-PHS-JR}
            \bm f_{\theta}(\bm x,\bm u)=\bm f^{JR}_{\theta}(\bm x,\bm u)=(\bm J-\bm R_{\theta_R}(\bm x,\bm u))\begin{bmatrix}
                \nabla H_{\theta_H}(\bm x) \\ \bm u
            \end{bmatrix}
    \end{equation}

\end{itemize}

\begin{figure*}[h!]
    \centering
        \subfigure[Architecture of PHNN-S]{\includegraphics[width=0.48\textwidth]{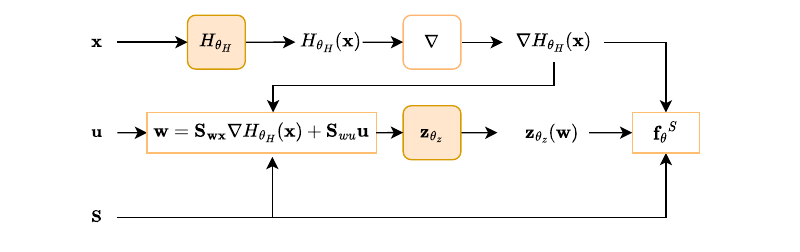}}
        \hfill 
        \subfigure[Architecture of PHNN-JR]{\includegraphics[width=0.48\textwidth]{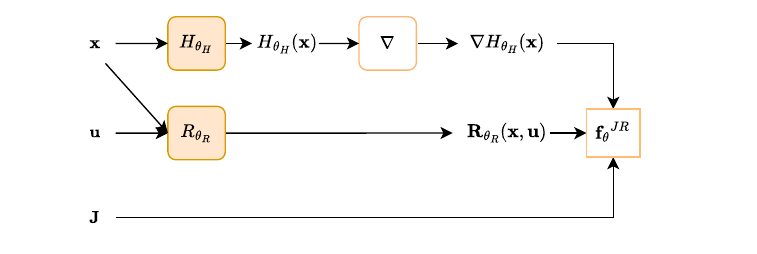}}
        \caption{Architecture of the two PHNN models considered in this work. White boxes with orange contour denote fixed algebraic operations whereas orange boxes indicate the trainable parameters.}
        \label{fig:phnn-architectures}
\end{figure*}

Figure \ref{fig:phnn-architectures} shows the architecture of the two different PHNN models according to each of the PHS formulations.

\subsection{Physical constraints enforcement in neural networks}

According to Section \ref{sec:physical_constraints}, the PHS model is passive if the Hamiltonian and the dissipative terms take the form of \eqref{eq:physical_constraints_H} and \eqref{eq:physical_constraints_Gamma} 
or \eqref{eq:physical_constraints_R}. The non-zero coefficients of a $n \times n$ lower (resp. strictly lower) triangular matrix can be parameterized by the $n(n+1)/2$ (resp. $n(n-1)/2$) outputs of a neural network 
$\bm L_{\theta_L}$ (resp. $\bm {K}_{\theta_K}$), implementing the functions:
\begin{equation}
\begin{cases} 
H_{\theta_H}(\bm x)=\frac{1}{2}\bm x^T\bm L_{\theta_{L_H}}^T(\bm x)\bm L_{\theta_{L_H}}(\bm x)\bm x\\
\bm z_{\theta_z}(\bm w)=\left(\bm L_{\theta_{L_z}}^T(\bm w)\bm L_{\theta_{L_z}}(\bm w)+\bm K_{\theta_{K_z}}(\bm w) - \bm K_{\theta_{K_z}}^T(\bm w)\right)\bm w \\
\bm R_{\theta_{R}}(\bm x,\bm u)=\bm L_{\theta_{L_R}}^T(\bm x,\bm u)\bm L_{\theta_{L_R}}^T(\bm x,\bm u)
\end{cases}
\end{equation}
This parametrization is widely used in the literature~\cite{zhong2020,cherifi2025nonlinearporthamiltonianidentificationinputstateoutput,roth2025stableporthamiltonianneuralnetworks,
schwerdtner2021porthamiltonianidentificationnoisyfrequency,schwerdtner2022structurepreservingmodelorderreduction}. Table \ref{tab: inputs, outputs, train loss and learned objects} shows the fixed and learned objects 
for the port-Hamiltonian networks based on formulations (i) and (ii).
\begin{table*}[h!]
    \renewcommand{\arraystretch}{2} 
    \centering
    \begin{tabular}{|c|c|c|}
    \hline \centering \diagbox[width=4.2cm]{Objects}{Formulation}
 &
    
     (i)\hspace{0.2cm} $\begin{bmatrix}
        \bm{\Dot{x}} \\ \bm{y}
        \end{bmatrix}=(\bm{J}-\bm{R}(\bm{x},\bm{u}))\begin{bmatrix}
        \nabla H(\bm{x}) \\ \bm{u}
        \end{bmatrix}$ &
        
        (ii) \hspace{0.2cm}$\begin{bmatrix}
        \bm{\Dot{x}} \\ \bm{w} \\ \bm{y}
        \end{bmatrix}=\bm{S}\begin{bmatrix}
        \nabla H(\bm{x}) \\ \bm{z}(\bm{w}) \\ \bm{u}
        \end{bmatrix}$  \\
        \hline
        $H(\bm{x})$ & \multicolumn{2}{|c|}{$ H_{\theta_H}(\bm x)=\frac{1}{2}\bm x^T\bm L_{\theta_{L_H}}(\bm x)^T\bm L_{\theta_{L_H}}(\bm x)\bm x$} \\
        
        \hline
        Dissipation  &   $\bm R_{\theta_R}(\bm w)=\bm L_{\theta_{L_R}}(\bm x,\bm u)^{T}\bm L_{\theta_{L_R}}(\bm x,\bm u)$ &  $\begin{aligned} \bm z_{\theta_z}(\bm w)&=\bm \Gamma_{\theta_{\Gamma}}(\bm w)\bm w \\ \bm \Gamma_{\theta_{\Gamma}}(\bm w)&=(\bm L_{\theta_{L_z}}^T(\bm w)\bm L_{\theta_{L_z}}(\bm w))\\ &+(\bm K_{\theta_{K_z}}(\bm w)-\bm K_{\theta_{K_z}}^T(\bm w))\end{aligned}$  \\
        \hline
         \makecell{Interconnection \\ matrix}  & $\bm J\in\mathbb{R}^{(n_x+n_u)\times (n_x+n_u)}$ & $\bm S\in\mathbb{R}^{(n_x+n_w+n_u)\times (n_x+n_w+n_u)}$ \\
        \hline
    \end{tabular}
    \caption{Fixed and learned objects for the different port-Hamiltonian networks based on formulations (i) and (ii).}
    \label{tab: inputs, outputs, train loss and learned objects}
\end{table*}

\subsection{Continuous- vs. discrete-time models}

A \textit{continuous model} $\bm f_{\theta}(\bm x_t,\bm u_t)$ represents the 
field $\dot{\bm x}_t$. During training, it learns an estimate of $\dot{\bm x}_t$ by minimizing the loss 
\begin{equation}
\label{eq: continuous-loss}
    \mathcal{L}_{c}=\| \dot{\bm x}_t - \bm f_{\theta}(\bm x_t,\bm u_t) \|_2^2
\end{equation}
During inference, it is integrated as the right-hand side of the ODE using a numerical method. \\

A \textit{discrete model} $\bm g_{\theta,h}(\bm x_n,\bm u_n)$ represents the next state $\bm x_{n+1}$. During training, 
it learns an estimate of $\bm x_{n+1}$ by minimizing the loss
\begin{equation}
\label{eq: discrete-loss}
    \mathcal{L}_{d}=\left\| \frac{\bm x_{n+1} - \bm g_{\theta,h}(\bm x_n,\bm u_n)}{h}\right\|_1
\end{equation}
via backpropagation through a differentiable numerical ODE solver using a discretization step $h$. We introduce \eqref{eq: discrete-loss} inspired by Zhu et al.~\cite{zhu2022numericalintegrationneuralordinary}, 
as it was shown to have theoretical guarantees for NeuralODEs using explicit Runge Kutta methods (see \cite{zhu2022numericalintegrationneuralordinary}, Theorem 3.1). During inference, the discrete model is applied autoregressively to predict future states. Examples of discrete PHNN model learning framework can be found in literature \cite{zhong2020, zhong2020_workshop, neary2023, chen2019}.\\

\begin{figure*}[t!]
    \centering
    \includegraphics[width=0.6\linewidth]{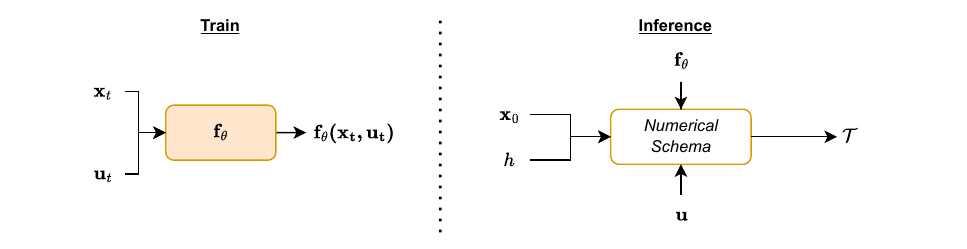}
    \caption{Training and inference diagram for the continuous models $\bm f_{\theta}$.}
    \label{fig:continuous models diagram}
\end{figure*} 

\begin{figure*}[t!]
    \centering
        \includegraphics[width=\linewidth]{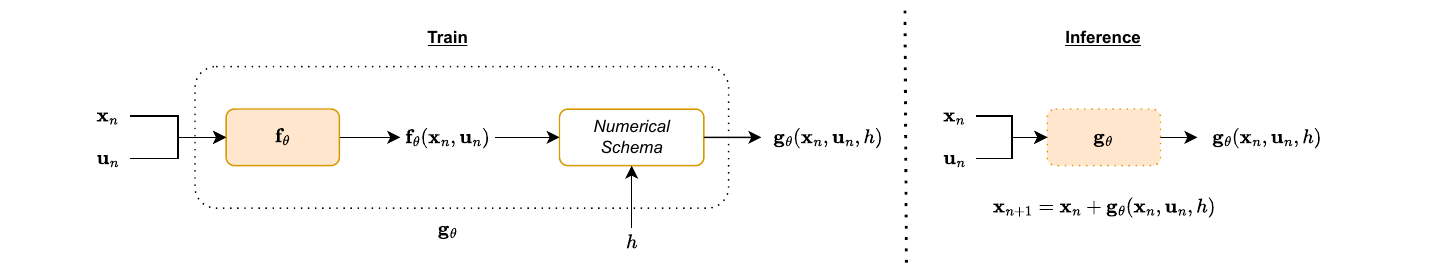}
        \caption{Training and inference diagram for the discrete models $\bm g_{\theta}$.}
        \label{fig:discrete models diagram}
    \end{figure*}

As it can be seen on Figures \ref{fig:continuous models diagram} and \ref{fig:discrete models diagram}, the neural network parameterizing the continuous- and the discrete-time models is the same. However, the ODE solver is used at different phases: during training and inference for discrete-time model, but only during 
inference for continuous-time model. In this work, we focus on discrete models, where the two different ODE solver schemes compared are shown in Table \ref{tab:numerical-schema}. This decision is based on the fact that in practice it is generally impossible to have access to state derivatives $\dot{\bm x}_t$. Following 
\eqref{eq: discrete-loss}, training discrete models does not require state derivative measurements and relies only on state measurements. Nevertheless, their performance during inference is limited to using the same discretization step $h$ as during training.
\begin{table}[h!]
\centering
\scalebox{0.7}{\begin{tabular}{ |c|c|c|c|} 
\hline
Accuracy order & Name & Characteristic & Power-balance \\
 \hline
2 & Explicit midpoint (RK2) & Explicit &  $\times$\\
2 & Discrete gradient (DG) & Inverse-explicit & $\checkmark$  \\
 \hline
 
\end{tabular}}
\caption{Accuracy order, name, characteristic and whether power-balance is respected for the chosen numerical schema.}
\label{tab:numerical-schema}
\end{table} 

\subsection{Jacobian regularization on PHNN}
Let $\bm{J}_{\bm f_{\theta}}(\bm{x},\bm u)$ denote 
the Jacobian of a PHNN, where $\bm{f}_{\theta}(\bm{x},\bm{u})$ corresponds to \eqref{eq:f_theta-PHS-DAF}-(\ref{eq:f_theta-PHS-JR}). 
Ideally, we would like $\bm f_{\theta}(\bm x,\bm u)$ not only to be  physically-consistent, which is already structurally guaranteed in 
the PHNNs, but also \textit{numerically smooth}. This means that, among the different physically-consistent solutions learned by
 our models, those that are ill-posed, ill-conditioned or highly stiff should be penalized through some kind of Jacobian 
regularization, as proposed in the NeuralODE literature~\cite{dupont2019augmentedneuralodes,finlay2020trainneuralodeworld,josias2022}. 
Following this approach, we soft-constraint the training loss $\mathcal{L}_d$ to avoid non-desirable numerical behaviors in the PHNNs. \\

We experiment with regularizing the spectral norm $\|\bm{J}_{\bm f_{\theta}}\|_2$, the condition number $\kappa(\bm{J}_{\bm f_{\theta}})$ 
or the stiffness ratio $\rho(\bm{J}_{\bm f_{\theta}})$ as defined in Section \ref{sec:dynamical_systems}. Penalizing higher spectral norm values promotes solutions with a lower Lipschitz 
constant whereas doing so with condition number rewards well-conditioned solutions. The rationale for introducing $\rho(\mathbf{J}_{\bm f_{\theta}})$ 
as a new regularization term stems from the observation that ill-conditioned Jacobians, often associated with a high NFE in adaptive numerical 
methods, are characteristic of stiff ODE systems. Thus, penalizing large values of $\rho(\bm{J}_{\bm f_{\theta}})$ discourages the training procedure 
from converging toward stiff dynamics, thereby promoting indirectly better-conditioned ODE solutions. During the training, these regularization 
quantities are obtained after evaluating the Jacobian at the input data from a given batch, making the same assumption as in~\cite{josias2022}: 
\textit{that regularizing dynamics at the input data can lead to regularized dynamics across the entire solution space}. The following loss functions 
with Jacobian regularization terms are introduced:
\begin{equation}
    \label{eq:condition-number-loss}
    \mathcal{L}_{CN}=\mathcal{L}_d+\lambda_{CN}\|\kappa(\bm{J}_{\bm f_{\theta}}(\bm{x},\bm u))\|^2_2
\end{equation}
\begin{equation}
    \label{eq:spectral-norm-loss}
    \mathcal{L}_{SN}=\mathcal{L}_d+\lambda_{SN}\|\bm{J}_{\bm f_{\theta}}(\bm{x},\bm u)\|_2
\end{equation}
\begin{equation}
    \label{eq:stiffness-ratio-loss}
    \mathcal{L}_{SR}=\mathcal{L}_d+\lambda_{SR}\|\rho(\bm{J}_{\bm f_{\theta}}(\bm{x},\bm u))-1\|^2_2
\end{equation}
We set $\lambda_{SN}=10^{-6}$, $\lambda_{CN}=10^{-6}$ and $\lambda_{SR}=10^{-4}$ in later experiments, and, as in~\cite{josias2022}, we add $10^{-6}$ 
to the denominator of $\kappa(\bm{J}_{\bm f_{\theta}})$ and $\rho(\bm{J}_{\bm f_{\theta}})$ to avoid underflow in the early stages of training.
\section{Experiments}\label{sec:experiments}

\subsection{Our questions}
Given a discrete trajectory $\mathcal{T}$ and its initial state $\bm x_0$, the objective of this work is to design 
PHNN models capable of accurately generating a trajectory $\tilde{\mathcal{T}}=\{ \tilde{\bm x}_n\}_{n\in \mathbb{N}^{+}}$, with $\tilde{\bm x}_0=\bm x_0$, 
that approximates $\mathcal{T}$ as closely as possible. Each of the PHNN models presented in this work can be characterized by the: 
\begin{enumerate}
    \item PHS formulation: PHNN-S or PHNN-JR.
    \item Type of model: Continuous or discrete.
    \item Numerical method: RK2 or DG (see Table \ref{tab:numerical-schema}).
    \item Number of trainable parameters: 800, 2k or 10k.
    \item Number of training points: 25, 100 or 400.
    \item Loss function: $\mathcal{L}_d$, $\mathcal{L}_{SN}$, $\mathcal{L}_{CN}$ or $\mathcal{L}_{SR}$.
\end{enumerate}
We set a \textit{baseline} NODE~\cite{chen2018} model and conduct several experiments to compare its performance to 
the proposed familiy of PHNNs models.  \\ 

\textit{Study I}: \textbf{Impact of the number of training points}. The first study focuses on how the different PHNN architectures 
compare at inference stage with different numerical methods for the systems in Table \ref{tab:PHS-examples-and-formulations}. For 
this study, experiments are carried out for 25, 100 and 400 training points and the smallest number of training parameters is 
considered ($\approx 800$). 

\vspace{\baselineskip}
\textit{Study II}: \textbf{Impact of the number of trainable parameters}. The second study uses the same combination of PHNN 
architectures and numerical methods but, opposite to Study I, the number of training points is fixed to 25 and the number of training 
parameters varies between \textit{small}($\approx 800$), \textit{medium}($\approx 4k$) and \textit{large}($\approx 20k$). Tables 
\ref{tab:reference-size-parameters-neural-networks}-\ref{tab:low-size-parameters-neural-networks} in \ref{appendix:implementation-details} 
detail the design choices of the neural network components for each model. In each case, the criteria is to have a similar number of 
parameters for each formulation, which enables a fair comparison between them.

\vspace{\baselineskip}
\textit{Study III}: \textbf{Impact of the Jacobian regularizations}. The third study focuses on how the different Jacobian 
regularizations in \eqref{eq:condition-number-loss}-(\ref{eq:stiffness-ratio-loss}) influence the inference performance of the models. 
In this case, the same combination of PHNN architectures, numerical methods, training parameters and training points as in Study I is considered.

\begin{table*}[t!]
\centering
\begin{tabular}{|c|c|c|c|}
\hline
System & \textbf{Harmonic Oscillator} & \textbf{Duffing Oscillator} & \multicolumn{1}{|c|}{\textbf{Self-sustained oscillator}}   \\ 
\hline
Intrinsic parameters & $(m,k)$ & $(m,k_1,k_3)$ & $(m,k)$  \\
\hline
$f_0$~[Hz] & \multicolumn{3}{|c|}{$1Hz$} \\
\hline

$m$~[kg] & \multicolumn{3}{|c|}{$0.16$} \\
\hline
$1/k,1/k_1$~[$N/m$] & \multicolumn{3}{|c|}{$0.16$} \\
\hline
$k_3$~[$N/m^3$] & - & $100k_1$ & -  \\
\hline
$z(w)$ & \multicolumn{2}{|c|}{$0.9w$~[N]} & $1.3w^3-4w^2+3w$~$[m/s]$ \\
\hline
\hline

Constant control $u$~[N] & \multicolumn{2}{|c|}{Such that the equilibrium point energy lies in $\mathcal{I}_{E_0}$} & Such that system reaches a limit cycle \\
\hline 
$\mathcal{I}_{E_0},\mathcal{I}_{E_{eq}}$~[J] & \multicolumn{3}{|c|}{$[0.1,1]$}  \\
\hline
${sr}_{gen}$~[Hz] & \multicolumn{3}{|c|}{$400f_0$} \\
\hline
Generation numerical method & \multicolumn{3}{|c|}{Gonzalez discrete gradient \eqref{eq: gonzalez-discrete-gradient}} \\
\hline
\hline 
$\alpha$ & \multicolumn{3}{|c|}{$0.31$}  \\
\hline
$\beta$ & \multicolumn{3}{|c|}{5}  \\
\hline
${D}_{train}$~[s] & \multicolumn{3}{|c|}{$\alpha T_0$} \\
\hline
${D}_{infer}$~[s] & \multicolumn{3}{|c|}{$\beta T_0$} \\
\hline 
${sr}_{train},{sr}_{infer}$~[Hz] & \multicolumn{3}{|c|}{$100f_0$} \\
\hline \hline 
$N_{train}$ & \multicolumn{3}{|c|}{$[25,100,400]$} \\
\hline
Batch size & \multicolumn{3}{|c|}{$64$} \\
\hline 
$N_{eval}$ & \multicolumn{3}{|c|}{$2500$} \\
\hline 
$N_{infer}$ & \multicolumn{3}{|c|}{$100$} \\
\hline 
\end{tabular}
\caption{Implementation hyperparameters for the different experiments.}
\label{tab:physical_parameters_}
\end{table*} 

\begin{figure*}[h!]
    \centering
        \subfigure[Harmonic oscillator]{\includegraphics[width=0.32\textwidth]{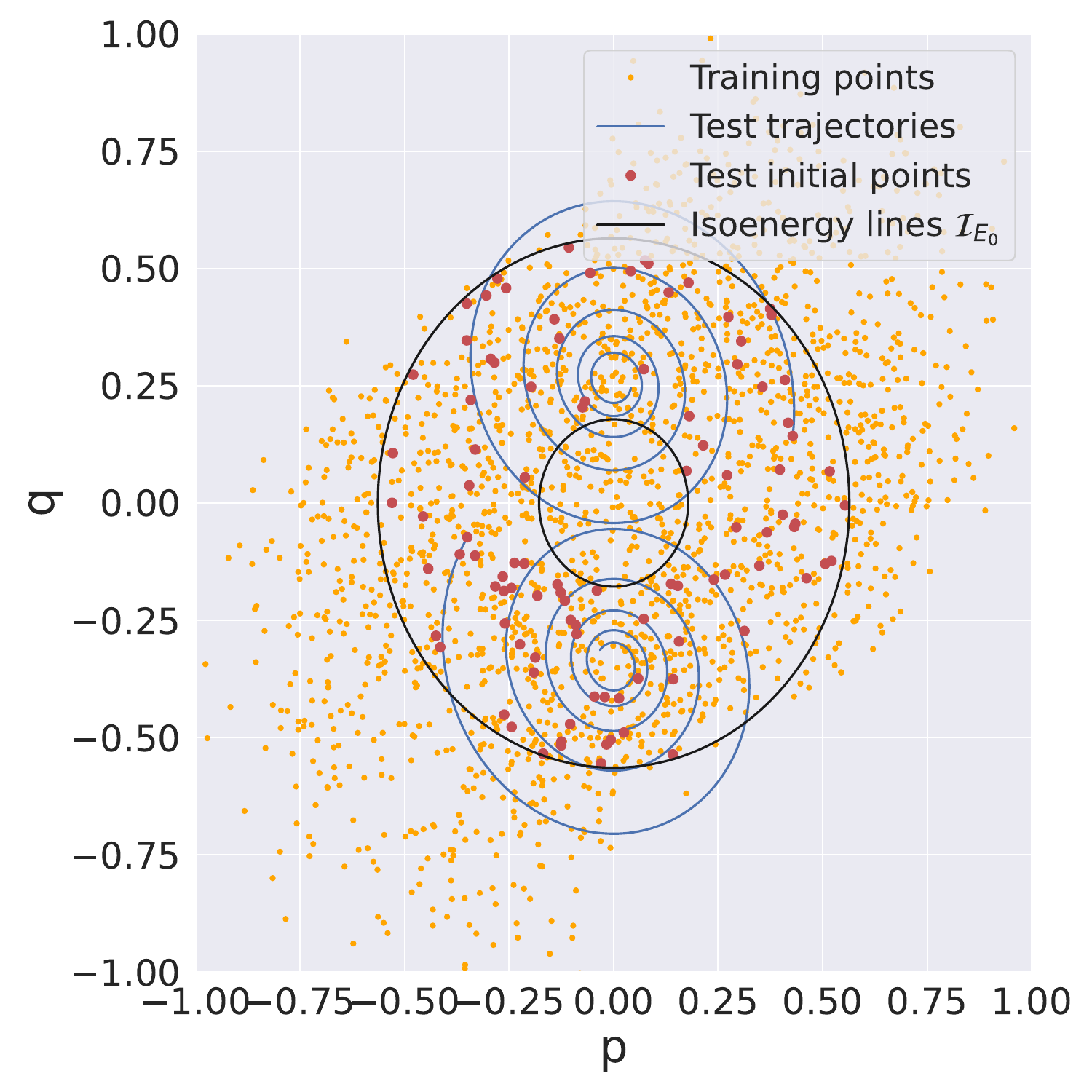}}
        \hfill 
        \subfigure[Duffing oscillator]{\includegraphics[width=0.32\textwidth]{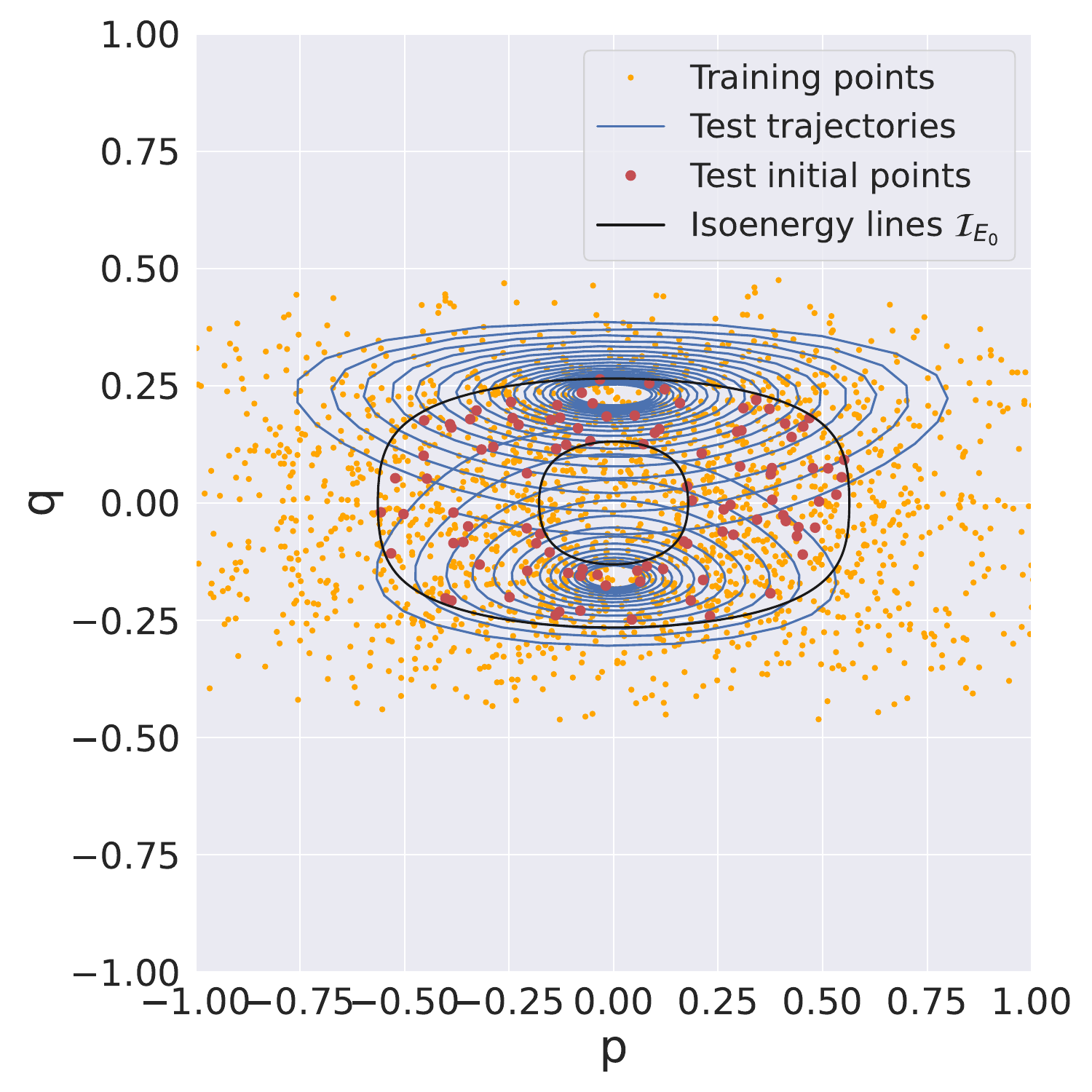}}
        \hfill 
        \subfigure[Self-sustained oscillator]{\includegraphics[width=0.32\textwidth]{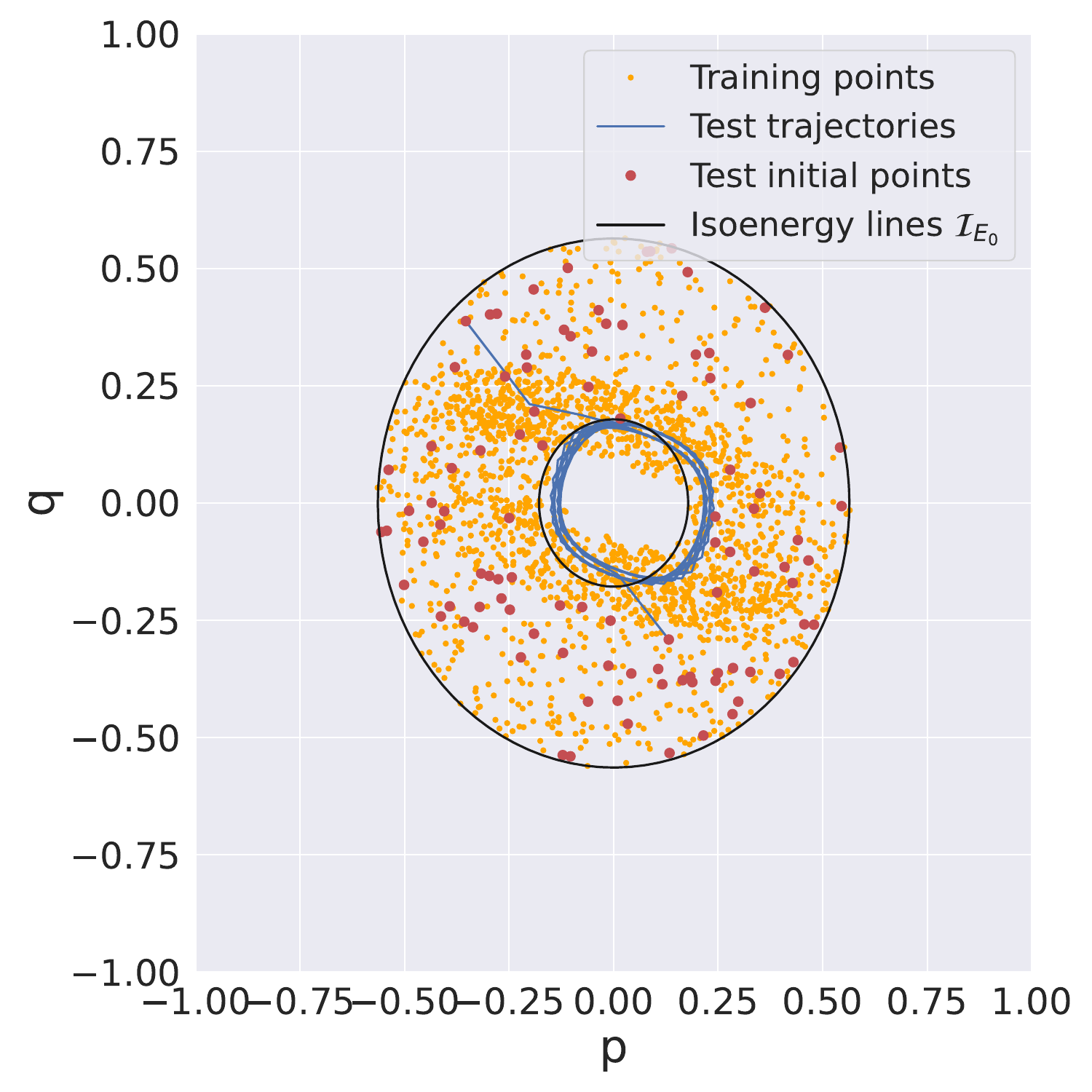}}
        \caption{Training points, test initial points and two complete test trajectories for each of the three oscillatory systems. Note that in the case of the harmonic 
        and Duffing oscillator, the applied control shifted the equilibrium point from $(p,q)=(0,0)$ whereas in the case of the self-sustained oscillator, it stabilizes 
        the trajectories  in a limit cycle.}
        \label{fig:sampling-distribution}
    \end{figure*}

\subsection{Implementation details}
Table \ref{tab:physical_parameters_} shows the parameters considered for the generation of trajectories for each physical system as well 
as the training and inference parameters considered for the neural network experiments, whereas Figure \ref{fig:sampling-distribution} shows 
the distribution of the training and test points as well as two complete test trajectories for each oscillatory system. \\

\textbf{Generation of synthetic data}. The dataset $\mathcal{T}$ consists of $N_{traj}=12500$ trajectories. Each trajectory is generated 
synthetically according to the following criteria:
\begin{enumerate}
    \item[i] \textbf{Initial condition}: We fix $E_{min},E_{max}$, and we sample an initial condition $\bm x_0$ such that
    \begin{equation}
        H(\bm x_0)\in\mathcal{I}_{E_0}=[E_{min},E_{max}]
    \end{equation}
    with $E_{min},E_{max}\in\mathbb{R}^{+}$ (see Table \ref{tab:physical_parameters_} for numerical values).
    \item[ii] \textbf{Control}: For the harmonic and Duffing oscillators, external control is applied as a constant force $\bm u$ so that \begin{equation}
        H(\bm x^{*})\in\mathcal{I}_{E_{eq}}= [E_{eq}^{min},E_{eq}^{max}]
    \end{equation}
    where $\bm x^{*}$ is the equilibrium point of the system. In control theory, this is usually referred to as \textit{potential energy shaping}~\cite{ortega2001}. For the self-sustained oscillator, the external constant control $\bm u$ is applied so that the system stabilizes in a limit cycle around $\bm x^{*}$ 
    (see Table \ref{tab:physical_parameters_} for numerical values).
\end{enumerate}

 Further details about the sampling of the initial conditions and the control design can be found in \ref{appendix:control-design}. As for the 
 intrinsic parameters of each system, they are chosen to satisfy that the natural frequency $f_0$ is $1Hz$. For the harmonic and Duffing oscillator, a linear 
 dissipation function $z(w)=cw$ is considered. In this case, $c$ is chosen so that the damped harmonic oscillator has dissipated $99\%$ of its energy in $D=5T_0=5s$, 
 where $T_0=1/f_0$ is the natural period. The same value of $c$ is considered for the Duffing oscillator. As for the self-sustained oscillator, the considered nonlinear 
 dissipation function is $z(w)=1.3w^3-4w^2+3w$. In each physical system, the numerical scheme considered for the generation of the trajectories is the 
 \textit{Gonzalez discrete gradient}~\eqref{eq: gonzalez-discrete-gradient} for a duration $D$ and a sampling rate ${sr}_{gen}=400f_0$. \\

 \begin{figure*}[h!]
    \centering
    \includegraphics[width=0.8\linewidth]{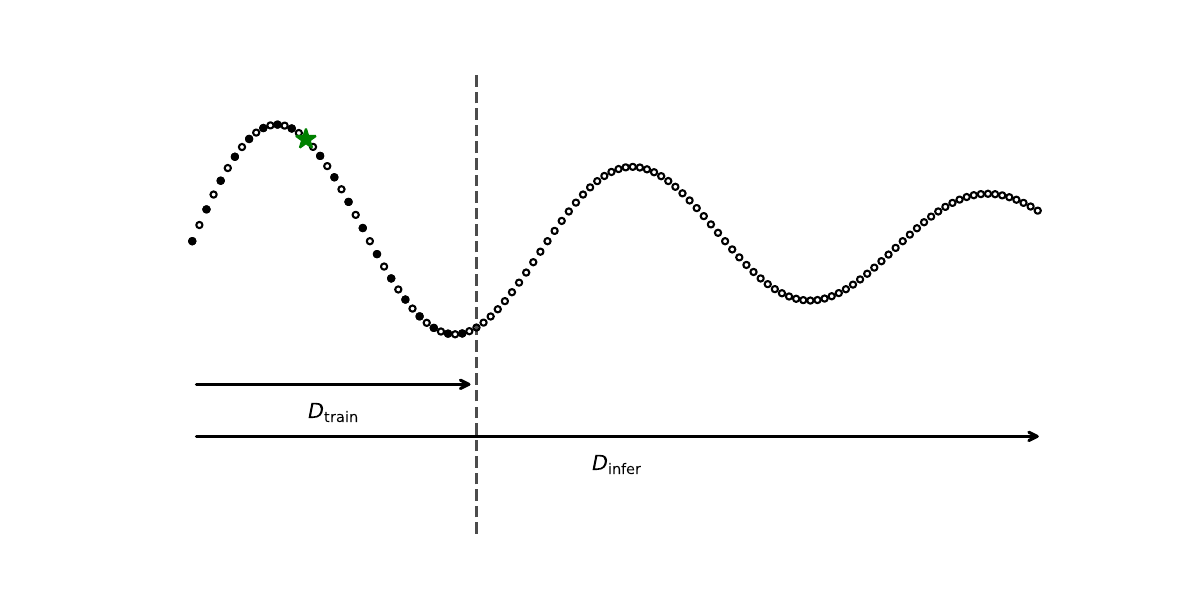}
    \caption{Schematic sampling and dataset construction procedure. A trajectory generated at sampling frequency ${sr}_{gen}$ 
    over a duration $D=D_{infer}=\beta T_0$ is shown as white dot markers. From this trajectory, a training point, highlighted with a green star marker, 
    is uniformly sampled from a subset of samples, shown as black dot markers, obtained at frequency ${sr}_{train}$ and restricted to $t\leq\alpha T_0$. 
    The training horizon $D_{train}$ and the inference horizon $D_{infer}$ are indicated by arrows, with the vertical dashed line marking the end of the 
    training interval.}
    \label{fig:sampling_technique_description}
\end{figure*}

\textbf{Training details}. Let ${D}_{train}=\alpha T_0$ and ${sr}_{train}=\gamma f_0<sr_{gen}$ be the training time and data sampling rate, respectively, where $\alpha$ and $\gamma$ are hyperparameters. 
For each system, we consider $\alpha$ such that the harmonic oscillator has dissipated $25\%$ of its initial energy in the absence of control and $\gamma=100$. The set 
of training trajectories $\mathcal{T}_{train}$ is built considering $N_{train}$ trajectories from $\mathcal{T}$ up until a duration $T_{train}$ sub-sampled at ${sr}_{train}$. 
Once $\mathcal{T}_{train}$ is created, one point $\xi^{i}_{train}=((\bm x_n^i,\bm u_n^i),\bm x_{n+1}^i)$ is uniformly sampled from each trajectory $\tau_{train}^i\in\mathcal{T}_{train}$ (see Figure \ref{fig:sampling_technique_description} for a graphical description). The set of points 
$\xi_{train}=\{\xi_{train}^i\}_{i=1,...,N_{train}}$ constitutes the training dataset. Note that training on a set of isolated points sampled at random from complete trajectories is closer to experimental conditions than training on a set of complete trajectories that might be more difficult to measure accurately in practice. It also probably adds complexity for the model, as isolated points are not obviously correlated. Experiments are carried out for $N_{train}\in[25,100,400]$. After each training 
epoch, models are validated on another dataset $\xi_{eval}$, which is constructed as $\xi_{train}$ but for 2500 trajectories. For each of the datasets, experiments are 
carried out for 10 runs with different model initializations using the Adam optimizer \cite{kingma2014adam} with a batch size $64$ for $50k$ optimizer steps with a learning rate 
of $10^{-3}$. \\

\textbf{Inference details}. The inference dataset $\mathcal{T}_{infer}$ is constructed taking a subset of $N_{infer}=100$ trajectories from $\mathcal{T}$ for a duration 
${D}_{infer}=\beta T_0$ and a sampling rate ${sr}_{infer}={sr}_{train}$. Note that $\beta \gg \alpha$, i.e. the model has to generate in inference longer sequences that it has been trained for.  \\

The performance of the model is then assessed generating autoregressively a trajectory $\tilde{\tau}$ 
from each of the initial conditions of the trajectories in $\mathcal{T}_{infer}$ and comparing it with the reference trajectory $\tau$. For each model initialization, 
we compute the Mean Squared Error (MSE) between the reference and predicted trajectories:
\begin{equation}
        \label{eq:inference-loss}
        \mathcal{L}_{traj}=\frac{1}{N_{inf}\cdot T_{inf}\cdot f_{inf}}\sum_{i=1}^{N_{inf}}\|\tau_{i}-\tilde{\tau}_{i} \|^2
    \end{equation}
\begin{figure*}[ht]
    \centering
    \includegraphics[width=\linewidth]{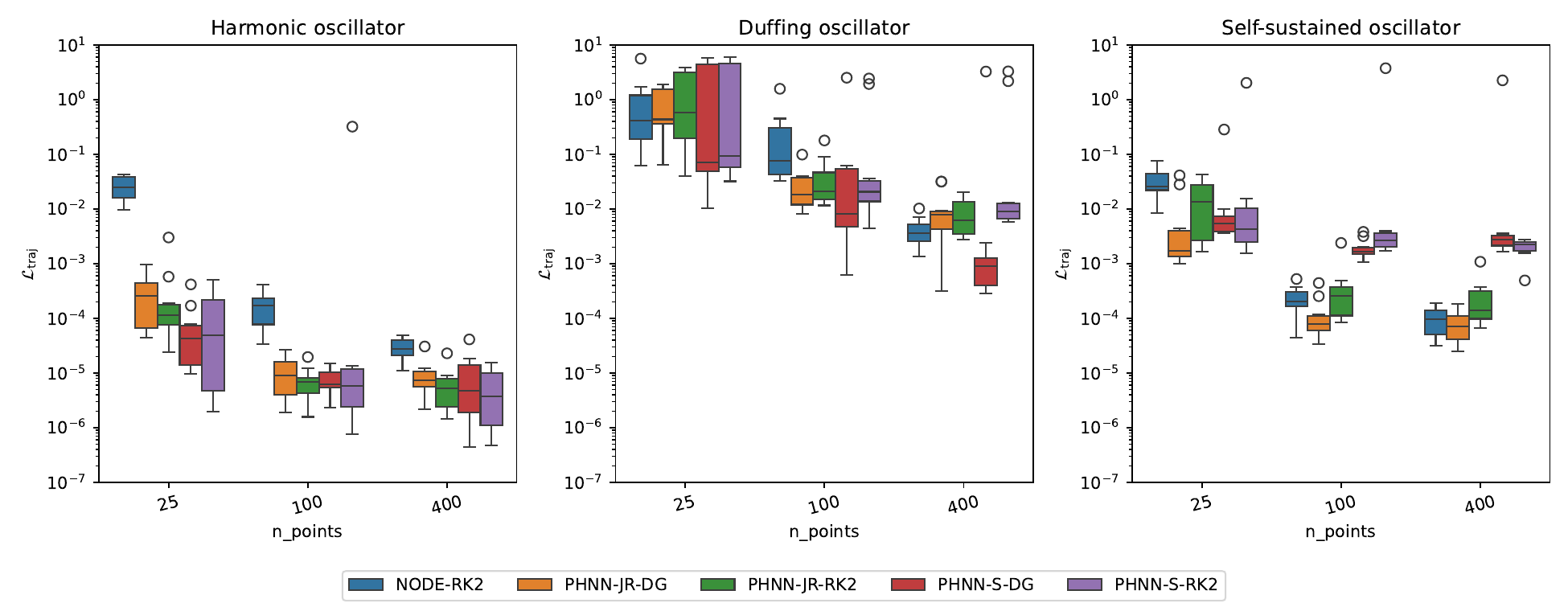}
    \caption{Boxplot of the inference errors for the three different oscillators and varying numbers of training points (small models). From left to right: harmonic, Duffing and self-sustained oscillators.}
    \label{fig:study-1-training-points}
\end{figure*}

\section{Results and discussions}\label{sec:results}
\begin{figure*}[h!]
    \centering
        \subfigure[Harmonic oscillator]{\includegraphics[width=0.32\textwidth]{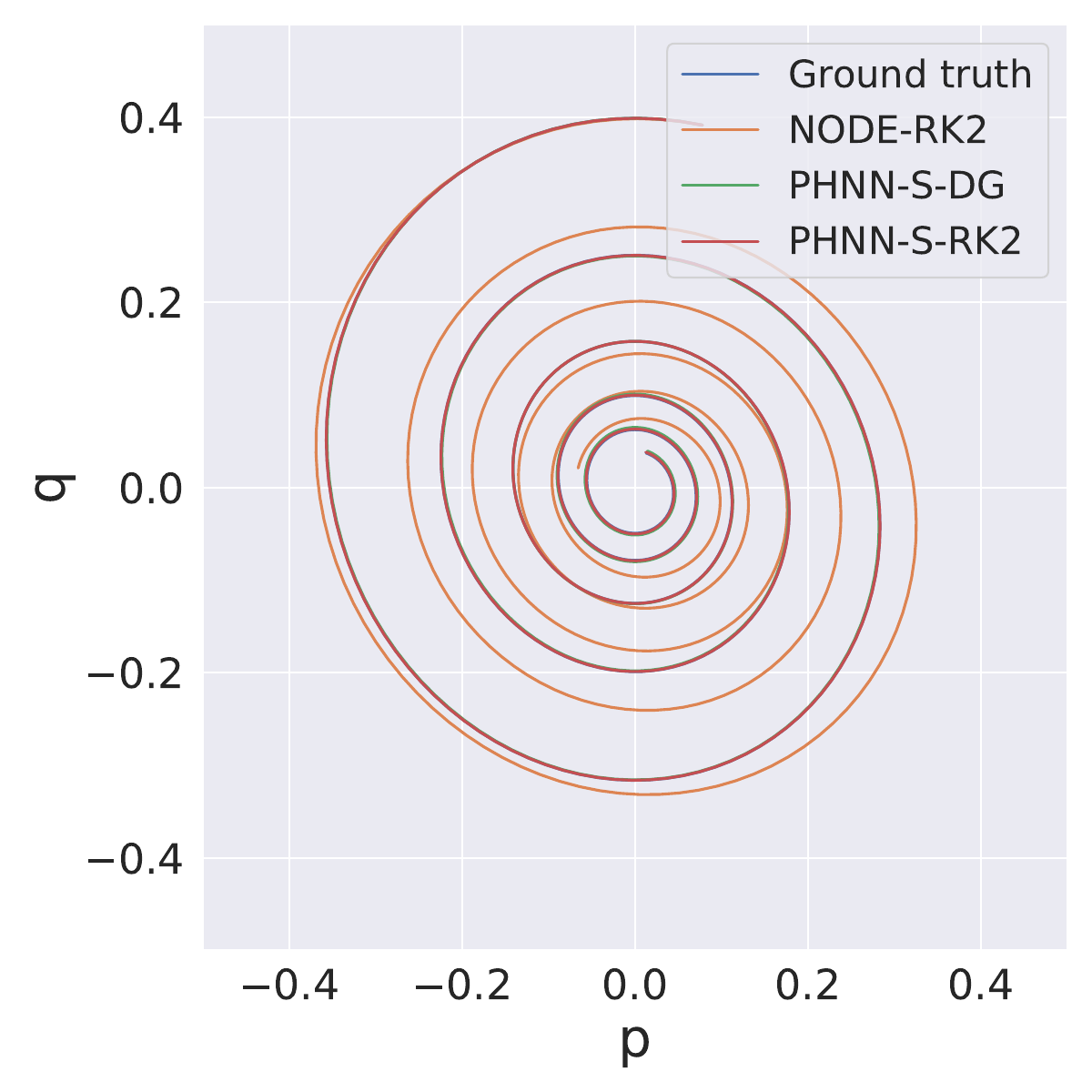}}
        \hfill 
        \subfigure[Duffing oscillator]{\includegraphics[width=0.32\textwidth]{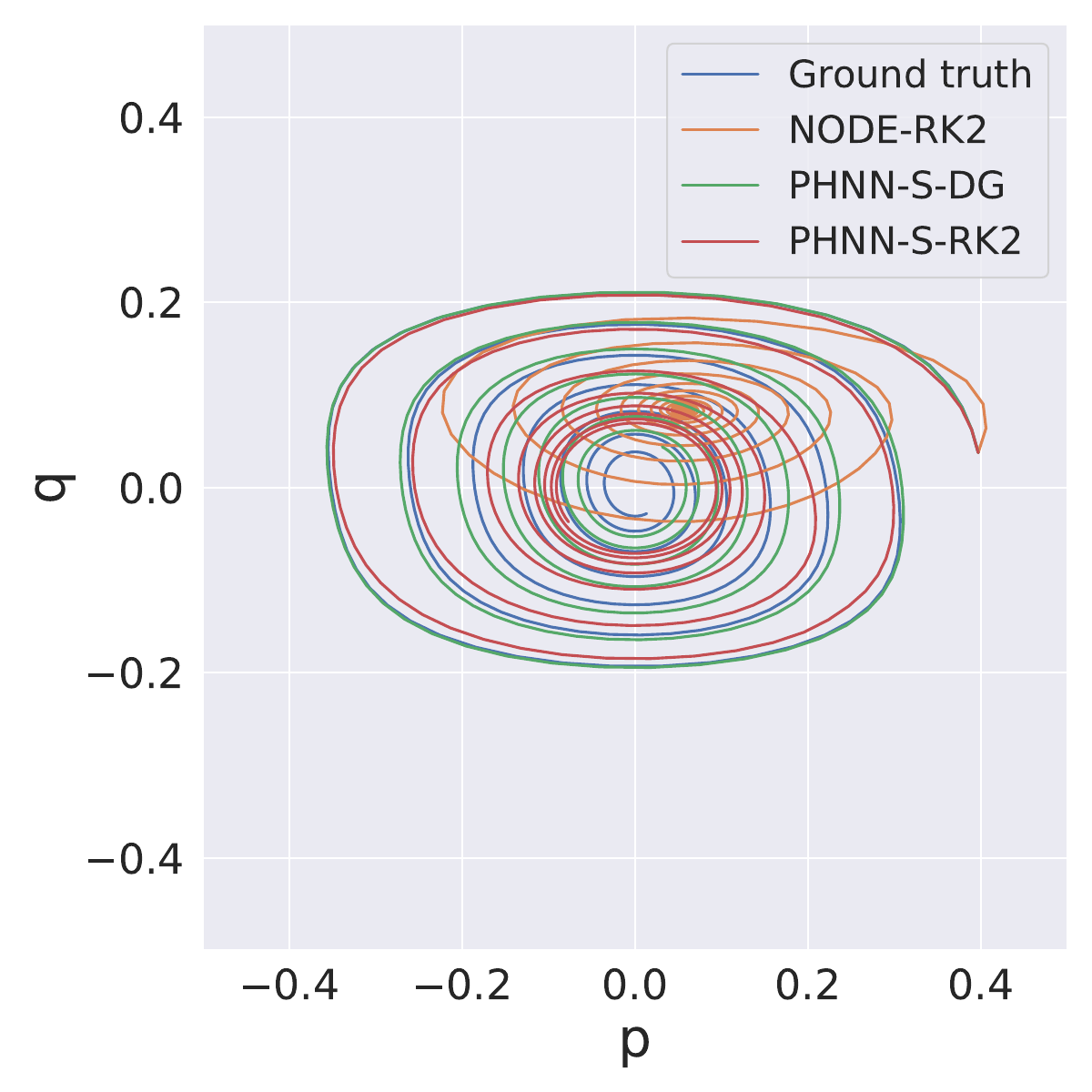}}
        \hfill 
        \subfigure[Self-sustained oscillator]{\includegraphics[width=0.32\textwidth]{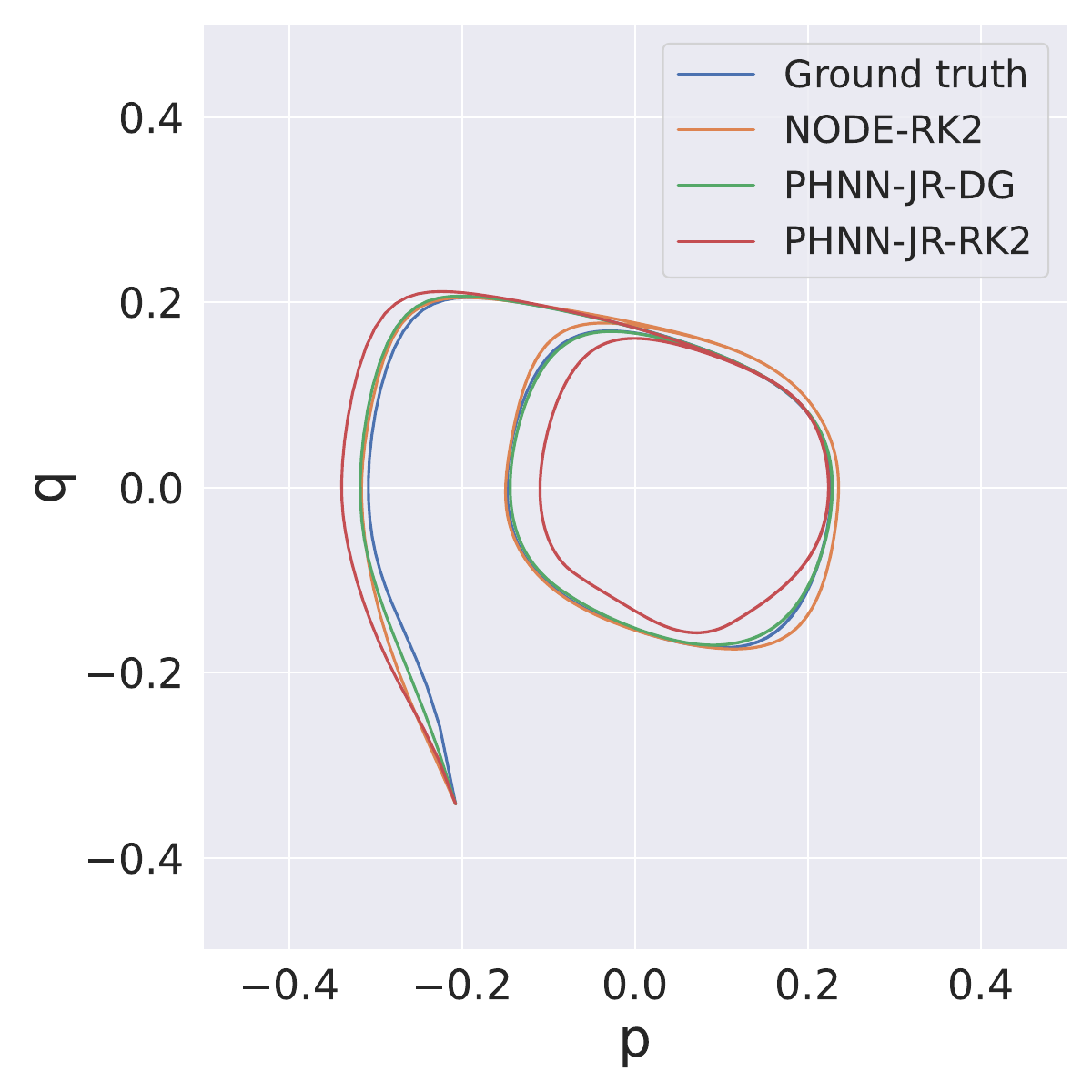}}
        \caption{Comparison of the learned trajectory discretizations for each oscillatory system starting on $\bm x_0$ such that 
        $H(\bm x_0)=0.5J$. (\textbf{Left}) NODE and PHNN-S models trained on $N_{train}=25$ from the harmonic oscillator. (\textbf{Center}) NODE and PHNN-S models 
        trained on $N_{train}=100$ from the Duffing oscillator. (\textbf{Right}) NODE and PHNN-JR models trained on $N_{train}=25$ from the self-sustained oscillator. For the harmonic and Duffing 
        oscillator, the trajectory is obtained with $\bm u=0$; and for the self-sustained, with $\bm u$ such that the system stabilizes 
        in a limit cycle.}
        \label{fig:compare-one-trajectory-study-1}
    \end{figure*}
\subsection{Impact of the number of training points}
Figure \ref{fig:study-1-training-points} shows the inference error (see Table \ref{tab: results-study-I} in the \ref{ap: results-experiments} 
for the detailed numerical values) for the different models across the three values of $N_{train}$ when learning the three 
controlled oscillatory systems with small model size. As expected, increasing the number of points decreases the error for all systems 
and architectures, except for when we train the PHNN-S model on the self-sustained oscillator where the performance seems to be 
fairly constant. The NODE architecture is consistently outperformed by one of the physically constrained models. Regarding the 
differences between the two numerical methods, the discrete gradient outperforms the RK2 in the low-data regime for the harmonic 
oscillator, and for every number of training points in the case of the nonlinear oscillators. For the Duffing oscillator, all 
performance errors show a high dispersion (as shown by large IQR) in the low-data regime, reflecting the model sensitivity to weight initialization when very few training points are available. Finally, whereas the PHNN-S model achieves the lowest performance error for the harmonic and 
Duffing oscillator, it is not the case for the self-sustained oscillator, where it is outperformed by the PHNN-JR and the NODE. Figure \ref{fig:compare-one-trajectory-study-1} 
compares the NODE with the best PHNN architecture for each oscillator when discretizing a trajectory starting on an initial condition $\bm x_0$ such that $H(\bm x_0)=0.5J$.

\subsection{Impact of the number of trainable parameters}

\begin{figure*}[h!]
    \centering
    \includegraphics[width=\linewidth]{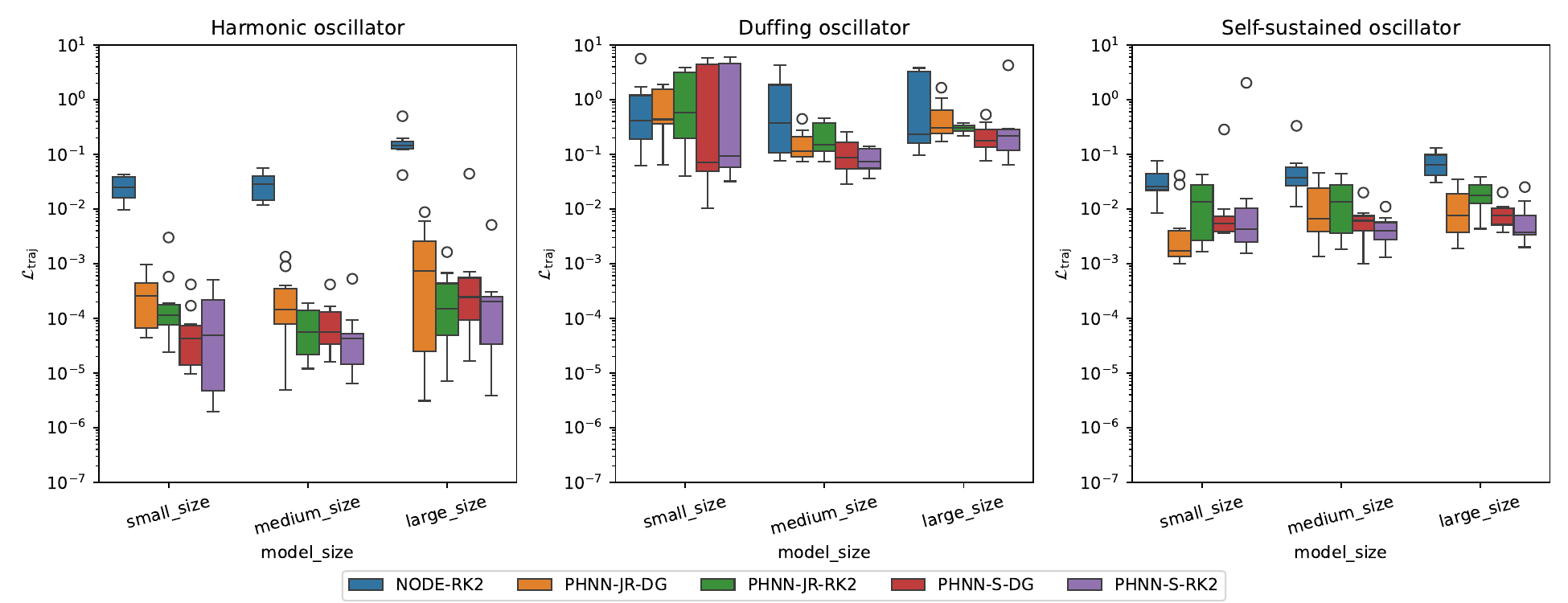}
    \caption{Boxplot of the inference errors for the three different oscillators and model sizes ($N_{train}=25$). From left to right: harmonic, Duffing and self-sustained oscillators.}
    \label{fig:study-2-number-of-parameters}
\end{figure*}

\begin{figure*}[h!]
\centering
    \subfigure[Harmonic oscillator]{\includegraphics[width=0.32\textwidth]{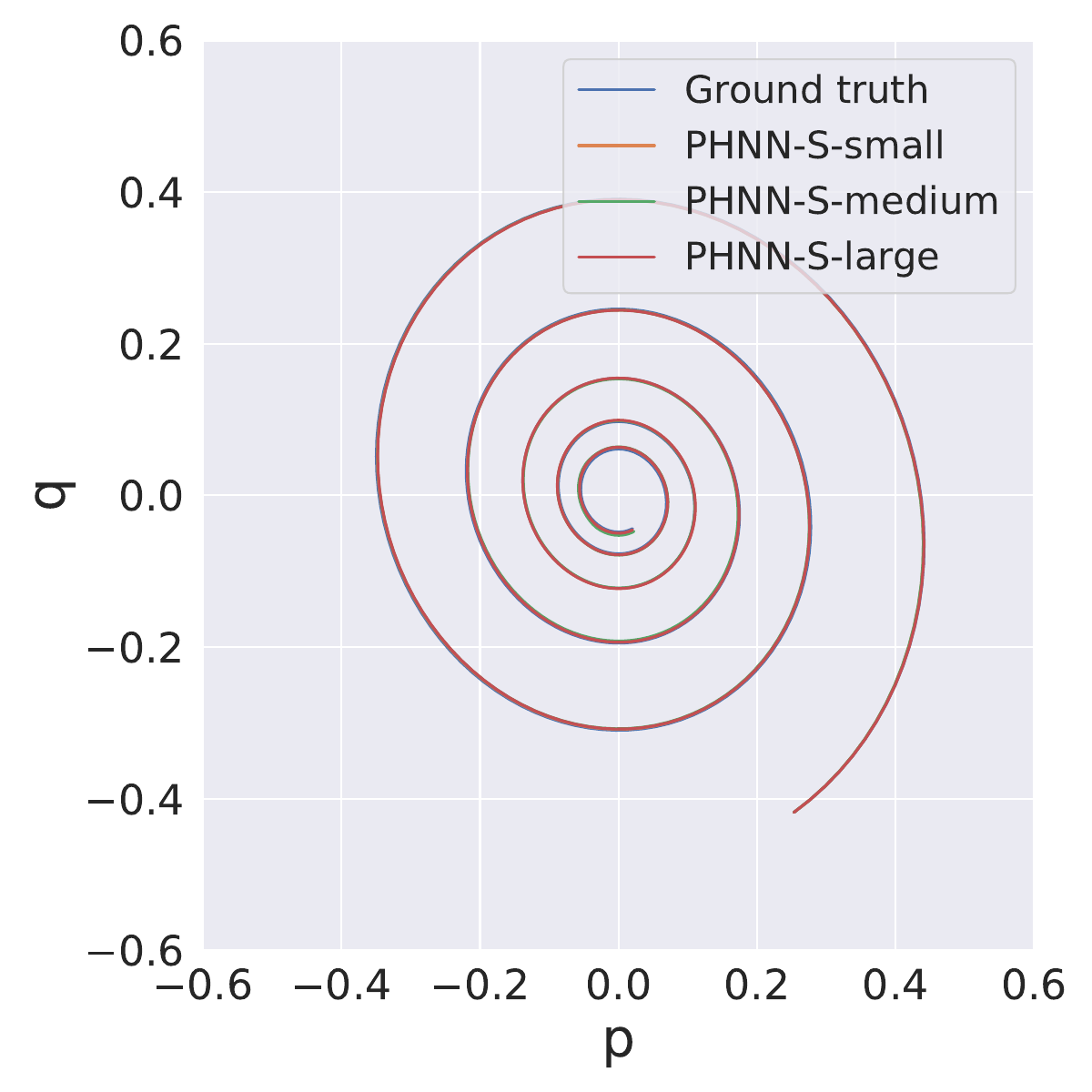}}
    \hfill 
    \subfigure[Duffing oscillator]{\includegraphics[width=0.32\textwidth]{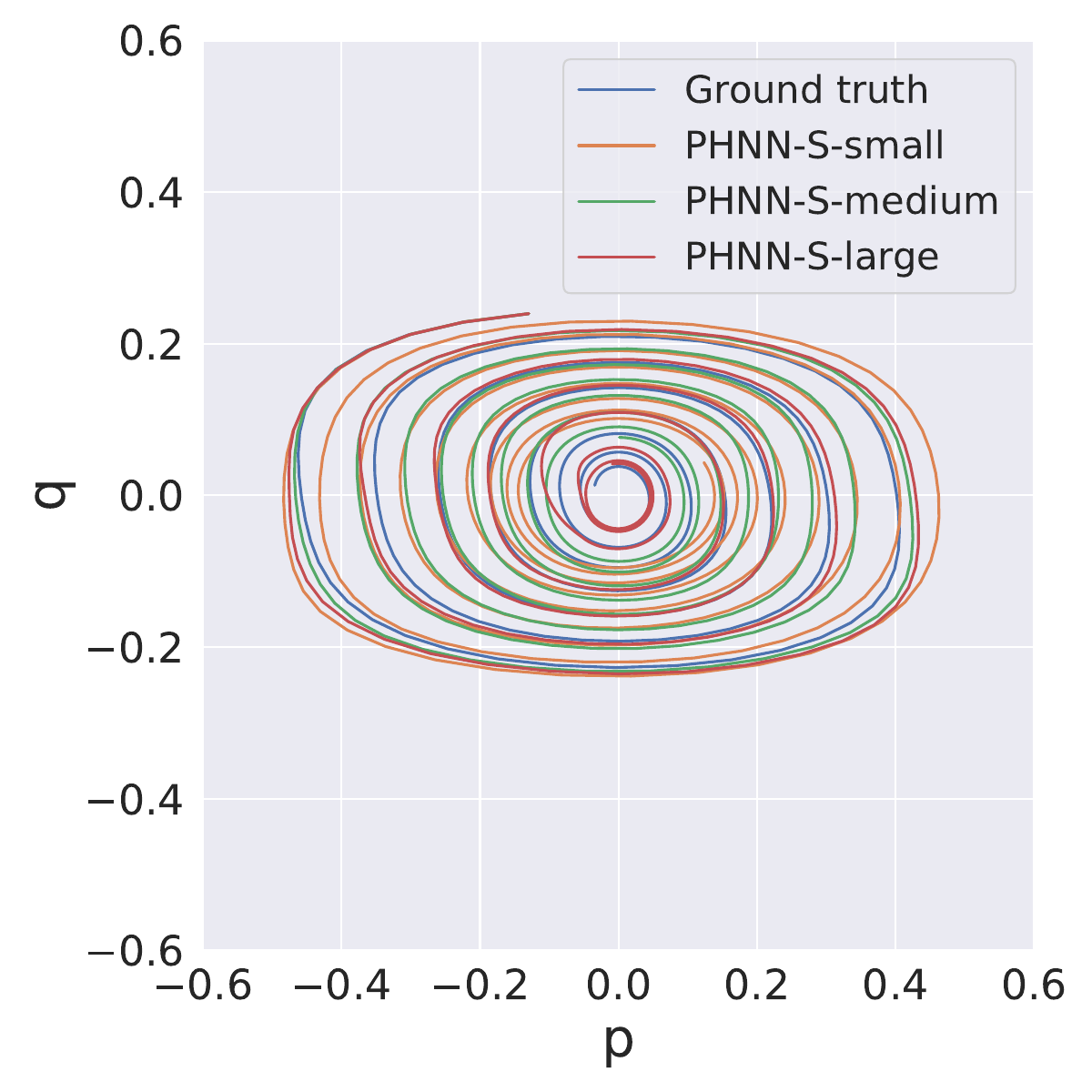}}
    \hfill 
    \subfigure[Self-sustained oscillator]{\includegraphics[width=0.32\textwidth]{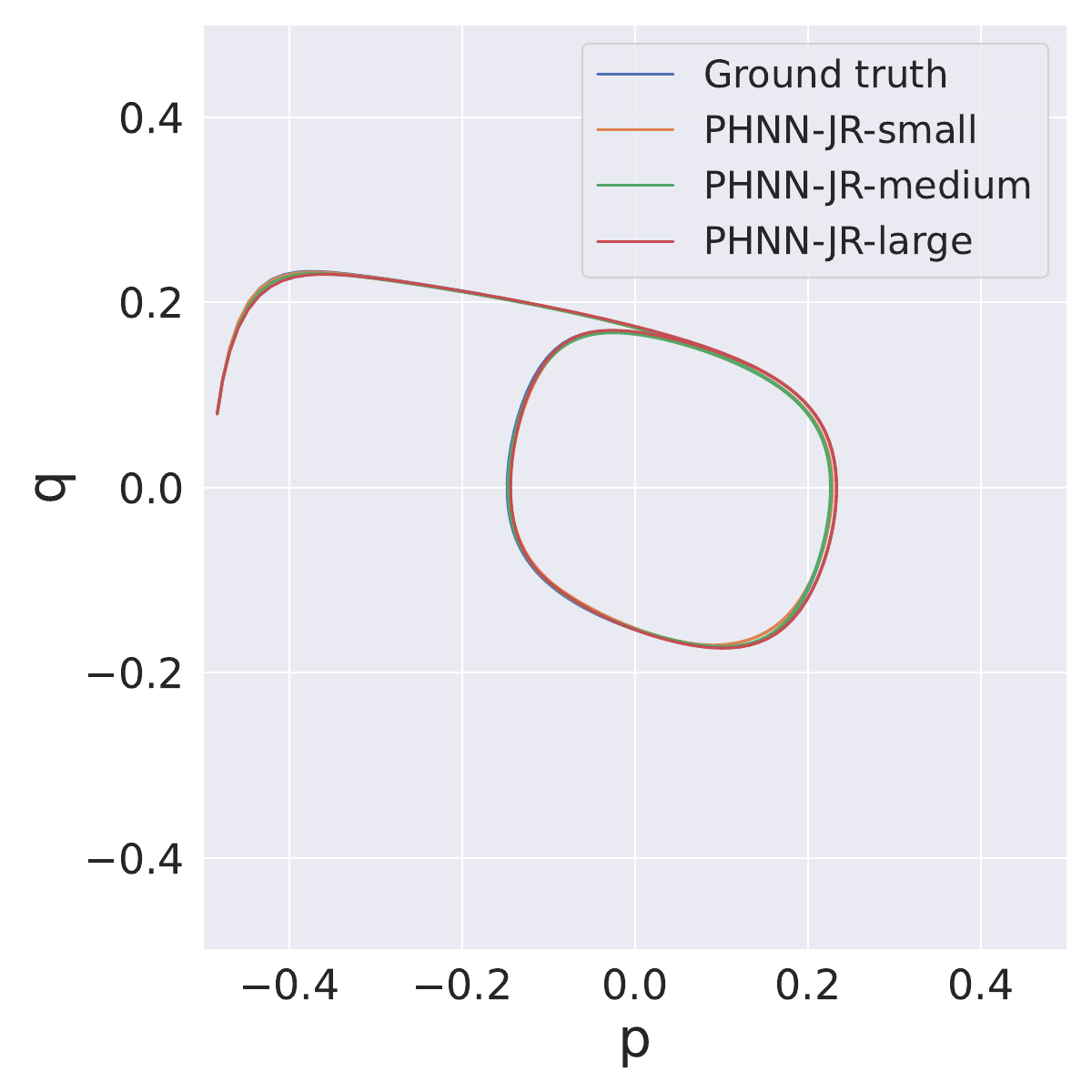}}
    \caption{Comparison of the learned trajectory discretizations by the discrete gradient for each oscillatory system starting 
    on $\bm x_0$ such that $H(\bm x_0)=0.75J$. (\textbf{Left}) PHNN-S models trained on $N_{train}=25$ from the harmonic oscillator. (\textbf{Center}) 
    PHNN-S models trained on $N_{train}=25$ from the Duffing oscillator. (\textbf{Right}) PHNN-JR models trained on $N_{train}=25$ from the self-sustained oscillator. For the harmonic and Duffing 
    oscillator, the trajectory is obtained with $\bm u=0$; and for the self-sustained, with $\bm u$ such that the system stabilizes 
    in a limit cycle.}
    \label{fig:compare-one-trajectory-study-2}
\end{figure*}

Figure \ref{fig:study-2-number-of-parameters} shows the inference error (see Table \ref{tab: results-study-II} in the \ref{ap: results-experiments} 
for the detailed numerical values) for the different models across the three regimes (small, medium, large) of trainable parameters when learning the 
three controlled oscillatory systems with $N_{train}=25$. In this case, increasing the number of trainable parameters increases 
the error in general for all the systems and architectures. This can be explained by the fact that there are not enough training 
samples to fit larger models and, thus, medium and large size models directly overfit and are not able to generalize correctly. 
The NODE is again consistently outperformed by one of the physically constrained models. Regarding the impact of using the 
different numerical methods, the discrete gradient only outperforms the RK2 when training small models. Interestingly, increasing 
the number of trainable parameters has a positive effect on the dispersion of the results for the Duffing oscillator, as we 
observe how the IQR is reduced with respect to the small size setting. Figure \ref{fig:compare-one-trajectory-study-2} 
shows how the trajectory discretization from the best PHNN architecture for each oscillator changes with respect to the number of trainable parameters 
starting on an initial condition $\bm x_0$ such that $H(\bm x_0)=0.5J$. 

\begin{figure*}[h!]
    \centering
        \subfigure[Condition number]{\includegraphics[width=0.32\textwidth]{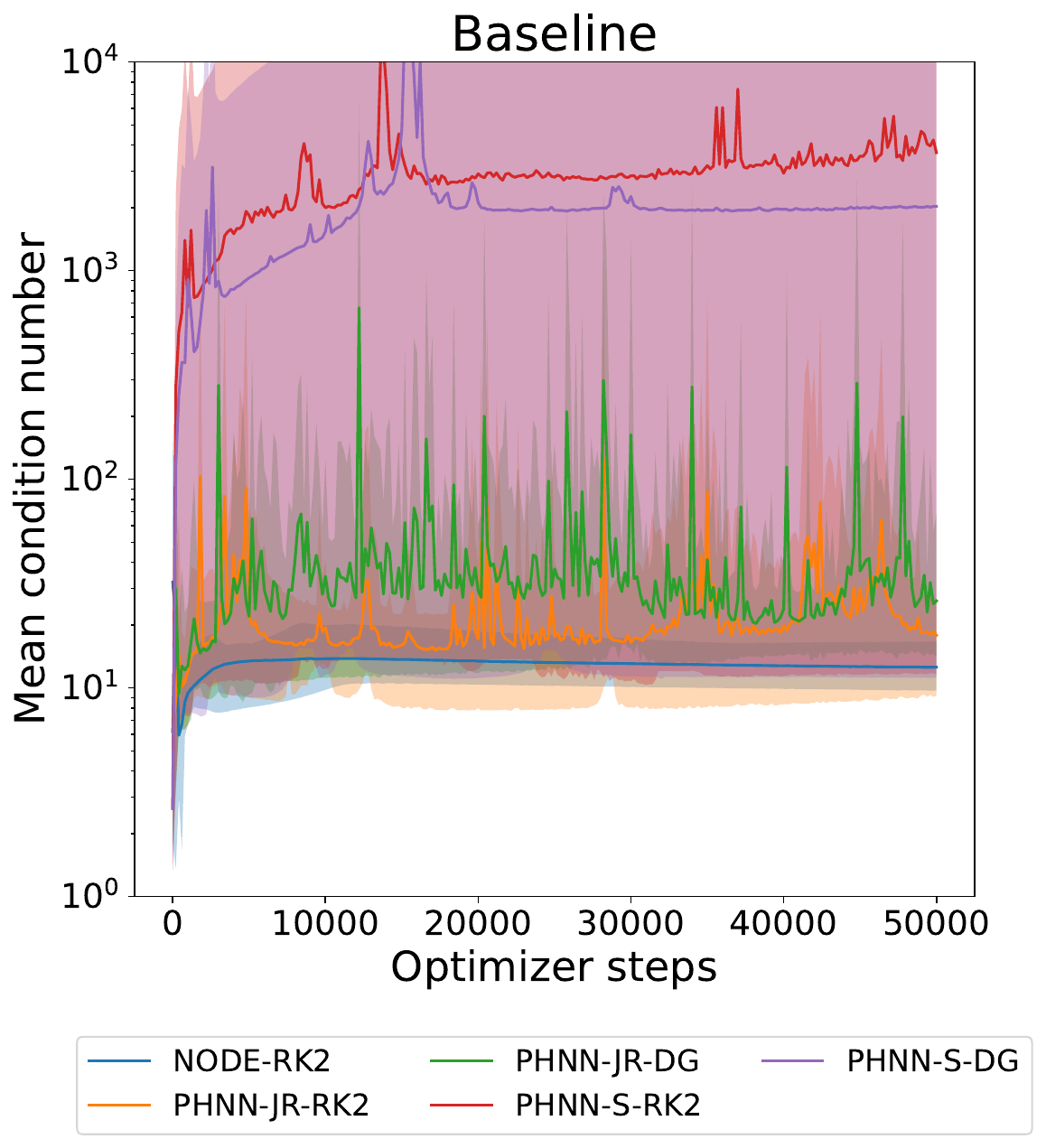}}
        \hfill 
        \subfigure[Spectral norm]{\includegraphics[width=0.32\textwidth]{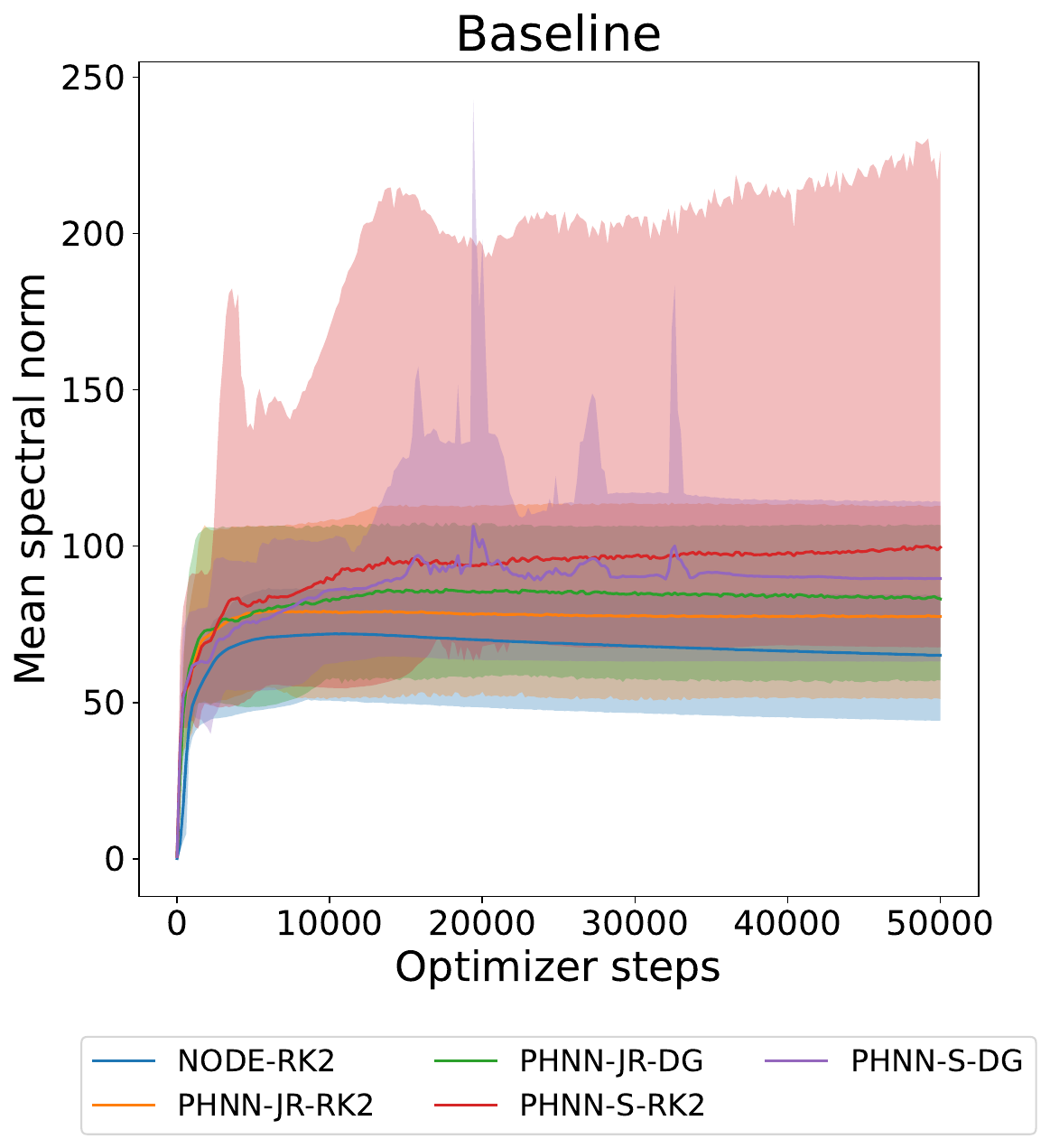}}
        \hfill 
        \subfigure[Stiffness ratio]{\includegraphics[width=0.32\textwidth]{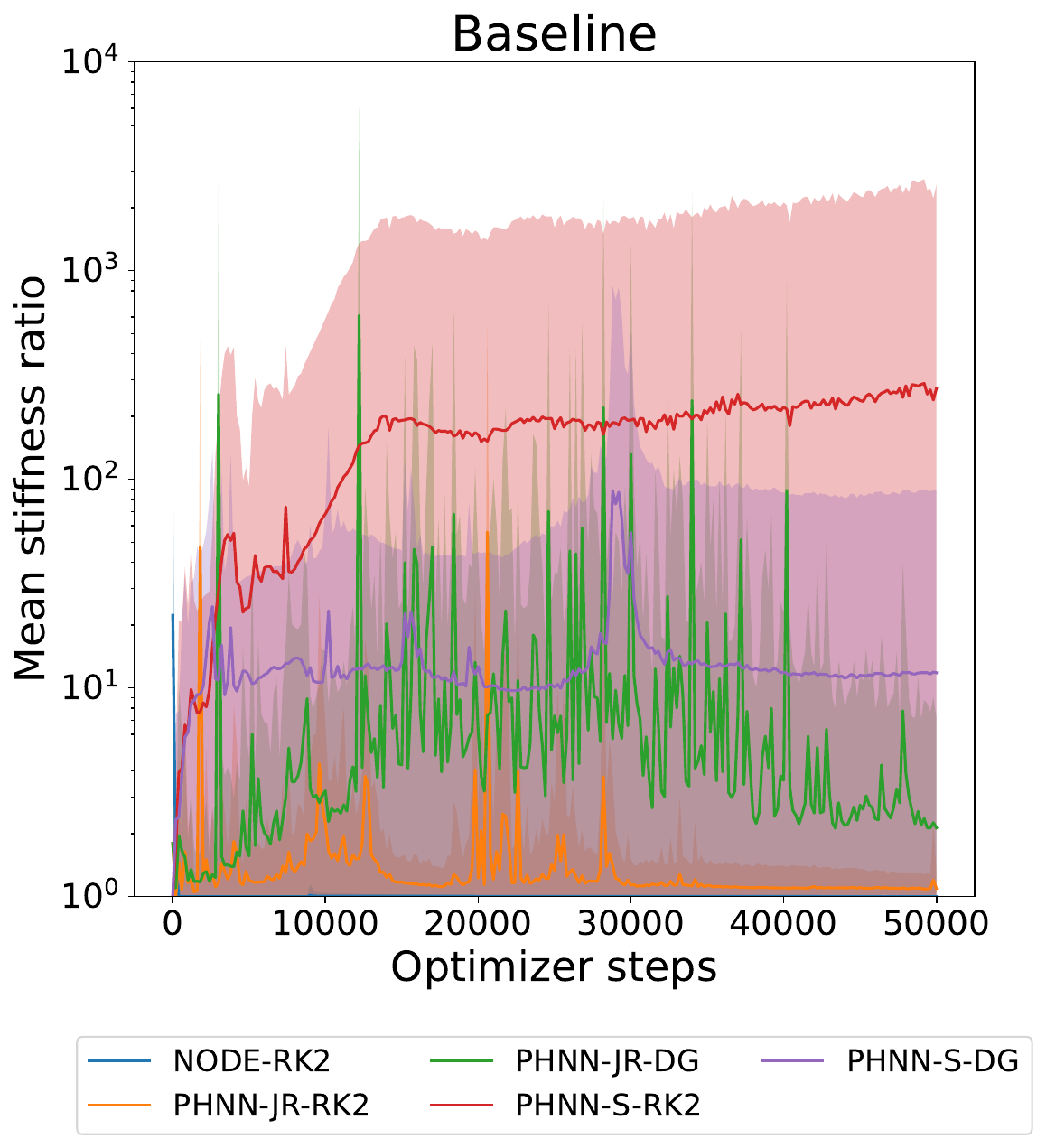}}
        \caption{Mean condition number (\textbf{left}), spectral norm (\textbf{center}) and stiffness ratio (\textbf{right}) per 
        optimizer step during the training for the different architectures when modeling the Duffing oscillator with $N_{train}=25$ 
        points and no Jacobian regularization (BL). For each optimizer step, the corresponding mean value (solid line) is obtained 
        computing the average over the 10 model initializations, with the shaded area indicating the range between the minimum and 
        maximum values.}
        \label{fig:jacobian_quantities_baseline}
    \end{figure*}

\subsection{Impact of the Jacobian regularizations}

\begin{figure*}[h!]
    \centering
    \includegraphics[width=\linewidth]{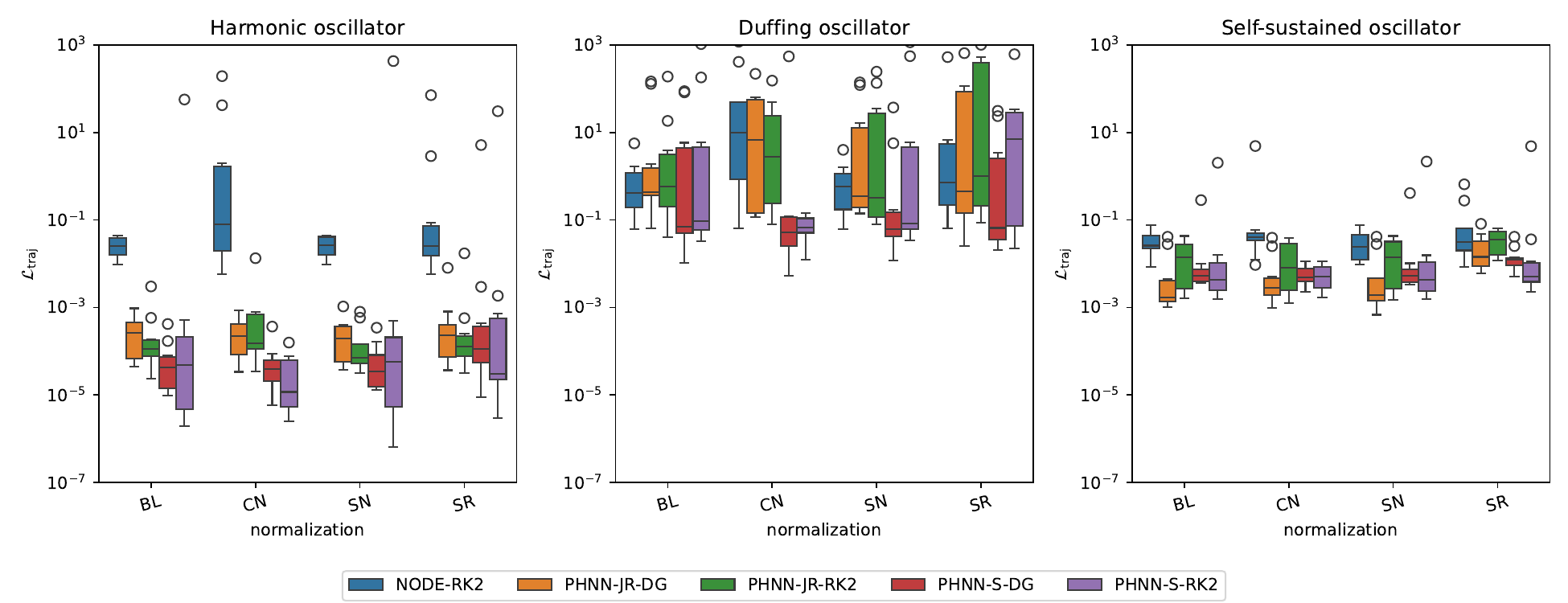}
    \caption{Boxplot of the inference errors for the three different oscillators and normalizations ($N_{train}=25$ and small 
    number of parameters). Notation: BL refers to the baseline (i.e. no Jacobian regularization); CN, to the condition number 
    regularization \eqref{eq:condition-number-loss}; SN, to the spectral norm regularization \eqref{eq:spectral-norm-loss}; 
    and SR, to the stiffness ratio regularization \eqref{eq:stiffness-ratio-loss}. From left to right: harmonic, Duffing and self-sustained oscillators.}
    \label{fig:study_3_jacobian_regularization}
\end{figure*}

\begin{figure*}[h!]
\centering
    \subfigure[Condition number]{\includegraphics[width=0.32\textwidth]{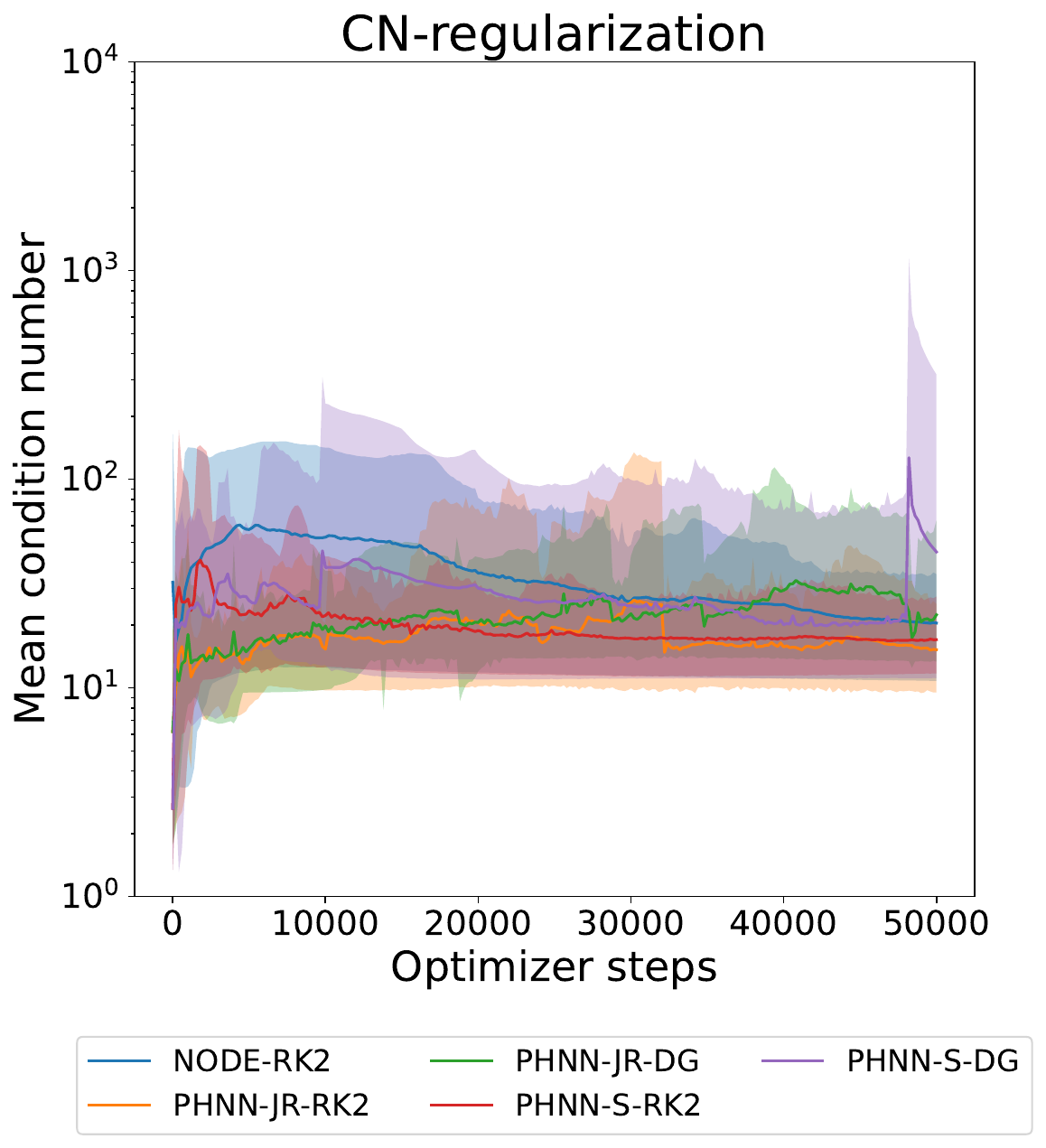}}
    \hfill 
    \subfigure[Spectral norm]{\includegraphics[width=0.32\textwidth]{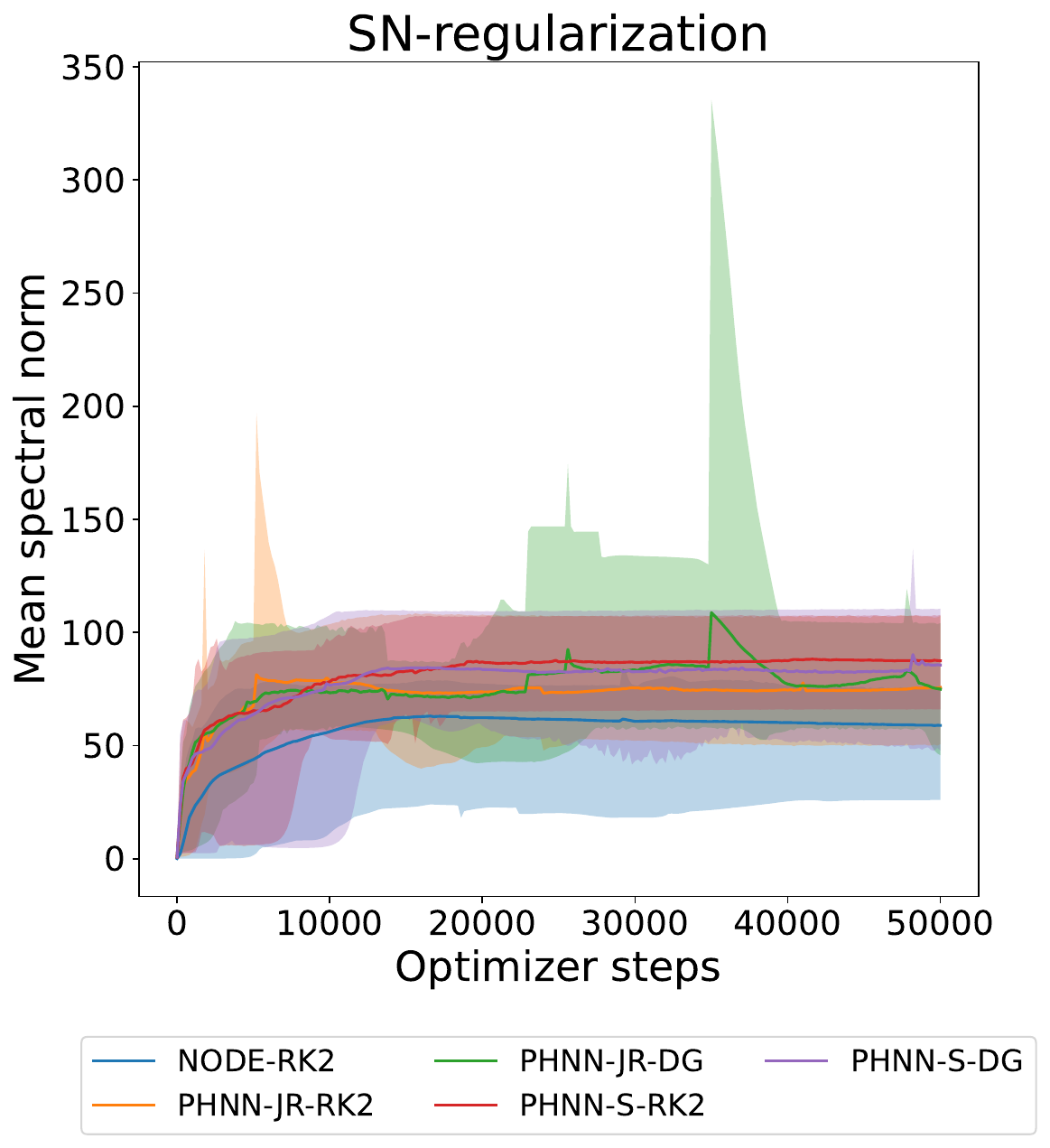}}
    \hfill 
    \subfigure[Stiffness ratio]{\includegraphics[width=0.32\textwidth]{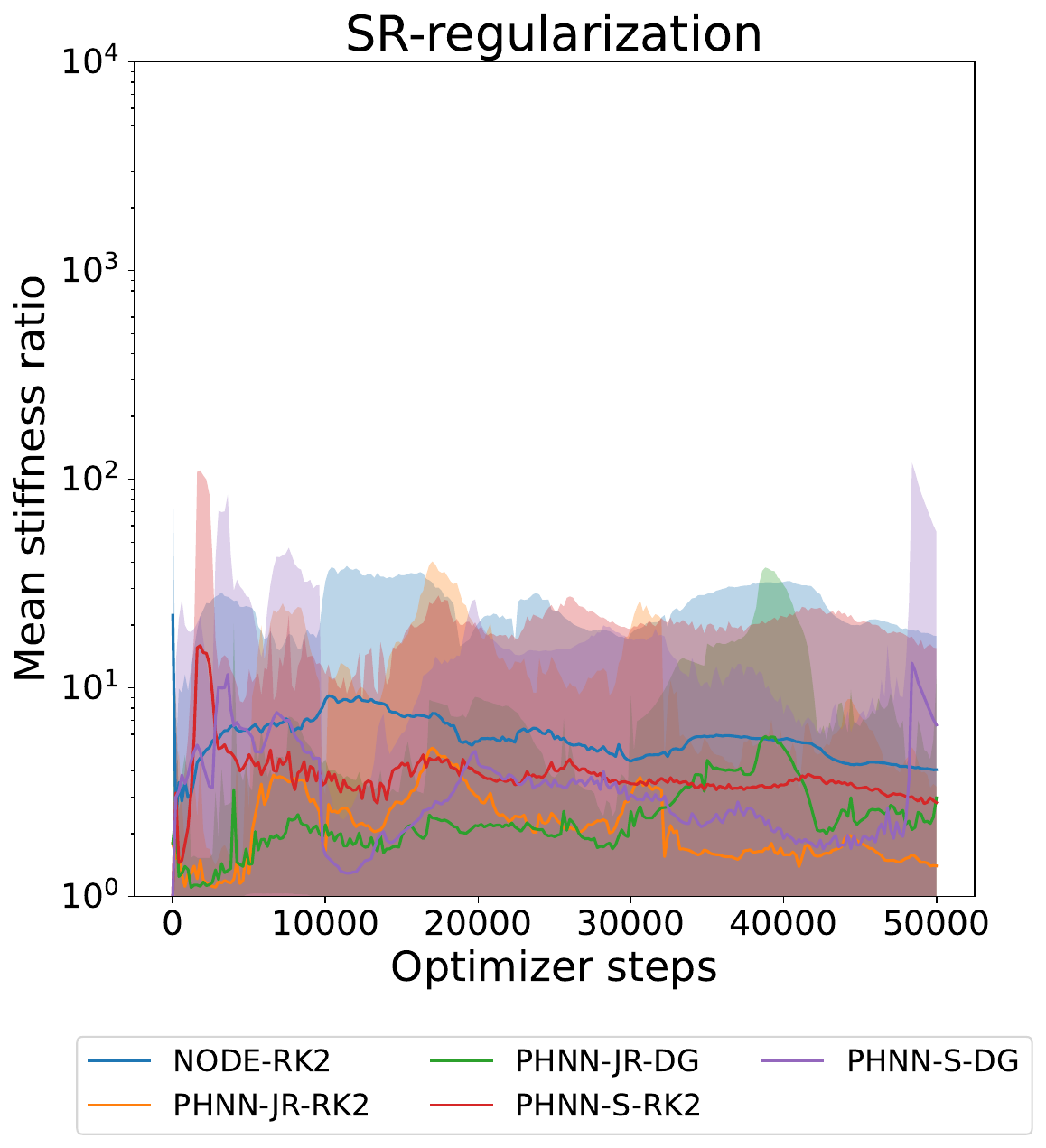}}
    \caption{Impact on the mean condition number (\textbf{left}), spectral norm (\textbf{center}) and stiffness ratio 
    (\textbf{right}) per optimizer step during the training for the different architectures when modeling the Duffing 
    oscillator with $N_{train}=25$ points and the proposed Jacobian regularizations. For each optimizer step, the corresponding 
    mean value (solid line) is obtained computing the average over the 10 model initializations, with the shaded area indicating 
    the range between the minimum and maximum values.}
    \label{fig:jacobian_quantities_regularizations}
\end{figure*}

Regularizing the Jacobian of the PHNN models is motivated by the potential correlation existing between the inference dispersion and a high condition number, spectral norm or stiffness ratio mean value at the end of the training. As we observe in Figure \ref{fig:jacobian_quantities_baseline}, the models with the highest
 mean condition number, spectral norm and stiffness ratio when modeling the Duffing oscillator with $N_{train}=25$, i.e. PHNN-S-RK2 and PHNN-S-DG, are precisely the ones with a higher 
 IQR value (see Figure \ref{fig:study-1-training-points} and Table \ref{tab: results-study-I} in \ref{ap: results-experiments}). Furthermore, their stiffness ratio is higher than 1, meaning
  that these models are learning a representation that is stiffer than the original system from which we generate the data. Experiments are, thus, carried out to test whether controlling the 
  numerical behavior through any of the proposed Jacobian regularizations lowers the inference error and/or dispersion. \\ 
 
 Figure \ref{fig:study_3_jacobian_regularization} shows the inference error (see Table \ref{tab:study_3_jacobian_regularization} 
in the \ref{ap: results-experiments} for the detailed numerical values) when learning the three oscillatory systems under the different Jacobian 
regularizations \eqref{eq:condition-number-loss}-(\ref{eq:stiffness-ratio-loss}) with different small models and $N_{train}=25$. 
BL refers to the baseline (i.e. no Jacobian regularization); CN, to the condition number regularization \eqref{eq:condition-number-loss}; 
SN, to the spectral norm regularization \eqref{eq:spectral-norm-loss}; and SR, to the stiffness ratio regularization \eqref{eq:stiffness-ratio-loss}. 
In general, Jacobian regularization does not improve the results for those systems in which the model performances were already good (harmonic and self-sustained
 oscillator). However, CN-regularization considerably improves the results for the Duffing oscillator, where the dispersion was high. 
 Interestingly, controlling the stiffness ratio towards the ground truth system value is not directly associated with an improvement in
 the accuracy nor the dispersion: the stiffness ratio values after SR-regularization are closer to 1 (see Figure \ref{fig:jacobian_quantities_regularizations}) but the inference error or the dispersion
 is not always reduced. 

\subsection{Differences between PHNN-S and PHNN-JR}
It is a recurrent conclusion in all the studies carried out in this work that the architecture that performs the best for the harmonic and Duffing
oscillator is the PHNN-S, while for the self-sustained oscillator, is the PHNN-JR, independently of the numerical method used for discretization. 
We hypothesize that this could be due to whether the dissipation of the system is nonlinear or not, implying that PHNN-S has difficulty learning the system's 
dynamics in the presence of nonlinear dissipation, which is the case for the self-sustained oscillator.

\section{Conclusion}\label{sec:conclusion}
\textbf{Contributions}. In this work, we have addressed the problem of designing \textit{physically consistent} 
neural network models based on PHS formulations and energy-preserving numerical methods. The main contributions 
of this study include a comparison of two theoretically equivalent PHS formulations: the PH-DAE and the input-state-output 
PHS with feedthrough, when implemented as PHNNs; a performance comparison between a second-order energy-preserving numerical 
method and a Runge-Kutta method of the same order; and a empirical study of the impact of regularizing the Jacobian of PHNN 
through two methods already applied to NODEs and a new one tackling the stiffness of the learned ODE solutions. Our 
results demonstrate that the PHNN based on the discrete gradient method outperform those based on RK2, especially when the 
data comes from nonlinear systems and is learned by models with a small number of trainable parameters. The results also 
indicate that the best PHS formulation to be implemented as a neural network depends on the modeled system and where the 
nonlinearity operates. Interestingly, this dependence is also found in the results when Jacobian regularization is applied, 
where the harmonic and Duffing oscillator clearly benefited from these additional soft constraints while the self-sustained oscillator 
did not. Overall, the experiments on Jacobian regularization show how this technique could be used in PHNNs to improve the generalization
when increasing the number of training points is not possible.\\

\textbf{Limitations}. While the proposed frameworks have shown promising performance, they are subject to certain limitations, 
including the fact the interconnection matrix ($\bm S$ or $\bm J$, depending on the formulation) is always given a priori, and 
that we are only considering Hamiltonians of the form \eqref{eq:quadratic_H} and semi-explicit PHS formulations, which are not 
sufficiently general to describe any given physical system~\cite{van2014port}. Numerically, our results are also limited by 
the fact that in our experiments we use second-order numerical methods, as the objective is to compare two classes of schemes 
(energy preserving versus non-preserving), rather than to achieve the best possible performance. As for the Jacobian regularizations, 
they are limited by the hypothesis that regularizing dynamics at the input data can lead to regularized dynamics across the entire 
solution space, which could not necessarily be the case when working with the nonlinear dynamics induced by neural network-based models. 
Furthermore, the real impact of each Jacobian regularization is not fully explored in this study, as the value of 
$\lambda_{CN},\lambda_{SN}$ and $\lambda_{SR}$ in \eqref{eq:condition-number-loss}-\eqref{eq:stiffness-ratio-loss} 
are chosen equally for each physical system and architecture, without an exhaustive grid search for each case. \\

\textbf{Future works}. Future work will focus on designing and testing learning frameworks capable of modelling PHS without 
knowing the interconnection matrix a priori and from which we only have access to its inputs and outputs. Of particular interest 
will be understanding the reasons why the performance is dependent on the PHS formulation and the modeled physical system, as well as 
detecting the factors that hinder the training of these models and how this relates to a well-conditioned and non-stiff learned 
Jacobian matrix. Whether the same conclusions hold when using a fourth-order discrete gradient and Runge-Kutta method remains to be 
explored. Finally, this framework could be used to model more complex ODE systems or even extended to tackle the learning of Hamiltonian 
PDEs, where power-preserving spatial discretizations should also be considered~\cite{trenchant2018}.  \\

Overall, this paper provides a first comparison between the use of Runge-Kutta and discrete gradient methods in the PHNN framework 
as well as an extensive empirical study on the impact that different PHS formulations, number of training points, size of the models 
or Jacobian regularizations have on the modeling of a given physical system. The authors also hope that this type of work serves as 
a baseline for future methods and thus contribute to addressing the lack of homogeneity that characterizes similar work in the 
literature.\\

\textbf{Acknowledgements}. This project is co-funded by the European Union's Horizon Europe research and innovation program Cofund 
SOUND.AI under the Marie Sklodowska-Curie Grant Agreement No 101081674. This project was also provided with computing HPC and storage 
resources by GENCI at IDRIS thanks to the grant 2024-102911 on the supercomputer Jean Zay's V100 partition.

\bibliographystyle{elsarticle-num}
\bibliography{references}

@article{Cicirello_2024,
   title={Physics-Enhanced Machine Learning: a position paper for dynamical systems investigations},
   volume={2909},
   ISSN={1742-6596},
   number={1},
   journal={Journal of Physics: Conference Series},
   publisher={IOP Publishing},
   author={Cicirello, Alice},
   year={2024},
   month=dec, pages={012034},
   DOI={10.1088/1742-6596/2909/1/012034}}

@article{Baxter_2000,
   title={A Model of Inductive Bias Learning},
   volume={12},
   ISSN={1076-9757},
   DOI={10.1613/jair.731},
   journal={Journal of Artificial Intelligence Research},
   publisher={AI Access Foundation},
   author={Baxter, J.},
   year={2000},
   month=mar, pages={149–198} }

@ARTICLE{karniadakis2021,
       author = {{Karniadakis}, George Em and {Kevrekidis}, Ioannis G. and {Lu}, Lu and {Perdikaris}, Paris and {Wang}, Sifan and {Yang}, Liu},
        title = "{Physics-informed machine learning}",
      journal = {Nature Reviews Physics},
         year = 2021,
        month = jun,
       volume = {3},
       number = {6},
        pages = {422-440},
          doi = {10.1038/s42254-021-00314-5}
}

@article{eidnes2023,
  title={Pseudo-Hamiltonian neural networks with state-dependent external forces},
  author={Eidnes, S{\o}lve and Stasik, Alexander J and Sterud, Camilla and B{\o}hn, Eivind and Riemer-S{\o}rensen, Signe},
  journal={Physica D: Nonlinear Phenomena},
  volume={446},
  pages={133673},
  year={2023},
  publisher={Elsevier},
  doi={10.1016/j.physd.2023.133673}
}

@article{raissi2019,
title = {Physics-informed neural networks: A deep learning framework for solving forward and inverse problems involving nonlinear partial differential equations},
journal = {Journal of Computational Physics},
volume = {378},
pages = {686-707},
year = {2019},
issn = {0021-9991},
doi = {10.1016/j.jcp.2018.10.045},
author = {M. Raissi and P. Perdikaris and G.E. Karniadakis}
}

@article{meng2025physics,
  title={When physics meets machine learning: A survey of physics-informed machine learning},
  author={Meng, Chuizheng and Griesemer, Sam and Cao, Defu and Seo, Sungyong and Liu, Yan},
  journal={Machine Learning for Computational Science and Engineering},
  volume={1},
  pages={20},
  year={2025},
  publisher={Springer},
  doi={10.1007/s44379-025-00016-0}
}

@book{khalil_nonlinear_2002,
  address = {Upper Saddle River, {N.J.}},
  title = {Nonlinear systems},
  edition = {3},
  publisher = {Prentice Hall},
  author = {Khalil, Hassan K},
  year = {2002}
}

@article{hadamard1902problemes,
  title={Sur les probl{\`e}mes aux d{\'e}riv{\'e}es partielles et leur signification physique},
  author={Hadamard, Jacques},
  journal={Princeton university bulletin},
  pages={49--52},
  year={1902},
  publisher={Princeton University}
}

@book{hirsch1974differential,
  title={Differential equations, dynamical systems, and linear algebra},
  author={Hirsch, Morris W and Devaney, Robert L and Smale, Stephen},
  edition = {1},
  volume={60},
  year={1974},
  publisher={Academic press}
}

@book{golub2013matrix,
  title={Matrix Computations},
  author={Golub, G.H. and Van Loan, C.F.},
  edition = {4},
  series={Johns Hopkins Studies in the Mathematical Sciences},
  year={2013},
  publisher={Johns Hopkins University Press}
}

@book{iserles2009first,
  title={A First Course in the Numerical Analysis of Differential Equations},
  author={Iserles, A.},
  edition = {2},
  series={A First Course in the Numerical Analysis of Differential Equations},
  year={2009},
  publisher={Cambridge University Press}
}

@book{taylor2005,
    author = {John R. Taylor},
    title = {Classical mechanics},
    edition = {1},
    publisher = {Calif.:University Science Books},
    year = {2005}
}

@article{maschke1992,
title = {An intrinsic hamiltonian formulation of network dynamics: non-standard poisson structures and gyrators},
journal = {Journal of the Franklin Institute},
volume = {329},
pages = {923-966},
year = {1992},
doi = {10.1016/S0016-0032(92)90049-M},
author = {B.M. Maschke and A.J. {Van Der Schaft} and P.C. Breedveld}
}

@book{duindam2009modeling,
  title={Modeling and control of complex physical systems: the port-Hamiltonian approach},
  author={Duindam, Vincent and Macchelli, Alessandro and Stramigioli, Stefano and Bruyninckx, Herman},
  year={2009},
  edition = {1},
  publisher={Springer Science \& Business Media}
}

@article{van2014port,
  title={Port-Hamiltonian systems theory: An introductory overview},
  author={Van Der Schaft, Arjan and Jeltsema, Dimitri and others},
  journal={Foundations and Trends{\textregistered} in Systems and Control},
  volume={1},
  number={2-3},
  pages={173--378},
  year={2014},
  publisher={Now Publishers, Inc.}
}

@book{chaigne2016acoustics,
  title={Acoustics of musical instruments},
  author={Chaigne, Antoine and Kergomard, Jean},
  edition = {1},
  year={2016},
  publisher={Springer}
}

@article{aoues2017modeling,
  title={Modeling and control of a rotating flexible spacecraft: A port-Hamiltonian approach},
  author={Aoues, Said and Cardoso-Ribeiro, Fl{\'a}vio Luiz and Matignon, Denis and Alazard, Daniel},
  journal={IEEE Transactions on Control Systems Technology},
  volume={27},
  number={1},
  pages={355--362},
  year={2017},
  publisher={IEEE},
  doi = {10.1109/TCST.2017.2771244.},
}

@unpublished{helie2022,
  TITLE = {{Elementary tools on Port-Hamiltonian Systems with applications to audio/acoustics}},
  AUTHOR = {H{\'e}lie, Thomas},
  DOI = {hal-03986168},
  NOTE = {Lecture},
  TYPE = {Doctoral},
  ADDRESS = {Fauenchiemsee, Germany},
  PAGES = {30},
  YEAR = {2022},
  MONTH = Mar,
}

@article{cardoso2024port,
  title={Port-Hamiltonian formulations for the modeling, simulation and control of fluids},
  author={Cardoso-Ribeiro, Fl{\'a}vio Luiz and Haine, Ghislain and Le Gorrec, Yann and Matignon, Denis and Ramirez, Hector},
  journal={Computers \& Fluids},
  pages={106407},
  year={2024},
  publisher={Elsevier},
  doi = {10.1016/j.compfluid.2024.106407},
}

@article{roze2024time,
  title={Time-space formulation of a conservative string subject to finite transformations},
  author={Roze, David and H{\'e}lie, Thomas and Rouhaud, Emmanuelle},
  journal={IFAC-PapersOnLine},
  volume={58},
  number={6},
  pages={232--237},
  year={2024},
  publisher={Elsevier},
  doi = {10.1016/j.ifacol.2024.08.286},
}

@book {Hairer2006,
    AUTHOR = {Hairer, Ernst and Lubich, Christian and Wanner, Gerhard},
     TITLE = {Geometric numerical integration},
    SERIES = {Springer Series in Computational Mathematics},
    VOLUME = {31},
   EDITION = {Second},
 PUBLISHER = {Springer-Verlag, Berlin},
      YEAR = {2006},
}

@article{celledoni2017energy,
  title={Energy-preserving and passivity-consistent numerical discretization of port-Hamiltonian systems},
  author={Celledoni, Elena and H{\o}iseth, Eirik Hoel},
  journal={arXiv preprint arXiv:1706.08621},
  year={2017}
}

@article{quispel1996discrete,
  title={Discrete gradient methods for solving ODEs numerically while preserving a first integral},
  author={Quispel, GRW and Turner, Grant S},
  journal={Journal of Physics A: Mathematical and General},
  volume={29},
  number={13},
  pages={L341},
  year={1996},
  publisher={IOP Publishing},
  doi = {10.1088/0305-4470/29/13/006},
}

@article{gonzalez1996time,
  title={Time integration and discrete Hamiltonian systems},
  author={Gonzalez, Oscar},
  journal={Journal of Nonlinear Science},
  volume={6},
  pages={449--467},
  year={1996},
  publisher={Springer},
  doi = {10.1007/BF02440162},
}

@article{CELLEDONI2025134471,
title = {Learning dynamical systems from noisy data with inverse-explicit integrators},
journal = {Physica D: Nonlinear Phenomena},
volume = {472},
pages = {134471},
year = {2025},
doi = {10.1016/j.physd.2024.134471},
author = {Elena Celledoni and Sølve Eidnes and Håkon Noren Myhr},
}

@inproceedings{chen2018,
author = {Chen, Ricky T. Q. and Rubanova, Yulia and Bettencourt, Jesse and Duvenaud, David},
title = {Neural ordinary differential equations},
year = {2018},
publisher = {Curran Associates Inc.},
address = {Red Hook, NY, USA},
booktitle = {Proceedings of the 32nd International Conference on Neural Information Processing Systems},
pages = {6572–6583},
numpages = {12},
location = {Montr\'{e}al, Canada},
series = {NIPS'18}
}

@inproceedings{dupont2019augmentedneuralodes,
author = {Dupont, Emilien and Doucet, Arnaud and Teh, Yee Whye},
title = {Augmented neural ODEs},
year = {2019},
publisher = {Curran Associates Inc.},
address = {Red Hook, NY, USA},
booktitle = {Proceedings of the 33rd International Conference on Neural Information Processing Systems},
articleno = {282},
pages = {3140-3150}
}

@article{deng2012mnist,
  title={The mnist database of handwritten digit images for machine learning research},
  author={Deng, Li},
  journal={IEEE Signal Processing Magazine},
  volume={29},
  number={6},
  pages={141--142},
  year={2012},
  publisher={IEEE},
  doi = {10.1109/MSP.2012.2211477},
}

@inproceedings{finlay2020trainneuralodeworld,
author = {Finlay, Chris and Jacobsen, J\"{o}rn-Henrik and Nurbekyan, Levon and Oberman, Adam M},
title = {How to train your neural ODE: the world of Jacobian and Kinetic regularization},
year = {2020},
publisher = {JMLR.org},
booktitle = {Proceedings of the 37th International Conference on Machine Learning},
articleno = {296},
pages = {3154-3164},
series = {ICML'20}
}

@InProceedings{josias2022,
author="Josias, Shane
and Brink, Willie",
title="Jacobian Norm Regularisation and Conditioning in Neural ODEs",
booktitle="Artificial Intelligence Research",
year="2022",
publisher="Springer Nature Switzerland",
address="Cham",
pages="31--45"
}

@article{yoshida2017spectralnormregularizationimproving,
      title={Spectral Norm Regularization for Improving the Generalizability of Deep Learning}, 
      author={Yuichi Yoshida and Takeru Miyato},
      year={2017},
      primaryClass={stat.ML},
      journal={arXiv preprint arXiv:1705.10941}
}

@inproceedings{miyato2018spectralnormalizationgenerativeadversarial,
  title={Spectral normalization for generative adversarial networks},
  author={Takeru, Miyato and Toshiki, Kataoka and Masanori, Koyama and Yuichi, Yoshida},
  booktitle={International Conference on Learning Representations},
  pages={1--26},
  year={2018}
}

@article{greydanus2019,
  title={Hamiltonian neural networks},
  author={Greydanus, Samuel and Dzamba, Misko and Yosinski, Jason},
  journal={Advances in neural information processing systems},
  volume={32},
  year={2019}
}

@article{greydanus2022,
  title={Dissipative hamiltonian neural networks: Learning dissipative and conservative dynamics separately},
  author={Sosanya, Andrew and Greydanus, Sam},
  journal={arXiv preprint arXiv:2201.10085},
  year={2022}
}

@article{desai2021,
  title={Port-Hamiltonian neural networks for learning explicit time-dependent dynamical systems},
  author={Desai, Shaan A and Mattheakis, Marios and Sondak, David and Protopapas, Pavlos and Roberts, Stephen J},
  journal={Physical Review E},
  volume={104},
  number={3},
  pages={034312},
  year={2021},
  publisher={APS}
}

@article{zhong2020,
  title={Symplectic ode-net: Learning hamiltonian dynamics with control},
  author={Zhong, Yaofeng Desmond and Dey, Biswadip and Chakraborty, Amit},
  journal={arXiv preprint arXiv:1909.12077},
  year={2019}
}

@article{zhong2020_workshop,
  title={Dissipative SymODEN: Encoding Hamiltonian dynamics with dissipation and control into deep learning},
  author={Zhong, Yaofeng Desmond and Dey, Biswadip and Chakraborty, Amit},
  journal={arXiv preprint arXiv:2002.08860},
  year={2020}
}

@article{cherifi2025nonlinearporthamiltonianidentificationinputstateoutput,
  title={Nonlinear port-Hamiltonian system identification from input-state-output data},
  author={Cherifi, Karim and Messaoudi, Achraf El and Gernandt, Hannes and Roschkowski, Marco},
  journal={arXiv preprint arXiv:2501.06118},
  year={2025}
}

@article{roth2025stableporthamiltonianneuralnetworks,
  title={Stable Port-Hamiltonian Neural Networks},
  author={Roth, Fabian J and Klein, Dominik K and Kannapinn, Maximilian and Peters, Jan and Weeger, Oliver},
  journal={arXiv preprint arXiv:2502.02480},
  year={2025}
}

@article{chen2019,
  title={Symplectic recurrent neural networks},
  author={Chen, Zhengdao and Zhang, Jianyu and Arjovsky, Martin and Bottou, L{\'e}on},
  journal={arXiv preprint arXiv:1909.13334},
  year={2019}
}

@article{zhu2020deephamiltoniannetworksbased,
  title={Deep Hamiltonian networks based on symplectic integrators},
  author={Zhu, Aiqing and Jin, Pengzhan and Tang, Yifa},
  journal={arXiv preprint arXiv:2004.13830},
  year={2020}
}

@article{dipietro2020sparsesymplecticallyintegratedneural,
  title={Sparse symplectically integrated neural networks},
  author={DiPietro, Daniel and Xiong, Shiying and Zhu, Bo},
  journal={Advances in Neural Information Processing Systems},
  volume={33},
  pages={6074--6085},
  year={2020}
}

@article{xiong2022nonseparablesymplecticneuralnetworks,
  title={Nonseparable symplectic neural networks},
  author={Xiong, Shiying and Tong, Yunjin and He, Xingzhe and Yang, Shuqi and Yang, Cheng and Zhu, Bo},
  journal={arXiv preprint arXiv:2010.12636},
  year={2020}
}

@article{choudhary2025learninggeneralizedhamiltoniansusing,
  title={Learning Generalized Hamiltonians using fully Symplectic Mappings},
  author={Choudhary, Harsh and Gupta, Chandan and Kungurtsev, Vyacheslav and Leok, Melvin and Korpas, Georgios},
  journal={arXiv preprint arXiv:2409.11138},
  year={2024}
}

@book{van2000l2,
  title={L2-gain and passivity techniques in nonlinear control},
  author={Van der Schaft, Arjan},
  edition = {2},
  year={2000},
  publisher={Springer}
}

@inproceedings{muller2018power,
  title={Power-balanced modelling of circuits as skew gradient systems},
  author={Muller, R{\'e}my and H{\'e}lie, Thomas},
  booktitle={21 st International Conference on Digital Audio Effects (DAFx-18)},
  year={2018},
  pages = {1-8},
}

@book{press2007numerical,
  title={Numerical Recipes: The art of Scientific Computing, Thrid Edition in C++},
  author={Press, W and Teukolsky, S and Vetterling, W and Flannery, B},
  year={2007},
  edition = {3},
  publisher={Cambridge University Press}
}

@inproceedings{helie2025,
  TITLE = {{Mod{\`e}le passif minimal d'instrument musical auto-oscillant {\`a} configuration variable en temps}},
  AUTHOR = {H{\'e}lie, Thomas and Linares, Maximino and Doras, Guillaume},
  BOOKTITLE = {{CFA 2025 - 17e Congr{\`e}s Fran{\c c}ais d'Acoustique}},
  ADDRESS = {Paris, France},
  YEAR = {2025},
  MONTH = Apr,
  DOI = {hal-05228704},
  HAL_VERSION = {v1},
  PAGES = {1-21}
}

@inproceedings{zhu2022numericalintegrationneuralordinary,
  title={On numerical integration in neural ordinary differential equations},
  author={Zhu, Aiqing and Jin, Pengzhan and Zhu, Beibei and Tang, Yifa},
  booktitle={International Conference on Machine Learning},
  pages={27527--27547},
  year={2022},
  organization={PMLR}
}

@inproceedings{neary2023,
  title={Compositional learning of dynamical system models using port-Hamiltonian neural networks},
  author={Neary, Cyrus and Topcu, Ufuk},
  booktitle={Learning for Dynamics and Control Conference},
  pages={679--691},
  year={2023},
  organization={PMLR}
}

@incollection{ortega2001,
  title={Energy shaping control revisited},
  author={Ortega, Romeo and van der Schaft, Arjan J and Mareels, Iven and Maschke, Bernhard},
  booktitle={Advances in the control of nonlinear systems},
  pages={277--307},
  year={2007},
  publisher={Springer}
}

@article{kingma2014adam,
  title={Adam: A method for stochastic optimization},
  author={Diederik P. Kingma and Jimmy Ba},
  journal={arXiv preprint arXiv:1412.6980},
  year={2014}
}

@article{trenchant2018,
title = {Finite differences on staggered grids preserving the port-Hamiltonian structure with application to an acoustic duct},
journal = {Journal of Computational Physics},
volume = {373},
pages = {673-697},
year = {2018},
issn = {0021-9991},
doi = {10.1016/j.jcp.2018.06.051},
author = {Vincent Trenchant and Hector Ramirez and Yann {Le Gorrec} and Paul Kotyczka},
}

@book{strogatz2018nonlinear,
  title={Nonlinear dynamics and chaos: with applications to physics, biology, chemistry, and engineering},
  author={Strogatz, Steven H},
  year={2024},
  edition = {3},
  publisher={CRC press}
}

@article{schwerdtner2021porthamiltonianidentificationnoisyfrequency,
  title={Port-Hamiltonian system identification from noisy frequency response data},
  author={Schwerdtner, Paul},
  journal={arXiv preprint arXiv:2106.11355},
  year={2021}
}

@article{schwerdtner2022structurepreservingmodelorderreduction,
  title={Structure-preserving model order reduction for index one port-Hamiltonian descriptor systems},
  author={Schwerdtner, Paul and Moser, Tim and Mehrmann, Volker and Voigt, Matthias},
  journal={arXiv preprint arXiv:2206.01608},
  year={2022}
}

@book{trefethrn1997,
author = {Trefethen, Lloyd N. and Bau, III, David},
title = {Numerical Linear Algebra},
publisher = {Society for Industrial and Applied Mathematics},
year = {1997},
doi = {10.1137/1.9780898719574},
address = {Philadelphia, PA},
edition   = {1}
}

\appendix
\section{Well-posedness} \label{ap: preliminaries}
\begin{thm}
\label{thm: well-posedness}
    If $\|\bm J_{\bm f}(\bm x)\|_2\leq K<\infty,\forall\bm x\in\mathcal{D}$, then $\bm f$ is K-Lipschitz.
\end{thm}
\textbf{Proof}. Fix $\bm x,\bm y\in\mathcal{D}$. Define the segment between them like $\gamma(t)=\bm x+t(\bm y-\bm x)$ for $t\in[0, 1]$. Let $\bm g(t)=\bm f(\gamma(t))$. By the chain rule, $\bm g'(t)=\bm J_{\bm f}(\gamma(t))(\bm y-\bm x)$. Using the fundamental theorem of calculus:
\[ \bm f(\bm y)-\bm f(\bm x)=\int_0^1\bm J_{\bm f}(\gamma(t))(\bm y-\bm x)dt  \]
Considering norms,
\[ \|\bm y-\bm x\|_2\leq\int_0^1 \|\bm J_{\bm f}(\gamma(t))(\bm y-\bm x)\|_2dt \]
By operator norm inequality
\[ \|\bm J_{\bm f}(\gamma(t))(\bm y-\bm x)\|_2\leq\|\bm J_{\bm f}(\gamma(t))\|_2\| (\bm y-\bm x)\|_2 \]
So 
\[ \|\bm f(\bm y)-\bm f(\bm x)\|_2\leq \int_0^1\|\bm J_{\bm f}(\gamma(t))\|_2\| (\bm y-\bm x)\|_2dt  \]
From which, using $\|\bm J_f(\bm x)\|_2\leq K$,
\[ \|\bm f(\bm y)-\bm f(\bm x)\|_2\leq K\|\bm y-\bm x\|_2 \]\hfill $\square$ \\

Following the proof of Theorem \ref{thm: well-posedness}, it is easy to see that $K=\sup_{\bm x\in\mathcal{D}}\|\bm J_{\bm f}(\bm x)\|_2$ is a Lipschitz constant.

\section{Passivity} \label{appendix-a-passive-power-balance}

The concept of \textit{passivity} is related to the power-balance property and 
it is a powerful tool for the analysis of nonlinear open systems. In the following, 
we denote $\langle \bm a|\bm b\rangle=\bm a^T\bm b$.

\begin{defi}[\protect{Passivity, \cite[Def 6.3]{khalil_nonlinear_2002}}]
\label{defi: passivity}
    A system of state $\bm x(t)\in\mathbb{R}^n$, input 
    $\overline{\bm u}(t)\in\mathbb{R}^p$ and output 
    $\overline{\bm y}(t)\in\mathbb{R}^p$ is said to be \textit{passive} 
    if there exists a continuously differentiable positive semidefinite 
    function $V(\bm x)$ (called the storage function\footnote{The storage 
    function is a Lyapunov function if it is positive definite.}) such that 
    for all the trajectories and inputs
    \begin{equation}
        \label{eq: passivity}
        \langle \overline{\bm u}(t)|\overline{\bm y}(t)\rangle\geq\langle\nabla V(\bm x(t))|\dot{\bm x}(t)\rangle,~~~\forall t\in\mathbb{R}^{+}
    \end{equation}
    Moreover, it is said to be
    \begin{itemize}
        \item \textit{lossless} if
        \begin{equation}
          \langle \overline{\bm u}(t)|\overline{\bm y}(t)\rangle=\langle\nabla V(\bm x(t))|\dot{\bm x}(t)\rangle 
          \label{eq: passive}
        \end{equation}
        \item \textit{strictly passive} if 
        \begin{equation}
            \langle\overline{\bm u}(t)|\overline{\bm y}(t)\rangle\geq\langle\nabla V(\bm x(t))|\dot{\bm x}(t)\rangle+\psi(\bm x(t))
            \label{eq: stricly-passive}
        \end{equation}
        for some positive definite function $\psi(\bm x)$.
    \end{itemize}
\end{defi}

Note that if $V(\bm x)$ represents the system's energy and $\overline{P}_{ext}=\langle \overline{\bm u}|\overline{\bm y}\rangle$ the external power given to the system, the integral of \eqref{eq: passivity} over any period of time $[0,t]$ reads
\begin{equation}
    \label{eq: integral-passivity}
    \int_0^t\langle\overline{\bm u}(s)| \overline{\bm y}(s)\rangle ds\geq V(\bm x(t))-V(\bm x(0))
\end{equation}
which interprets as "the system is passive if the energy given to the network through inputs and outputs over that period is greater or equal to the increase in the energy stored in the network over the same period". In the absence of external control $\bm u=0$, the passivity condition \eqref{eq: passivity} implies
    \begin{equation}
        \label{eq: passivity-without-control}
        V(\bm x(0))\geq V(\bm x(t))~~~~\forall t\in\mathbb
        R
    \end{equation}

\begin{rmk}
    If the system is governed by 
\begin{equation}
    \label{eq: state-space-representation}
    \begin{cases}
        \dot{\bm x}=\bm f(\bm x,\overline{\bm u}) \\
        \overline{\bm y}=\bm h(\bm x,\overline{\bm u})
    \end{cases},
\end{equation}
the criterion \eqref{eq: passivity} can be written as
\begin{equation}
    \langle \overline{\bm u}|\bm h(\bm x,\overline{\bm u})\rangle\geq\nabla V(\bm x)^T\bm f(\bm x, \overline{\bm u}),~~~\forall (\bm x,\bm u)\in\mathbb{R}^{n}\times\mathbb{R}^p.
\end{equation}
\end{rmk}

\subsection{Passivity of PHS formulations (i)-(iii)}
\begin{proposition}
    The semi-explicit port-Hamiltonian DAE (formulation (i)) governed by
    \begin{equation}
\label{eq:passivity-in-phs_S}
    \underbrace{\begin{bmatrix} \dot{\bm x} \\ \bm w \\ \bm y \end{bmatrix}}_{\mathcal{F}_{i}}
= \bm S
\underbrace{\begin{bmatrix} \nabla H(\bm x) \\ \bm z(\bm w) \\ \bm u \end{bmatrix}}_{\mathcal{E}_{i}},
\end{equation}
with $\bm S=-\bm S^T$ and $\langle \bm z(\bm w)|\bm w\rangle\geq0$, is passive in the sense of Definition \ref{defi: passivity} for storage function $V=H$ and input-output $(\overline{\bm u},\overline{\bm y})=(\bm u,-\bm y)$. 
\end{proposition}
\textbf{Proof}. By skew-symmetry of $\bm S$, $\mathcal{E}^T\mathcal{F}=\mathcal{E}^T\bm S\mathcal{E}=0$, so that the instantaneous power balance 
\begin{equation}
    \underbrace{\langle \nabla H(\bm x)|\dot{\bm x}\rangle}_{\text{stored power}~P_S}+\underbrace{\langle\bm z(\bm w)|\bm w\rangle}_{\text{dissipated power}~P_R\geq0}+\underbrace{\langle\bm u|\bm y\rangle}_{\text{external power}~P_P}=0.
    \label{eq: power-balance-equation-S}
\end{equation} 
After introducing $V=H$ and $(\overline{\bm u},\overline{\bm y})=(\bm u,-\bm y)$, \eqref{eq: power-balance-equation-S} implies
\begin{equation}
    \langle \nabla V(\bm x)|\dot{\bm x}\rangle-\langle\overline{\bm u}|\overline{\bm y}\rangle=-\langle\bm z(\bm w)|\bm w\rangle\leq0
\end{equation}
from which we obtain the passivity condition (\eqref{eq: passivity}).
\hfill $\square$ 
\begin{proposition}
    The input-state-output port-Hamiltonian system with feed through:
    \begin{equation}
            \label{eq:passivity-formulation-phs_JR}
    \underbrace{\begin{bmatrix} \dot{\bm x} \\ \bm y \end{bmatrix}}_{\mathcal{F}_{ii}}
    = 
\big(\bm J- \bm R(\bm x, \bm u)\big) 
\underbrace{\begin{bmatrix} \nabla H(\bm x) \\ \bm u \end{bmatrix}}_{\mathcal{E}_{ii}},
    \end{equation}
    where $\bm J=-\bm J^T$ and $\bm R=\bm R^T\succeq 0$ is passive in the sense of Definition \ref{defi: passivity} for storage function $V=H$ and input-output $(\overline{\bm u},\overline{\bm y})=(\bm u,-\bm y)$. Moreover, it is strictly passive if
    \begin{equation}
        \psi(\bm x)=\begin{bmatrix}
            \nabla H(\bm x) \\ 0
        \end{bmatrix}^T\bm R\begin{bmatrix}
            \nabla H(\bm x) \\ 0
        \end{bmatrix}>0,~~~\forall\bm x\neq 0
    \end{equation}
\end{proposition}
\textbf{Proof}. The instantaneous power balance
\begin{equation}
    \begin{split}
    \mathcal{E}^T\mathcal{F}&=\underbrace{\langle \nabla H(\bm x)|\dot{\bm x}\rangle}_{P_S}+\underbrace{\langle\bm u|\bm y\rangle}_{P_P}\\ &=-\underbrace{\left\langle\begin{bmatrix}
        \nabla H(\bm x) \\ \bm u
    \end{bmatrix}\left| \bm R(\bm x)\right|\begin{bmatrix}
        \nabla H(\bm x) \\ \bm u
    \end{bmatrix} \right\rangle}_{P_R\succeq0}<0
    \end{split}
\label{eq: power-balance-jr}
\end{equation}
The passivity condition is obtained from \eqref{eq: power-balance-jr} introducing $V=H$ and $(\overline{\bm u},\overline{\bm y})=(\bm u,-\bm y)$.
Note that, because $\bm R\succeq0$, $-P_R<-\psi(\bm x)$. Introducing this inequality into
\eqref{eq: power-balance-jr} gives
\begin{equation}
    \langle\nabla V(\bm x)|\bm x\rangle-\langle \overline{\bm u}|\overline{\bm y}\rangle=-P_R<-\psi(\bm x),
\end{equation}
from which the strict passivity condition (\eqref{eq: stricly-passive}) is obtained.
$\square$
\begin{proposition} The skew-symmetric gradient PH-DAE:
\begin{equation}
    \label{eq:skew-gradient-system}
    \underbrace{\begin{bmatrix} \dot{\bm x} \\ \bm w \\ \bm y \end{bmatrix}}_{\mathcal{F}_{iii}}
    = \bm S \, \underbrace{\nabla F\!\!\left(\! \begin{bmatrix}
        \bm x \\ \bm w \\ \bm u
    \end{bmatrix}\!\right)}_{\mathcal{E}_{iii}},
\end{equation}
where $\bm S=-\bm S^T$ and $F= H(\bm x)+Z(\bm w) + \frac{\bm u^T \bm u}{2}$, with $\nabla \bm Z(\bm w)\bm w\geq0$, is passive in the sense of Definition \ref{defi: passivity} for storage function $V=H$ and input-output $(\overline{\bm u},\overline{\bm y})=(\bm u,-\bm y)$.
\end{proposition}
The proof is similar to Proposition 1 with $\bm z(\bm w)=\nabla \bm Z(\bm w)$. 
\section{Jacobian quantities for the physical systems}
\label{ap: jacobian-quantities}

Tables \ref{tab:jacobian-quantities-harmonic-oscillator}-\ref{tab:jacobian-quantities-self-oscillating} 
show the Jacobian matrix $\bm J_{\bm f}(\bm x)$, the spectral norm $\|\bm J_{\bm f}(\bm x)\|_2$, 
the condition number $\kappa(\bm J_{\bm f}(\bm x))$ and the stiffness ratio $\rho(\bm J_{\bm f}(\bm x))$ 
for the different physical systems. The spectral norm is computed using the formula~\cite{golub2013matrix}
\begin{equation}
\label{eq: spectral-norm-expression}
    \|\bm J_{\bm f}(\bm x) \|_2=\sqrt{\lambda_{max}((\bm J_{\bm f}(\bm x))^T\bm J_{\bm f}(\bm x))}=\sigma_{max}(\bm J_{\bm f}(\bm x)),
\end{equation}
where $\lambda_{max}(\cdot),\sigma_{max}(\cdot)$ denotes the maximum eigenvalue and singular value of a matrix. 
The condition number is computed using the formula~\cite{golub2013matrix}
\begin{equation}
\label{eq: condition-number-expression}
\kappa(\bm J_{\bm f}(\bm x))=\|\bm J_{\bm f}(\bm x)\|_2\|\bm J_{\bm f}(\bm x)^{-1}\|_2=\frac{\sigma_{max}(\bm J_{\bm f}(\bm x))}{\sigma_{min}(\bm J_{\bm f}(\bm x))}
\end{equation}
where $\sigma_{max}(\cdot),\sigma_{min}(\cdot)$ denote the maximum and minimum singular value of a matrix. 
The stiffness ratio is computed using the formula~\cite{iserles2009first}
\begin{equation}
    \label{eq: stiffness-ratio-expression}
    \rho(\bm J_{\bm f}(\bm x))=\frac{\lambda_{max}(\bm J_{\bm f}(\bm x))}{\lambda_{min}(\bm J_{f}(\bm x))}
\end{equation}
where $\lambda_{max}(\cdot),\lambda_{min}(\cdot)$ denote the maximum and minimum eigenvalue of a matrix.
\begin{table*}[h!]
    \centering
    \begin{tabular}{|c|c|}
    \hline 
       System  & Harmonic oscillator (linear) \\
    \hline 
    \rule{0pt}{0.6cm}
      $\bm J_{\bm f}(\bm x)$  &  $\begin{bmatrix}
    0 & \frac{1}{m} \\ -k & -\frac{\alpha}{m}
\end{bmatrix}$ \\[4mm]
    \hline
    \rule{0pt}{0.8cm}
    $\|\bm J_{\bm f}(\bm x) \|_2$ & $\sqrt{\frac{1}{2}\left(k^2+\frac{1+\alpha^2}{m^2}+\sqrt{\left(k^2+\frac{1+\alpha^2}{m^2}\right)^2-\frac{4k^2}{m^2}}\right)}$ \\[4mm]
\hline 
\rule{0pt}{1cm}
$\kappa(\bm J_{\bm f}(\bm x))$ & $\sqrt{ \frac{ k^2+\tfrac{1+\alpha^2}{m^2} + \sqrt{ \left(k^2-\tfrac{1+\alpha^2}{m^2}\right)^2 + \tfrac{4k^2\alpha^2}{m^2} } }{ k^2+\tfrac{1+\alpha^2}{m^2} - \sqrt{ \left(k^2-\tfrac{1+\alpha^2}{m^2}\right)^2 + \tfrac{4k^2\alpha^2}{m^2}}}}$\\[6mm]
\hline 
\rule{0pt}{0.8cm}
$\rho(\bm J_{\bm f}(\bm x))$ & $1,~if~\alpha^2<4km~(\text{underdamped})$\\[4mm]
\hline 
    \end{tabular}
    \caption{Jacobian matrix, spectral norm, condition number and stiffness ratio for the harmonic oscillator.}
    \label{tab:jacobian-quantities-harmonic-oscillator}
\end{table*}

\begin{table*}[h!]
    \centering
    \begin{tabular}{|c|c|}
    \hline 
       System  & Duffing oscillator (nonlinear) \\
    \hline 
    \rule{0pt}{0.6cm}
      $\bm J_{\bm f}(\bm x)$  &  $\begin{bmatrix}
    0 & \frac{1}{m} \\ -k_1-3k_3q^2 & -\frac{\alpha}{m}
\end{bmatrix}$ \\[4mm]
    \hline
    \rule{0pt}{0.8cm}
    $\|\bm J_{\bm f}(\bm x) \|_2$ & $\sqrt{\frac{1}{2}\left((k_1+3k_3 q^2)^2 + \frac{1+\alpha^2}{m^2}+\sqrt{\left((k_1+3k_3 q^2)^2 + \frac{1+\alpha^2}{m^2}\right)^2-\frac{4 (k_1+3k_3 q^2)^2}{m^2}}\right)}$ \\[4mm]
\hline 
\rule{0pt}{1cm}
$\kappa(\bm J_{\bm f}(\bm x))$ & $\sqrt{ \frac{ (k_1+3k_3q^2)^2+\tfrac{1+\alpha^2}{m^2} + \sqrt{ \left((k_1+3k_3q^2)^2-\tfrac{1+\alpha^2}{m^2}\right)^2 + \tfrac{4\alpha^2 (k_1+3k_3q^2)^2}{m^2} } }{ (k_1+3k_3q^2)^2+\tfrac{1+\alpha^2}{m^2} - \sqrt{ \left((k_1+3k_3q^2)^2-\tfrac{1+\alpha^2}{m^2}\right)^2 + \tfrac{4\alpha^2 (k_1+3k_3q^2)^2}{m^2} } }}$\\[6mm]
\hline 
\rule{0pt}{0.8cm}
$\rho(\bm J_{\bm f}(\bm x))$ & 1,~if~$\alpha^2<4m(k_1+3k_3q^2)$~(\text{underdamped})\\[4mm]
\hline 
    \end{tabular}
    \caption{Jacobian matrix, spectral norm, condition number and stiffness ratio for the Duffing oscillator.}
    \label{tab:jacobian-quantities-duffing}
\end{table*}

\begin{table*}[h!]
    \centering
    \begin{tabular}{|c|c|}
    \hline 
       System  & Self-sustained oscillator (nonlinear) \\
    \hline 
    \rule{0pt}{0.6cm}
      $\bm J_{\bm f}(\bm x)$  &  $\begin{bmatrix}
    -\Gamma'(w)kq-\Gamma(w)k & -\frac{1}{m} \\
    k & 0
\end{bmatrix}$ \\[4mm]
    \hline
    \rule{0pt}{0.8cm}
    $\|\bm J_{\bm f}(\bm x) \|_2$ & $\sqrt{\frac{1}{2}\left(k^2\big(\Gamma'(w)q+\Gamma(w)\big)^2 + k^2 + \frac{1}{m^2}+\sqrt{\left(k^2\big(\Gamma'(w)q+\Gamma(w)\big)^2 + k^2 + \frac{1}{m^2}\right)^2-\frac{4k^2}{m^2}}\right)}$ \\[4mm]
\hline 
\rule{0pt}{1cm}
$\kappa(\bm J_{\bm f}(\bm x))$ & $\sqrt{ \frac{ k^2(\Gamma'(w)q+\Gamma(w))^2+k^2+\tfrac{1}{m^2} + \sqrt{ \left(k^2(\Gamma'(w)q+\Gamma(w))^2+k^2-\tfrac{1}{m^2}\right)^2 + \tfrac{4k^2(\Gamma'(w)q+\Gamma(w))^2}{m^2} } }{ k^2(\Gamma'(w)q+\Gamma(w))^2+k^2+\tfrac{1}{m^2} - \sqrt{ \left(k^2(\Gamma'(w)q+\Gamma(w))^2+k^2-\tfrac{1}{m^2}\right)^2 + \tfrac{4k^2(\Gamma'(w)q+\Gamma(w))^2}{m^2} } }}$\\[6mm]
\hline 
\rule{0pt}{0.8cm}
$\rho(\bm J_{\bm f}(\bm x))$ & 1,~if~$k^2(\Gamma'(w)q+\Gamma(w))^2<\frac{4k}{m}$~(\text{underdamped})\\[4mm]
\hline 
    \end{tabular}
    \caption{Jacobian matrix, spectral norm, condition number and stiffness ratio for the self-sustained oscillator.}
    \label{tab:jacobian-quantities-self-oscillating}
\end{table*}

\section{Implementation details}
\label{appendix:implementation-details}
Table \ref{tab:reference-size-parameters-neural-networks}-\ref{tab:low-size-parameters-neural-networks} show 
the design choices for each subnetwork inside the PHNN that models the energy or the dissipation of the system.

\section{Initial conditions and control design}
\label{appendix:control-design}
\textbf{Initial condition sampling}: All the Hamiltonians in this work can be 
written in the form
\begin{equation}
    \label{eq: all-Hamiltonians}
    H(q,p)=\frac{1}{2m}p^2+\frac{1}{2}k_1q^2+\frac{\beta}{4}k_3q^4
\end{equation}
For the harmonic (HO) and the self-sustained oscillator (SSO), the Hamiltonian 
satisfies $\beta=0$, so that \eqref{eq: all-Hamiltonians} particularizes to
\begin{equation}
    \label{eq: Hamiltonian-mass-spring-non-linear-RLC}
    H(q,p)=\frac{p^2}{2m}+\frac{1}{2}k_1q^2
\end{equation}
We now want to sample an initial condition $\bm x_0=(q_0,p_0)$ so that 
$H(\bm x_0)\in\mathcal{I}_{E_0}=[E_{min},E_{max}]$. This problems is equivalent 
to sample $(q,p)$ such that 
\begin{equation}
    \label{eq: sampling-energy-problem}
    E_0=\frac{p^2}{2m}+\frac{1}{2}k_1q^2,~~~~~~E_0\in\mathcal{I}_{E_0}
\end{equation}
Note that \eqref{eq: sampling-energy-problem} is the equation of an ellipse 
with semiaxes $a=\sqrt{\frac{2E_0}{k_1}}$ and $b=\sqrt{2E_0m}$. In parametric 
equations, this means that $(q,p)$ can be written as:
\begin{equation}
\label{eq: parametrization-p-q-mass-spring-nonlinear-rlc}
\begin{cases}
    q=\sqrt{\frac{2E_0}{k_1}}\sin\theta \\
    p=\sqrt{2mE_0}\cos\theta
\end{cases}~~~~E_0\in\mathcal{I}_{E_0},\theta\in[0,2\pi)
\end{equation}
The initial condition sampling method (Algorithm \ref{algo: sampling-mass-spring}) 
for the harmonic and self-sustained oscillator is based on \eqref{eq: parametrization-p-q-mass-spring-nonlinear-rlc}.\\

\begin{algorithm}
\caption{Initial condition sampling for HO and SSO}
\label{algo: sampling-mass-spring}
\begin{algorithmic}[1]
\Require $E_{\min}, E_{\max}$
\Ensure $(q,p)$

\State Sample $\theta \sim \mathcal{U}(0,2\pi)$
\State Sample $r \sim \mathcal{U}(E_{\min}, E_{\max})$
\State Compute $E_0 \gets \sqrt{r}$
\State Compute
\[
q \gets \sqrt{\frac{2E_0}{k_1}}\sin\theta, \qquad
p \gets \sqrt{2mE_0}\cos\theta
\]
\State \Return $(q,p)$
\end{algorithmic}
\end{algorithm}

For the Duffing oscillator (DO), the Hamiltonian satisfies $\beta>0$ in 
\eqref{eq: all-Hamiltonians}. This Hamiltonian is not quadratic, which increases 
the complexity if we are to design a similar sampling technique to the one 
described before. In this case, the initial condition sampling is based on an 
acceptance-rejection method ((Algorithm \ref{algo:sampling-duffing})). \\

\begin{algorithm}
\caption{Initial condition sampling for DO}
\label{algo:sampling-duffing}
\begin{algorithmic}[1]
\Require $E_{\min}, E_{\max}, q_{\min}, q_{\max}, p_{\min}, p_{\max}$
\Ensure $(q,p)$

\Repeat
    \State Sample $q \sim \mathcal{U}(q_{\min}, q_{\max})$
    \State Sample $p \sim \mathcal{U}(p_{\min}, p_{\max})$
    \State Compute $H_0 \gets H(q,p)$
\Until{$E_{\min} \le H_0 \le E_{\max}$}  \Comment{Accept only if Hamiltonian in 
desired range}

\State \Return $(q,p)$
\end{algorithmic}
\end{algorithm}

\textbf{Control design}: For the harmonic and Duffing oscillator, external 
control was applied as a constant force $\bm u$ so that 
\begin{equation}
        H(\bm x^{*})\in\mathcal{I}_{E_{eq}}= [E_{eq}^{min},E_{eq}^{max}]
\end{equation}
The equilibrium ph-DAE formulation for these systems is
\begin{equation}
    \label{eq: equilibrium-ms-duf}
    \begin{bmatrix} \dot{q}=0 \\ \dot{p}=0 \\  w^{*} \\  y^{*} \end{bmatrix}
= \begin{bmatrix}
    0 & 1 & 0 & 0 \\ -1 & 0 & -1 & -1 \\ 0 & 1& 0 & 0 \\ 0 & 1 & 0 & 0
\end{bmatrix}
\begin{bmatrix} \nabla_{q} H(q^{*},p^{*})=f(q^{*}) \\ \nabla_{p} H(q^{*},p^{*})=p^{*}/m \\ z(w^{*}) \\  u^{*} \end{bmatrix}
\end{equation}
which simplifies to
\begin{equation}
\label{eq: equilibrium-system-of-equations-ms-duf}
    \begin{cases}
        0=p^{*}/m \\
        0=-f(q^{*})-z(w^{*})-u^{*} \\
        w^{*}=p^{*}/m \\
        y^{*}=p^{*}/m
    \end{cases}
\end{equation}
From \eqref{eq: equilibrium-system-of-equations-ms-duf}, $p^{*}=w^{*}=y^{*}=0$, 
which also implies that $z(w^{*})=cw^{*}=0$. The equilibrium point of the system 
$(q^{*},p^{*})$ satisfies
\begin{equation}
    \begin{cases}
    H(q^{*},p^{*})=\frac{1}{2}k_1(q^{*})^2+\frac{\beta}{4}k_3(q^{*})^4=E_{eq}\in\mathcal{I}_{E_{eq}}\\
        f(q^{*})=k_1q^{*}+\beta k_3(q^{*})^3=-u^{*} \\
    \end{cases}
\end{equation}
The following algorithm is designed to find the control $u^{*}$ that shifts the 
equilibrium point to the desired energy. \\

\begin{algorithm}
\caption{Control design for the HO and DO}
\begin{algorithmic}[1]
\Require $E_{eq}^{\min}, E_{eq}^{\max}$
\Ensure Control input $u^{*}$

\State Sample $E_{eq} \sim \mathcal{U}(E_{eq}^{\min}, E_{eq}^{\max})$
\State Find $q^{*}$ such that $H(q^{*},0) = E_{eq}$
\State Compute control input $u^{*} \gets -k_1 q^{*} - \beta k_3 (q^{*})^3$
\State \Return $u^{*}$
\end{algorithmic}
\end{algorithm}

For the case of the self-sustained oscillator, we want to apply a constant 
control $\bm u$ such that the system stabilizes in a limit cycle around $\bm x^{*}$. 
The equilibrium ph-DAE formulation for this system is
\begin{equation}
        \label{eq: equilibrium-nl-rlc}
    \begin{bmatrix} \dot{q}=0 \\ \dot{p}=0 \\  w^{*} \\  y^{*} \end{bmatrix}
= \begin{bmatrix}
    0 & -1 & 1 & 0 \\ 1 & 0 & 0 & 0 \\ -1 & 0 & 0 & -1 \\ 0 & 0 & 1 & 0
\end{bmatrix}
\begin{bmatrix} \nabla_{q} H(q^{*},p^{*})=kq^{*} \\ \nabla_{p} H(q^{*},p^{*})=p^{*}/m \\ z(w^{*}) \\  u^{*} \end{bmatrix}
\end{equation}
which simplifies to
\begin{equation}
\label{eq: equilibrium-non-linear-rlc-simplification}
    \begin{cases}
        0=-p^{*}/m+z(w^{*}) \\
        0=kq^{*} \\
        w^{*}=-kq^{*}-u^{*} \\
        y^{*}=z(w^{*})
    \end{cases}
\end{equation}
We note that the system has en equilibrium point in $(q^{*},p^{*})=(0, m\cdot z(w^{*}))$. 
In order to determine whether closed orbits arise in this system, we use the 
Poincaré-Bendixson theorem~\cite{strogatz2018nonlinear}. As there is only one fixed 
point, and the phase space is of dimension 2, the existence of a close orbit 
(or, limit cycle) is guaranteed if the fixed point is unstable. The fixed point is 
unstable if $\Delta>0$ and $\tau<0$ (see~\cite{strogatz2018nonlinear}), where 
$\Delta$ and $\tau$ are the determinant and the trace of the Jacobian matrix 
$\bm J_{\bm f}$ of $f=(\dot{q},\dot{p})$ evaluated at the fixed point. 
From \eqref{eq: equilibrium-non-linear-rlc-simplification}
\begin{equation}
    \label{eq: jacobian-matrix-at-fixed-point}
    \bm J_{\bm f}(q^{*},p^{*})=\begin{bmatrix}
        z'(w^{*})\cdot k & -\frac{1}{m} \\ k & 0
    \end{bmatrix}
\end{equation}

\begin{table*}[h!]
\centering
\begin{tabular}{|c|c|c|c|c|c|c|c|c|}
\hline
\multicolumn{1}{|c|}{\textbf{Model}} & \textbf{Components} & \textbf{Network} & \textbf{Input} & \textbf{Layers} & \textbf{Hidden} & \textbf{Output} & \multicolumn{2}{|c|}{\textbf{Number of parameters}} \\
\cline{8-9}
\multicolumn{1}{|c|}{} & & \textbf{type} & \textbf{dimension} & & \textbf{units} & \textbf{dimension} & (per component) & (per model) \\
\hline
NODE & Black-box & MLP & 3 & 3 & 100 & 2 & 21,4k & 21,4k \\
\hline

\multirow{2}{*}{PHNN-JR} & $\mathbf{L}_{\theta_{L_H}}(\mathbf{x})$ & MLP & 2 & 2 & 100 & 3 & 10,7k & \multirow{2}{*}{21,8k} \\
\cline{2-8}
& $\mathbf{L}_{\theta_{L_R}}(\mathbf{x},\mathbf{u})$ & MLP & 3 & 2 & 100 & 3  & 11,1k & \\
\hline

\multirow{2}{*}{PHNN-S} & $\mathbf{L}_{\theta_{L_H}}(\mathbf{x})$ & MLP & 2 & 2 & 100 & 3 & 10,7k & \multirow{3}{*}{21,1k}\\
\cline{2-8}
& $\mathbf{L}_{\theta_{L_z}}(\mathbf{w})$ & MLP & 1 & 2 & 100 & 1 & 10,4k & \\
\cline{2-8}
&  $\mathbf{K}_{\theta_{L_z}}(\mathbf{w})$ & MLP & 1 & 2 & 100 & 0 & 0 & \\
\hline
\end{tabular}
\caption{Specificities of the implemented port-Hamiltonian neural networks with large number of parameters.}
\label{tab:reference-size-parameters-neural-networks}
\end{table*}

\begin{table*}[h!]
\centering
\begin{tabular}{|c|c|c|c|c|c|c|c|c|}
\hline
\multicolumn{1}{|c|}{\textbf{Model}} & \textbf{Components} & \textbf{Network} & \textbf{Input} & \textbf{Layers} & \textbf{Hidden} & \textbf{Output} & \multicolumn{2}{|c|}{\textbf{Number of parameters}} \\
\cline{8-9}
\multicolumn{1}{|c|}{} & & \textbf{type} & \textbf{dimension} & & \textbf{units} & \textbf{dimension} & (per component) & (per model) \\
\hline
\multirow{1}{*}{NODE}  & Black-box & MLP & 3 & 2 & 60 & 2 & 4,3k & 4,3k \\
\hline

\multirow{2}{*}{PHNN-JR}  & $\mathbf{L}_{\theta_{L_H}}(\mathbf{x})$ & MLP & 2 & 2 & 42 & 3 & 2,1k & \multirow{2}{*}{4,3k} \\
\cline{2-8}
 & $\mathbf{L}_{\theta_{L_R}}(\mathbf{x},\mathbf{u})$ & MLP & 3 & 2 & 42 & 3  & 2,2k & \\
\hline

\multirow{3}{*}{PHNN-S} & $\mathbf{L}_{\theta_{L_H}}(\mathbf{x})$ & MLP & 2 & 2 & 42 & 3 & 2,1k & \multirow{3}{*}{4,1k}\\
\cline{2-8}
& $\mathbf{L}_{\theta_{L_z}}(\mathbf{w})$ & MLP & 1 & 2 & 42 & 1 & 2k & \\
\cline{2-8}
& $\mathbf{K}_{\theta_{L_z}}(\mathbf{w})$ & MLP & 1 & 2 & 42 & 0 & 0 & \\
\hline
\end{tabular}
\caption{Specificities of the implemented port-Hamiltonian neural networks with medium size of parameters.}
\label{tab:medium-size-parameters-neural-networks}
\end{table*}

\begin{table*}[h!]
\centering
\begin{tabular}{|c|c|c|c|c|c|c|c|c|}
\hline
\multicolumn{1}{|c|}{\textbf{Model}} & \textbf{Components} & \textbf{Network} & \textbf{Input} & \textbf{Layers} & \textbf{Hidden} & \textbf{Output} & \multicolumn{2}{|c|}{\textbf{Number of parameters}} \\
\cline{8-9}
& & \textbf{type} & \textbf{dimension} & & \textbf{units} & \textbf{dimension} & (per component) & (per model) \\
\hline
\multirow{1}{*}{NODE} & Black-box & MLP & 3 & 2 & 24 & 2 & 848 & 848 \\
\hline

\multirow{2}{*}{PHNN-JR} & $\mathbf{L}_{\theta_{L_H}}(\mathbf{x})$ & MLP & 2 & 2 & 16 & 3 & 371 & \multirow{2}{*}{809} \\
\cline{2-8}
 & $\mathbf{L}_{\theta_{L_R}}(\mathbf{x},\mathbf{u})$ & MLP & 3 & 2 & 16 & 3  & 438 & \\
\hline

\multirow{3}{*}{PHNN-S}  & $\mathbf{L}_{\theta_{L_H}}(\mathbf{x})$ & MLP & 2 & 2 & 16 & 3 & 371 & \multirow{3}{*}{852}\\
\cline{2-8}
& $\mathbf{L}_{\theta_{L_z}}(\mathbf{w})$ & MLP & 1 & 2 & 20 & 1 & 481 & \\
\cline{2-8}
& $\mathbf{K}_{\theta_{L_z}}(\mathbf{w})$ & MLP & 1 & 2 & 20 & 0 & 0 & \\
\hline
\end{tabular}
\caption{Specificities of the implemented port-Hamiltonian neural networks with low size of parameters.}
\label{tab:low-size-parameters-neural-networks}
\end{table*} 

The Jacobian matrix $\bm J_{\bm f}$ has determinant $\Delta=km>0$ and trace $\tau=z'(w^{*})\cdot k$. In order to 
generate self-oscillations, the external force $u$ must be chosen so that $z'(w^{*})=z'(-kq^{*}-u)<0$. 
Taking into account that in the equilibrium $q^{*}=0$ and $w^{*}=-u^{*}$, the set of external controls that 
sets the system in a self-oscillation is
\begin{equation}
    u^{*}=-w^{*}~\text{where}~w^{*}\in\{w\in\mathbb{R}~|~z'(w)<0 \}
\end{equation}
The above information gives the following algorithm
\begin{algorithm}
\caption{Control design for the SSO}
\begin{algorithmic}[1]
\Require $z(w)$
\Ensure Control input $u^{*}$

\State Find $I_{w}=\{w\in\mathbb{R}~|~z'(w)<0\}$.
\State Sample $w \sim \mathcal{U}(I_w)$
\State Compute control input $u^{*}=-w$
\State \Return $u^{*}$
\end{algorithmic}
\end{algorithm}
\section{Table of results}
\label{ap: results-experiments}

Tables \ref{tab: results-study-I}-\ref{tab:study_3_jacobian_regularization} show the exact inference errors for each 
of the studies presented in Section \ref{sec:experiments}. In each table cell, the median value and IQR for 10 model 
initializations are reported.

\begin{table*}[h!]
\centering 
\begin{tabular}{|c|c|c|c|}
\hline
 \textbf{Harmonic} & 25 & 100 & 400 \\
\hline
NODE-RK2 & 2.53e-02 [2.33e-02] & 1.69e-04 [1.54e-04] & 2.79e-05 [1.84e-05] \\
\hline 
PHNN-JR-DG & 2.58e-04 [3.82e-04] & 9.10e-06 [1.22e-05] & 7.29e-06 [5.07e-06] \\
\hline 
PHNN-JR-RK2 & 1.14e-04 [1.02e-04] & 6.88e-06 [3.81e-06] & 5.18e-06 [5.50e-06] \\
\hline 
PHNN-S-DG & \textbf{4.22e-05 [5.99e-05]} & 6.24e-06 [4.97e-06] & 4.78e-06 [1.21e-05] \\
\hline
PHNN-S-RK2 & 4.89e-05 [2.11e-04] & \textbf{5.77e-06 [9.55e-06]} & \underline{\textbf{3.73e-06 [8.89e-06]}} \\
\hline 
 \textbf{Duffing} & 25 & 100 & 400 \\
\hline 
NODE-RK2 & 4.09e-01 [1.01e+00] & 7.62e-02 [2.62e-01] & 3.67e-03 [2.73e-03] \\
\hline
PHNN-JR-DG & 4.36e-01 [1.19e+00] & 1.83e-02 [2.55e-02] & 7.89e-03 [4.76e-03] \\
\hline 
PHNN-JR-RK2 & 5.82e-01 [2.96e+00] & 2.13e-02 [3.15e-02] & 6.15e-03 [9.92e-03] \\
\hline 
PHNN-S-DG & \textbf{7.06e-02 [4.37e+00]} & \textbf{8.08e-03 [4.93e-02]} & \underline{\textbf{9.08e-04 [8.69e-04]}} \\
\hline 
PHNN-S-RK2 & 9.35e-02 [4.54e+00] & 2.07e-02 [1.85e-02] & 9.17e-03 [6.13e-03] \\
\hline 
\textbf{Self-sustained} & 25 & 100 & 400 \\
\hline 
NODE-RK2 & 2.59e-02 [2.21e-02] & 2.01e-04 [1.38e-04] & 9.56e-05 [9.18e-05] \\
\hline 
PHNN-JR-DG & \textbf{1.71e-03 [2.71e-03]} & \textbf{7.93e-05 [5.03e-05]} & \underline{\textbf{7.06e-05 [7.01e-05]}} \\
\hline 
PHNN-JR-RK2 & 1.38e-02 [2.48e-02] & 2.59e-04 [2.63e-04] & 1.39e-04 [2.20e-04] \\
\hline 
PHNN-S-DG & 5.40e-03 [3.60e-03] & 1.68e-03 [4.83e-04] & 2.72e-03 [1.12e-03] \\
\hline 
PHNN-S-RK2 & 4.26e-03 [7.81e-03] & 2.64e-03 [1.60e-03] & 2.22e-03 [7.85e-04] \\
\hline 
\end{tabular}
\caption{Median inference error and IQR (in brackets) for different numbers of training points when learning the three oscillatory systems with small models. In bold, the best result for each combination of system and number of training points.}
\label{tab: results-study-I}
\end{table*}

\begin{table*}[h!]
\centering 
\begin{tabular}{|c|c|c|c|}
\hline 
 \textbf{Harmonic} & Small size & Medium size & Large size \\
\hline
NODE-RK2 & 2.53e-02 [2.33e-02] & 2.85e-02 [2.59e-02] & 1.45e-01 [4.67e-02] \\
\hline
PHNN-JR-DG & 2.58e-04 [3.82e-04] & 1.47e-04 [2.68e-04] & 7.33e-04 [2.52e-03] \\
\hline
PHNN-JR-RK2 & 1.14e-04 [1.02e-04] & 5.56e-05 [1.19e-04] & \textbf{1.51e-04 [3.87e-04]} \\
\hline
PHNN-S-DG & \underline{\textbf{4.22e-05 [5.99e-05]}} & 5.56e-05 [9.58e-05] & 2.45e-04 [4.59e-04] \\
\hline 
PHNN-S-RK2 & 4.89e-05 [2.11e-04] & \textbf{4.28e-05 [3.73e-05]} & 2.01e-04 [2.13e-04] \\
\hline 
\textbf{Duffing} & Small size & Medium size & Large size \\
\hline 
NODE-RK2 & 4.09e-01 [1.01e+00] & 3.80e-01 [1.77e+00] & 2.31e-01 [3.08e+00] \\
\hline 
PHNN-JR-DG & 4.36e-01 [1.19e+00] & 1.15e-01 [1.19e-01] & 3.08e-01 [4.01e-01] \\
\hline 
PHNN-JR-RK2 & 5.82e-01 [2.96e+00] & 1.50e-01 [2.60e-01] & 3.08e-01 [7.61e-02] \\
\hline 
PHNN-S-DG & \underline{\textbf{7.06e-02 [4.37e+00]}} & 8.63e-02 [1.10e-01] & \textbf{1.80e-01 [1.46e-01]} \\
\hline 
PHNN-S-RK2 & 9.35e-02 [4.54e+00] & \textbf{7.42e-02 [7.00e-02]} & 2.19e-01 [1.70e-01] \\
\hline
\textbf{Self-sustained} & Small size & Medium size & Large size \\
\hline 
NODE-RK2 & 2.59e-02 [2.21e-02] & 3.74e-02 [3.06e-02] & 6.53e-02 [5.74e-02] \\
\hline 
PHNN-JR-DG & \underline{\textbf{1.71e-03 [2.71e-03]}} & 6.56e-03 [2.02e-02] & 7.59e-03 [1.54e-02] \\
\hline 
PHNN-JR-RK2 & 1.38e-02 [2.48e-02] & 1.36e-02 [2.37e-02] & 1.75e-02 [1.50e-02] \\
\hline 
PHNN-S-DG & 5.40e-03 [3.60e-03] & 6.12e-03 [3.48e-03] & 7.74e-03 [5.18e-03] \\
\hline 
PHNN-S-RK2 & 4.26e-03 [7.81e-03] & \textbf{3.98e-03 [2.95e-03]} & \textbf{3.77e-03 [4.27e-03]} \\
\hline 
\end{tabular}
\caption{Median inference error and IQR (in brackets) for different model sizes when learning the three oscillatory systems with $N_{train}=25$. In bold, the best result for each combination of system and model size.}
\label{tab: results-study-II}
\end{table*}

\begin{table*}[h!]
    \centering

    \begin{tabular}{|c|c|c|c|c|}
        \hline 
        Harmonic & BL & CN & SN & SR \\
        \hline
        NODE-RK2 & 2.53e-02 [2.33e-02] & 7.82e-02 [1.63e+00] & 2.63e-02 [2.53e-02] & 2.56e-02 [5.70e-02] \\
        \hline 
        PHNN-JR-DG & 2.58e-04 [3.82e-04] & 2.20e-04 [3.40e-04] & 1.94e-04 [3.09e-04] & 2.30e-04 [3.19e-04] \\
        \hline 
        PHNN-JR-RK2 & 1.14e-04 [1.02e-04] & 1.50e-04 [5.70e-04] & 7.18e-05 [9.39e-05] & 1.26e-04 [1.46e-04] \\
        \hline 
        PHNN-S-DG & \textbf{4.22e-05 [5.99e-05]} & 3.84e-05 [4.05e-05] & \textbf{3.44e-05 [6.62e-05]} & 1.13e-04 [3.07e-04] \\
        \hline 
        PHNN-S-RK2 & 4.89e-05 [2.11e-04] & \underline{\textbf{1.17e-05 [5.67e-05]}} & 5.66e-05 [2.02e-04] & \textbf{2.97e-05 [5.35e-04]} \\
        \hline
        Duffing & BL & CN & SN & SR \\
        \hline
        NODE-RK2 & 4.09e-01 [1.01e+00] & 9.93e+00 [4.77e+01] & 5.77e-01 [9.72e-01] & 7.33e-01 [5.18e+00] \\
        \hline 
        PHNN-JR-DG & 4.36e-01 [1.19e+00] & 6.72e+00 [5.63e+01] & 3.55e-01 [1.24e+01] & 4.43e-01 [8.66e+01] \\
        \hline 
        PHNN-JR-RK2 & 5.82e-01 [2.96e+00] & 2.76e+00 [2.35e+01] & 3.26e-01 [2.69e+01] & 9.97e-01 [3.91e+02] \\
        \hline 
        PHNN-S-DG & \textbf{7.06e-02 [4.37e+00]} & \underline{\textbf{5.15e-02 [8.92e-02]}} & \textbf{6.26e-02 [1.07e-01]} & \textbf{6.57e-02 [2.54e+00]} \\
        \hline 
        PHNN-S-RK2 & 9.35e-02 [4.54e+00] & 6.74e-02 [5.83e-02] & 8.27e-02 [4.53e+00] & 6.97e+00 [2.80e+01] \\
        \hline
        Self-sustained & BL & CN & SN & SR \\
        \hline 
        NODE-RK2 & 2.59e-02 [2.21e-02] & 4.04e-02 [1.63e-02] & 2.39e-02 [3.27e-02] & 3.10e-02 [4.33e-02] \\
        \hline 
        PHNN-JR-DG & \underline{\textbf{1.71e-03 [2.71e-03]}} & \textbf{2.84e-03 [2.74e-03]} & \textbf{1.91e-03 [3.14e-03]} & 1.43e-02 [2.36e-02] \\
        \hline 
        PHNN-JR-RK2 & 1.38e-02 [2.48e-02] & 8.24e-03 [2.58e-02] & 1.38e-02 [2.93e-02] & 3.62e-02 [3.97e-02] \\
        \hline 
        PHNN-S-DG & 5.40e-03 [3.60e-03] & 4.90e-03 [3.89e-03] & 5.23e-03 [3.72e-03] & 1.22e-02 [4.37e-03] \\
        \hline 
        PHNN-S-RK2 & 4.26e-03 [7.81e-03] & 5.00e-03 [5.67e-03] & 4.31e-03 [8.56e-03] & \textbf{5.11e-03 [6.72e-03]} \\
        \hline 

    \end{tabular}
    
    \caption{Median inference error and IQR (in brackets) after applying the different Jacobian regularizations when learning the three oscillatory systems with $N_{train}=25$ 
    and small models. The best result for each combination of system and Jacobian regularization is shown in bold. Notation: BL refers to the baseline (i.e. no Jacobian regularization); 
    CN, to the condition number regularization \eqref{eq:condition-number-loss}; SN, to the spectral norm regularization \eqref{eq:spectral-norm-loss}; and SR, to the stiffness ratio 
    regularization \eqref{eq:stiffness-ratio-loss}}
    \label{tab:study_3_jacobian_regularization}
\end{table*}

\clearpage
\clearpage
\textbf{Figure captions} \\

Figure 1: Architecture of the two PHNN models considered in this work. White boxes with orange contour denote fixed algebraic operations whereas orange boxes indicate the trainable parameters. \\

Figure 2: Training and inference diagram for the continuous models $\bm f_{\theta}$. \\

Figure 3: Training and inference diagram for the discrete models $\bm g_{\theta}$. \\

Figure 4: Training points, test initial points and two complete test trajectories for each of the three oscillatory systems. Note that in the case of the harmonic and Duffing oscillator, the applied control shifted the equilibrium point from $(p,q)=(0,0)$ whereas in the case of the self-sustained oscillator, it stabilizes the trajectories in a limit cycle. \\

Figure 5: Schematic sampling and dataset construction procedure. A trajectory generated at sampling frequency ${sr}_{gen}$ 
    over a duration $D=D_{infer}=\beta T_0$ is shown as white dot markers. From this trajectory, a training point, highlighted with a green star marker, 
    is uniformly sampled from a subset of samples, shown as black dot markers, obtained at frequency ${sr}_{train}$ and restricted to $t\leq\alpha T_0$. 
    The training horizon $D_{train}$ and the inference horizon $D_{infer}$ are indicated by arrows, with the vertical dashed line marking the end of the 
    training interval. \\

Figure 6: Boxplot of the inference errors for the three different oscillators and varying numbers of training points (small models). From left to right: harmonic, Duffing and self-sustained oscillators. \\

Figure 7: Comparison of the learned trajectory discretizations for each oscillatory system starting on $\bm x_0$ such that 
        $H(\bm x_0)=0.5J$. (\textbf{Left}) NODE and PHNN-S models trained on $N_{train}=25$ from the harmonic oscillator. (\textbf{Center}) NODE and PHNN-S models 
        trained on $N_{train}=100$ from the Duffing oscillator. (\textbf{Right}) NODE and PHNN-JR models trained on $N_{train}=25$ from the self-sustained oscillator. For the harmonic and Duffing 
        oscillator, the trajectory is obtained with $\bm u=0$; and for the self-sustained, with $\bm u$ such that the system stabilizes 
        in a limit cycle.\\

Figure 8: Boxplot of the inference errors for the three different oscillators and model sizes ($N_{train}=25$). From left to right: harmonic, Duffing and self-sustained oscillators.\\

Figure 9: Comparison of the learned trajectory discretizations by the discrete gradient for each oscillatory system starting 
    on $\bm x_0$ such that $H(\bm x_0)=0.75J$. (\textbf{Left}) PHNN-S models trained on $N_{train}=25$ from the harmonic oscillator. (\textbf{Center}) 
    PHNN-S models trained on $N_{train}=25$ from the Duffing oscillator. (\textbf{Right}) PHNN-JR models trained on $N_{train}=25$ from the self-sustained oscillator. For the harmonic and Duffing 
    oscillator, the trajectory is obtained with $\bm u=0$; and for the self-sustained, with $\bm u$ such that the system stabilizes 
    in a limit cycle.\\

Figure 10: Mean condition number (\textbf{left}), spectral norm (\textbf{center}) and stiffness ratio (\textbf{right}) per 
        optimizer step during the training for the different architectures when modeling the Duffing oscillator with $N_{train}=25$ 
        points and no Jacobian regularization (BL). For each optimizer step, the corresponding mean value (solid line) is obtained 
        computing the average over the 10 model initializations, with the shaded area indicating the range between the minimum and 
        maximum values.\\

Figure 11: Boxplot of the inference errors for the three different oscillators and normalizations ($N_{train}=25$ and small 
    number of parameters). Notation: BL refers to the baseline (i.e. no Jacobian regularization); CN, to the condition number 
    regularization \eqref{eq:condition-number-loss}; SN, to the spectral norm regularization \eqref{eq:spectral-norm-loss}; 
    and SR, to the stiffness ratio regularization \eqref{eq:stiffness-ratio-loss}. From left to right: harmonic, Duffing and self-sustained oscillators.\\

Figure 12: Impact on the mean condition number (\textbf{left}), spectral norm (\textbf{center}) and stiffness ratio 
    (\textbf{right}) per optimizer step during the training for the different architectures when modeling the Duffing 
    oscillator with $N_{train}=25$ points and the proposed Jacobian regularizations. For each optimizer step, the corresponding 
    mean value (solid line) is obtained computing the average over the 10 model initializations, with the shaded area indicating 
    the range between the minimum and maximum values.\\

\end{document}